\documentclass[final,12pt]{colt2023} % Include author names

\usepackage{times}
\usepackage{booktabs}
\usepackage{multirow}
\usepackage{framed}
\usepackage{amsmath,amssymb,rotating,graphicx,rotating,float}
\usepackage{mathabx}
\usepackage{url}
\usepackage{cleveref}
\usepackage{bbm}
\usepackage{color}

\usepackage{breakcites}

\usepackage{mathrsfs}
\newlength{\dhatheight}

\newcommand{\Ascr}{\mathscr{A}}
\newcommand{\Bscr}{\mathscr{B}}

\newcommand{\cT}{\mathcal{T}}

\newcommand{\good}{\textup{good}}

\newcommand{\wb}{\widebar}
\newcommand{\wt}{\widetilde}
\newcommand{\ms}{\mathsf}
\newcommand{\bxi}{\boldsymbol{\xi}}
\newcommand{\bfeta}{\boldsymbol{\eta}}
\newcommand{\bbeta}{\boldsymbol{\beta}}
\newcommand{\balpha}{\boldsymbol{\alpha}}
\newcommand{\bkappa}{\boldsymbol{\kappa}}
\newcommand{\Yp}{\widetilde{\mathcal{Y}^2}}
\newcommand{\W}{\mathsf{W}}
\newcommand{\sfh}{\mathsf{h}}

\newcommand{\cH}{\mathcal{H}}
\newcommand{\bff}{\textbf{\textup{f}}}
\newcommand{\bb}{\textbf{\textup{b}}}
\newcommand{\bu}{\textbf{\textup u}}
\newcommand{\bv}{\textbf{\textup{v}}}
\newcommand{\bw}{\textbf{\textup{w}}}
\newcommand{\bg}{\textbf{\textup{g}}}
\newcommand{\bt}{\textbf{\textup{t}}}
\newcommand{\bI}{\textbf{\textup{I}}}
\newcommand{\bi}{\textbf{\textup{i}}}
\newcommand{\bY}{\textbf{\textup{Y}}}
\newcommand{\by}{\textbf{\textup y}}
\newcommand{\bx}{\textbf{\textup x}}
\newcommand{\bz}{\textbf{\textup z}}
\newcommand{\bs}{\textbf{\textup s}}
\newcommand{\wmP}{\widebar{\mathsf P}}
\newcommand{\wmD}{\widebar{\mathsf D}}
\newcommand{\wmA}{\widebar{\mathsf A}}
\newcommand{\wmW}{\widebar{\mathsf{W}}}
\newcommand{\ORD}{\textup{ORD}}
\newcommand{\indi}{\mathbbm{1}}
\newcommand{\wh}{\widehat}
\newcommand{\fB}{\mathfrak B}
\newcommand{\N}{\mathbb N}
\newcommand{\cD}{\mathcal D}

\newcommand{\PC}{\mathsf{PC}}
\newcommand{\XPC}{\mathsf{XPC}}
\newcommand{\XY}{\mathsf{XY}}
\newcommand{\cS}{\mathcal{S}}
\newcommand{\cI}{\mathcal{I}}
\newcommand{\cJ}{\mathcal{J}}

\newcommand{\sfJ}{\mathsf{J}}
\newcommand{\sfK}{\mathsf{K}}

\newcommand{\fkJ}{\mathfrak{J}}

\newcommand{\seq}{\textup{seq}}

\newcommand{\conv}{\textup{Conv}}

\newcommand{\pr}{\textup{\textbf{P}}}
\newcommand{\bE}{\textup{\textbf{E}}}

\newcommand{\per}{\textup{per}}

\newcommand{\val}{{\rm val}}
\newcommand{\A}{\mathcal{A}}
\newcommand{\B}{\mathcal{B}}

\newcommand{\RE}{{\rm RE}}

\newcommand{\G}{{\mathcal{G}}}

\renewcommand{\epsilon}{\varepsilon}

\newcommand{\X}{\mathcal X}

\newcommand{\Y}{\mathcal Y}

\newcommand{\er}{{\rm er}}

\DeclareSymbolFont{bbold}{U}{bbold}{m}{n}
\DeclareSymbolFontAlphabet{\mathbbold}{bbold}

\newcommand{\F}{\mathcal{F}}

\newcommand{\reals}{\mathbb{R}}

\newcommand{\argmax}{\mathop{\rm argmax}}

\newcommand{\ignore}[1]{}
\newcommand{\oldstuff}[1]{}

\colorlet{sgreen}{black!45!green}

\newsavebox{\savepar}

\makeatletter
\newcommand{\vast}{\bBigg@{3}}
\newcommand{\Vast}{\bBigg@{4}}
\makeatother

\renewenvironment{proof}[1][]{\par\noindent{\bf Proof #1\ }}{\hfill\BlackBox\\[2mm]}

\title[Multiclass Learning]{Universal Rates for Multiclass Learning}

\coltauthor{
\Name{Steve Hanneke} \Email{steve.hanneke@gmail.com} \\ 
\addr Purdue University
\AND
\Name{Shay Moran} \Email{smoran@technion.ac.il} \\ 
\addr Technion
\AND
\Name{Qian Zhang} 
\Email{zhan3761@purdue.edu} \\
\addr Purdue University
}

\begin{document}
\maketitle

\begin{abstract}
We study universal rates for multiclass classification, establishing the optimal rates (up to log factors) for all hypothesis classes. This generalizes previous results on binary classification (Bousquet, Hanneke, Moran, van Handel, and Yehudayoff, 2021), and resolves an open question studied by Kalavasis, Velegkas, and Karbasi (2022) who handled the multiclass setting with a bounded number of class labels. In contrast, our result applies for any countable label space. Even for finite label space, our proofs provide a more precise bounds on the learning curves, as they do not depend on the number of labels. Specifically, we show that any class admits exponential rates if and only if it has no infinite Littlestone tree, and admits (near-)linear rates if and only if it has no infinite Daniely-Shalev-Shwartz-Littleston (DSL) tree, and otherwise requires arbitrarily slow rates. DSL trees are a new structure we define in this work, in which each node of the tree is given by a pseudo-cube of possible classifications of a given set of points. Pseudo-cubes are a structure, rooted in the work of Daniely and Shalev-Shwartz (2014), and recently shown by Brukhim, Carmon, Dinur, Moran, and Yehudayoff (2022) to characterize PAC learnability (i.e., uniform rates) for multiclass classification. We also resolve an open question of Kalavasis, Velegkas, and Karbasi (2022) regarding the equivalence of classes having infinite Graph-Littlestone (GL) trees versus infinite Natarajan-Littlestone (NL) trees, showing that they are indeed equivalent. 
\end{abstract}


\section{Introduction}
Multiclass classification, i.e., classifying data into multiple classes in some label (class) space $\Y$ is a fundamental task in machine learning with direct application in a wide range of scenarios including image recognition \citep{8016501}, natural language processing \citep{8416973}, protein structure classification \citep{dietmann2001identification}, etc. 
In practice, the number of classes $(|\Y|)$ could be huge or infinite; e.g., in statistical language models \citep{song1999general}, $|\Y|$ is the vocabulary size; for count data prediction \citep{hellerstein1993theoretical}, $\Y$ is the set of natural numbers. 
Thus, the study of multiclass learnability and error rates has been a crucial problem in learning theory. However, even under the renowned PAC (Probably Approximately Correct) learning framework \citep{10.1145/1968.1972}, until recently solved by \citet{brukhim2022characterization}, the characterization of multiclass learnability for infinite number of classes $(|\Y|=\infty)$ remained to be a challenging problem for decades after the characterization of PAC learnability of binary classification ($|\Y|=2$) with the finiteness of the Vapnik-Chervonenkis (VC) dimension \citep{vapnik1971uniform,10.1145/76359.76371}. 
\citet{natarajan1988two,natarajan1989learning} defined two extentions of the VC dimension in multiclass learning, the Natarajan dimension ($\dim_{N}$) and the Graph dimension ($\dim_G$) which both characterize the multiclass PAC learnability for finite number of classes ($|\Y|<\infty$). Though the Graph dimension was shown to be unable to characterize the multiclass PAC learnability when $|\Y|=\infty$, it was conjectured if the Natarajan dimension would do \citep{natarajan1989learning}. \citet{daniely2014optimal} defined a new dimension named the \emph{Daniely-Shalev-Shwartz (DS) dimension} ($\dim$) by \citet{brukhim2022characterization} and showed that finite DS dimension is a necessary condition
for PAC learnability.
Recently, \citet{brukhim2022characterization} proved that the DS dimension fully characterizes PAC learnability in the multiclass setting by proposing an algorithm achieving $O(\frac{\dim(\cH)^{3/2}\log^2(n)}{n})$ (see Section \ref{sec:setting} for details) error rate for any hypothesis class $\cH$ under the PAC framework. They also refuted the conjecture that the Natarajan dimension characterizes multiclass PAC learnability by providing a hypothesis class with the Natarajan dimension 1 and an infinite DS dimension. 

In terms of the \emph{learning curve}, i.e., the error rate (measured on test data) as a function of the number of training examples, due to its distribution-free nature, the PAC framework, however, fails to capture the fine-grained and potentially faster \emph{distribution-dependent} learning curves of hypothesis classes. In the realizable setting, PAC learning considers the best worst-case (uniform) performance of any algorithm on a hypothesis class against any realizable distribution. 
While in real-world problems, the distribution for data generation is often fixed in one task and the study of learning curves under fixed distributions is concerned. 
These thoughts motivate the proposition of \emph{universal learning} in the work of \citet{universal_learning}, where they consider the distribution-dependent error rate of a learning algorithm on a hypothesis class, holding universally for all realizable distributions. 
They showed that for binary classification, the following \emph{trichotomy} exists for any hypothesis class~$\cH$ with $|\cH|>3$: $\cH$ is either universally learnable with optimal rate $e^{-n}$ (exponential rate), universally learnable with optimal rate $1/n$ (linear rate), or requires arbitrarily slow rates (see Section~\ref{sec:setting} for details), which is fully determined by the combinatorial properties of $\cH$ (the nonexistence of certain infinite trees). Compared to the dichotomy in PAC learning: $\cH$ is either PAC-learnable with a linear uniform rate ($1/n$) or is not PAC-learnable at all, universal learning provides more insights of the learning curve in binary classification. 

A natural direction is to extend the framework of universal learning to multiclass classification that would bring fine-grained disribution-dependent analysis of learning curves in multiclass problems.
Recently, \citet{kalavasis2022multiclass} proved the same trichotomy for multiclass universal learining assuming \emph{finite} label space $(|\Y|<\infty)$: a hypothesis class with finite label space is either universally learnable with optimal rate $e^{-n}$, universally learnable with optimal rate $1/n$, or requires arbitrarily slow rates, depending on the nonexistence of an infinite Littlestone tree and an infinite Natarajan-Littlestone (NL) tree they defined (see Section \ref{sec:main-results} for details). 
However, their analysis for the linear universal rate based on NL trees cannot be extended to the setting of countable label spaces. 
As is pointed out in \citet{kalavasis2022multiclass}, it is an important next step to characterize multiclass universal learning with \emph{infinite} label space ($|\Y|=\infty$).

However, for general uncountable label spaces, the existence of a universally measurable learning algorithm that is universally consistent (see Section \ref{sec:setting} for details), i.e., with an error rate converging to zero for any realizable distributions, remains unsolved to our knowledge, which is an important problem in itself. 
Thus, we focus on countable label spaces in this paper and summarize our contributions below.

\paragraph{Contributions.}
In this paper, we study multiclass universal learning for general \emph{countable} label spaces ($|\Y|$ can be infinite). We prove in Theorem \ref{thm:tree-rates} that a hypothesis class with a countable label space is either universally learnable with optimal rate $e^{-n}$, universally learnable with optimal rate in $\wt{\Theta}(1/n)$ (near-linear rate), or requires arbitrarily slow rates, which is fully characterized by the nonexistence of an infinite \emph{Littlestone tree} and an infinite \emph{Daniely-Shalev-Shwartz-Littleston (DSL) tree} proposed by us (see Section \ref{sec:main-results} for details). 
In particular, we propose different universally measurable learning algorithms that achieve the exponential and near-linear rates in those corresponding settings. 
We also show that the NL tree does \emph{not} characterize the near-linear rate by proving the existence of a hypothesis class that has an infinite DSL tree but has no NL tree of depth 2 for countable label space in Theorem \ref{thm:NL-counter}. 
Finally, we solve the first question in \citet[Open question 1]{kalavasis2022multiclass} by proving in Theorem \ref{thm:NL-GL} that a hypothesis class with finite label space $(|\Y|<\infty)$ has an infinite NL tree if and only if it has an infinite Graph-Littlestone (GL) tree defined in \citet[Definition 8]{kalavasis2022multiclass}, which implies that the GL tree is equivalent to the NL tree in determining the universal rate of multiclass learning with finite label space.

\paragraph{Outline.} In Section \ref{sec:setting}, we formally define the multiclass learning problem considered in this paper and the universal error rate which is compared to the uniform error rate in PAC learning. 
In Section \ref{sec:main-results}, we introduce the definitions of the different tree structures of a hypothesis class and state the main theoretical results. 
In Section \ref{sec:future}, we discuss some future research directions in multiclass learning. 
In Section \ref{sec:examples}, we provide three examples of the multiclass learning problem, each corresponding to a different universal rate in the trichotomy. 
In Section \ref{sec:tech}, we summarize the key technical details and the proof sketches of the main results.
The complete proofs are included in the appendix.

\subsection{The multiclass learning problem and the universal rates} \label{sec:setting}
In this section, we introduce the multiclass learning problem considered in this paper and the concept of universal learning. 
We refer readers to Appendix \ref{sec:notation} for the notation we used throughout the paper.
Let $\X$ denote the domain (feature space), $\Y$ denote the codomain (label space), and $\cH\subseteq\Y^{\X}$ denote the hypothesis class. 
To avoid measurability issues, we assume that $\X$ is a Polish space and $\Y$ is countable with $|\Y|\ge 2$ throughout the paper. 

A classifier in multiclass learning is a universally measurable function $h:\X\rightarrow\Y$. 
For any probability distribution $P$ on $\X\times\Y$, we define the error rate of $h$ under $P$ as
$$
\er(h)=\er_P(h):=P(\{(x,y)\in\X\times\Y: h(x)\neq y\}).
$$
In this paper, we focus on realizable distributions: a distribution $P$ is called ($\cH$-)realizable if $\inf_{h\in\cH}\er_P(h)=0$. We use $\RE(\cH)$ to denote the set of all $\cH$-realizable distributions. 
A multiclass learning algorithm is a sequence of universally measurable functions\footnote{For notational convenience, we only defines deterministic algorithms here. However, our results still hold when randomized algorithms are allowed, as all algorithms we construct to show the upper bounds are deterministic and all proofs of lower bounds apply to randomized algorithms.} 
$$
H_n: (\X\times\Y)^{n}\times\X\rightarrow\Y,\quad n\in\N_0.
$$
For a sequence of independent $P$-distributed samples $((X_1,Y_1))_{i\in\N}$, the learning algorithm outputs a data-dependent function for each $n\in\N_0$
$$
\wh h_n: \X\rightarrow\Y,\ x\mapsto H_n((X_1,Y_1),\dots,(X_n,Y_n),x).
$$ 
The objective of multiclass learning is to design a learning algorithm such that the expected error rate of the output classifier $\bE[\er(\wh h_n)]$ decreases as fast as possible with the size of the input sequence $n$. 
Since $\X$ is Polish and $\Y$ is countable, for $\cH$ defined as the set of all measurable functions in $\Y^\X$, there exists a universally consistent learning algorithm, i.e., a learning algorithm such that $\bE[\wh h_n]\rightarrow0$ for all realizable distributions $P$ \citep{10.1214/20-AOS2029}
\footnote{Actually, \citet{10.1214/20-AOS2029} establishes the existence of a universally consistent learning algorithm assuming $\X$ is essentially separable and $\Y$ is countable. Any Polish space, being separably metrizable, is essentially separable.}. 
Then, it is natural to ask about the rate of the convergence. 

Under PAC learning, the uniform error rate over all realizable distributions is concerned. For multiclass learning, the following upper and lower bounds of the uniform rate is proved:
{\small
\begin{align} \label{eq:uniform-rate}
\Omega\left(\frac{\dim(\cH)}{n}\right)\le \inf_{\wh h_n}\sup_{P\in\RE(\cH)}\bE[\er_P(\wh h_n)]\le O\left(\frac{\dim(\cH)^{3/2}\log^2(n)}{n}\right),
\end{align}
}
where the upper bound can be derived from the proof of \citet[Theorem 1]{brukhim2022characterization} (see Corollary \ref{coro:er_multiclass}) and the lower bound can be found in \citet{daniely2014optimal}.
However, the worst-case analysis of PAC learning is too pessimistic to reflect many practical machine learning scenarios where  the sample distribution keeps unchanged with the increase of the sample size, resulting in much faster decay in the error rate.
Thus, \citet{universal_learning} proposed the concept of universal learning to characterize the distribution-dependent universal error rate of a hypothesis class. We state the definition of universal rates below. 
\begin{definition}[Universal rate, \citealt{universal_learning}, Definition 1.4] \label{def:universal-rates}
Let $\cH$ be a hypothesis class. Let $R:\N\rightarrow[0,1]$ with $R(n)\rightarrow0$ be a rate function. 
\begin{itemize}
\itemsep0em
\item $\cH$ is \emph{learnable at rate $R$} if there is a learning algorithm $\wh h_n$ such that for every realizable distribution $P$, there exist $C,c>0$ for which $\bE[\er(\wh h_n)]\le CR(cn)$ for all $n$.
\item $\cH$ is \emph{not learnable at rate faster than $R$} if for every learning algorithm, there exists a realizable distribution $P$ and $C,c>0$ for which $\bE[\er(\wh h_n)]\ge CR(cn)$ for infinitely many $n$.
\item $\cH$ is \emph{learnable with optimal rate $R$} if $\cH$ is learnable at rate $R$ and $\cH$ is not learnable at rate faster than $R$.
\item $\cH$ \emph{is learnable but requires arbitrarily slow rates} if there is a learning algorithm $\wh h_n$ such that $\bE[\er(\wh h_n)]\rightarrow0$ for every realizable distribution $P$, and for every $R(n)\rightarrow0$, $\cH$ is not learnable faster than $R$.
\end{itemize}
\end{definition}
Note that in Definition \ref{def:universal-rates}, we define ``$\cH$ is learnable but requires arbitrarily rates'' instead of defining ``$\cH$ requires arbitrarily rates'' \citep[Definition 1.4]{universal_learning} to emphasize the existence of a universally consistent learning algorithm for $\X$ being Polish and $\Y$ being countable \citep{10.1214/20-AOS2029}. Thus, the case that $\cH$ is not universally learnable does not exist. 
As is formalized in the definition, the term ``universal'' refers to the requirement that the rate function $R$ is universal for all realizable distributions. The major difference between universal rates and uniform rates is that the constants $c$ and $C$ can depend on the distribution $P$ for universal rates, while the constants must be distribution-independent (i.e., uniform) for uniform rates. As is depicted in \citet[Figure 1]{universal_learning}, the distinction may results in the collapsing of exponential universal rates to linear uniform rates; e.g., in Example \ref{eg:exp}, we provide an example in multiclass learning where an exponential universal rate is achieved by the proposed algorithm, which is much faster than the linear uniform rate for finite label spaces.
\citet{universal_learning} successfully characterized the fined-grained trichotomy in the optimal universal rates of binary classification problems, which motivates us to study the characterization of universal rates in multiclass learning with potentially infinite label spaces.

\subsection{Main results} \label{sec:main-results}
In this section, we state the main results together with some key definitions. 
First, we rule out some trivial hypothesis classes by considering $\cH$ that is ``nondegenerate'' specified in the following definition. 
\begin{definition}[Nondegenerate hypothesis class]
A hypothesis class $\cH\in\Y^{\X}$ is called \emph{nondegenerate} if there exist $h_1,h_2\in\cH$ and $x,x'\in\X$ such that $h_1(x)=h_2(x)$ and $h_1(x')\neq h_2(x')$. $\cH$ is called \emph{degenerate} if it is not nondegenerate.
\end{definition}
Indeed, for $\cH$ that is degenerate, if $h_1,h_2\in\cH$ satisfy $h_1\neq h_2$, then, we have $h_1(x)\neq h_2(x)$ for any $x\in\X$. Thus, one sample suffices to reach zero error rate under any realizable distributions.
 
For the measurability of the learning algorithms we design in this paper, we need the following definition regarding the measurability of the hypothesis class $\cH$.
\begin{definition}[Measurable hypothesis class, \citealt{universal_learning}, Definition 3.3] \label{def:measurable}
A hypothesis class $\cH$ of functions $h: \X\rightarrow\Y$ on Polish spaces $\X$ and $\Y$ is said to be \emph{measurable} if there is a Polish space $\Theta$ and a Borel-measurable map $\sfh:\Theta\times\X\rightarrow\Y$ so that $\cH=\{\sfh(\theta,\cdot):\theta\in\Theta\}$.
\end{definition}
As is discussed in \citet{universal_learning}, the above definition is standard in the literature and almost any $\cH$ considered in practice is measurable.
\citet{universal_learning} and \citet{kalavasis2022multiclass} also assume measurable hypothesis classes in their results.

The following theorem depicts the trichotomy in the universal rates of multiclass learning for general countable label spaces.
\begin{theorem} \label{thm:rates}
For any nondegenerate measurable hypothesis class $\cH$, exactly one of the following holds:
\begin{itemize}
\item $\cH$ is learnable with optimal rate $e^{-n}$.
\item $\cH$ is learnable with optimal rate in $\wt\Theta(1/n)$.
\item $\cH$ is learnable but requires arbitrarily slow rates.
\end{itemize}
\end{theorem}

Then, we characterize the complexity measures of $\cH$ that determine the universal rates of it: the nonexistence of certain tree structures of $\cH$. 
We start with the Littlestone tree defined below.
\begin{definition}[Littlestone tree] \label{def:littlestone_tree}
A \emph{Littlestone tree} for $\cH\subseteq \Y^\X$ is a complete binary tree of depth $d\le\infty$ whose internal nodes are labelled by $\X$, and whose two edges connecting a node to its two children are labelled by two different labels from $\Y$, such that every finite path emanating from the root is consistent with a concept $h\in\cH$.
\end{definition}
Equivalently, a Littlestone tree of depth $d\le\infty$ for $\cH$ can also be represented as a collection 
\begin{align} \label{eq:littlestone}
\left\{(x_\bu,y_\bu^0,y_\bu^1)\in\wt\X:\bu\in\{0,1\}^k,0\le k< d\right\}\subseteq\wt \X:=\{(x,y,y')\in\X\times\Y^2:y\neq y'\}
\end{align}
such that for any $\bfeta\in\{0,1\}^d$ and $0\le n<d$, there exists a concept $h\in\cH$ such that $h(x_{\bfeta_{\le k}})=y_{\bfeta_{\le k}}^{\eta_{k+1}}$ 
for each $0\le k\le n$, where $\bfeta_{\le k}:=(\eta_1,\dots,\eta_k)$. 
We say that $\cH$ has an \emph{infinite Littlestone tree} if there is a Littlestone tree for $\cH$ of depth $d=\infty$. 

The definition of the Littlestone tree was first proposed by \citet{daniely2015multiclass} to generalize the Littleston dimension to multiclass hypothesis classes, where they assume that $\X$ and $\Y$ are countable. \citet{universal_learning} restricted the definition to binary hypothesis classes and emphasized the difference between having of an infinite Littlestone tree and having an infinite Littlestone dimension (i.e., having Littlestone trees of arbitrarily large depth), where they prove that the nonexistence of the former distinguishes the exponential rate and the linear rate. 
\citet{kalavasis2022multiclass} restricted the definition to multiclass hypothesis classes with finite label spaces ($|\Y|<\infty$) and proved that the nonexistence of an infinite Littlestone tree distinguishes the exponential rate and the linear rate for finite $\Y$.

Next, we introduce a new tree structure, the Daniely-Shalev-Shwartz-Littleston (DSL) tree which builds on the concept of pseudo-cubes in the definition of the DS dimension \citep{brukhim2022characterization}. For completeness, we state the definition of pseudo-cubes below.
\begin{definition}[Pseudo-cube, \citealt{brukhim2022characterization}, Definition 5]
For any $d\in\N$, a class $C\subseteq\Y^d$ is called a \emph{pseudo-cube} of dimension $d$ if it is non-empty, finite, and for every $h\in C$ and $i\in[d]$, there is an $i$-neighbor of $g\in C$ of $h$ (i.e., $g(i)\neq h(i)$ and $g(j)=h(j)$ for all $j\in[d]\backslash\{i\}$). 
\end{definition}
For any $d\in\N$ and hypothesis class $H\subseteq\Y^d$, let $\PC(H)$ denote the collection of all $d$-dimensional pseudo-cubes contained in $H$. Then, we provide the definition of DSL trees below.
\begin{definition}[DSL tree]
A \emph{DSL tree} for $\cH\subseteq\Y^{\X}$ of depth $d\le\infty$ is a tree of depth $d$ satisfying the following properties.
\begin{itemize}
\itemsep0em
\item For each integer $k$ such that $0\le k<d$ and each node $v$ in level $k$ of the tree (assume that the level of the root node is 0), node $v$ is labelled with some $\bx\in\X^{k+1}$. Moreover, there exists some pseudo-cube $C\in\PC(\cH|_{\bx})$ such that node $v$ has exactly $|C|$ children and each edge connecting node $v$ to its children is labelled with a unique element in $C$.
\item For each integer $k$ such that $0\le k<d$ and each node $v$ in level $k$, denote the label of $v$ with $\bx_k\in\X^{k+1}$. Denote the labels of the nodes and the labels of the edges along the path emanating from the root node to node $v$ with $\bx_0\in\X^1,\dots,\bx_{k-1}\in\X^{k}$ and $\by_0\in\Y^1,\dots,\by_{k-1}\in\Y^{k}$ correspondingly. Denote the number of the children of node $v$ with $n$ and the labels of the edges connecting node $v$ to its children with $\by_{k,1},\dots,\by_{k,n}\in\Y^{k+1}$. Then, for each $i\in[n]$, there exists some $h\in\cH$ such that $h|_{\bx_t}=\by_t$ for all $0\le t\le k-1$ and $h|_{\bx_{k}}=\by_{k,i}$.
\end{itemize}
\end{definition}
Similarly, we say that $\cH$ has an \emph{infinite DSL tree} if there is a DSL tree for $\cH$ of depth $d=\infty$. 
The definition of the DSL tree resembles those of the VCL tree \citep[Definition 1.8]{universal_learning}, the NL tree \citep[Definition 6]{kalavasis2022multiclass}, and the GL tree \citep[Definition 8]{kalavasis2022multiclass}. Each node in level $k$ is labelled with a sequence of $k+1$ points in $\X$ for $k\in\N_0$. However, for VCL trees and NL trees, the edges connecting a node to its children correspond to a copy of the Boolean-cube while they correspond to a pseudo-cube for DSL trees. Thus, the structure of a DSL tree is much more complicated since the sizes of pseudo-cubes of fixed dimension are not fixed, and it is hard to directly formulate a DSL tree like VCL trees or NL trees. 
For completeness and future reference, we state the definitions of the NL tree and the GL tree in Appendix \ref{sec:add-def}. 

Now, we are ready to present the characterization of the multiclass universal rates in terms of those definitions.
\begin{theorem} \label{thm:tree-rates}
For any nondegenerate measurable hypothesis class $\cH$, the followings hold:
\begin{itemize}
\item If $\cH$ does not have an infinite Littlestone tree, then $\cH$ is learnable with optimal rate $e^{-n}$.
\item If $\cH$ has an infinite Littlestone tree but does not have an infinite DSL tree, then $\cH$ is learnable at rate $\frac{\log^2 n}{n}$ and is not learnable at rate faster than $\frac{1}{n}$.
\item If $\cH$ has an infinite DSL tree, then $\cH$ is learnable but requires arbitrarily slow rates.
\end{itemize}
\end{theorem}
Since Theorem \ref{thm:rates} follows immediately from Theorem \ref{thm:tree-rates}, we directly prove Theorem \ref{thm:tree-rates} in this paper.
A major difference between Theorem \ref{thm:tree-rates} and \citet[Theorem 2]{kalavasis2022multiclass} lies in the complexity measure that distinguishes the (near-)linear rate and arbitrarily slow rates: \citet[Theorem 2]{kalavasis2022multiclass} uses the nonexistence of an infinite NL tree. 
Then, a natural question is whether having an infinite DSL tree is equivalent to having an infinite NL tree for $\cH\subseteq\Y^{\X}$ with $|\Y|=\infty$. 
Generalizing \citet[Theorem 2]{brukhim2022characterization}, we are able to show that they are not equivalent even for countably infinite $\X$ and $\Y$ in the following theorem. 
\begin{theorem} \label{thm:NL-counter}
There exist some countable sets $\X$ and $\Y$, and a hypothesis class $\cH\subseteq \Y^{\X}$ such that $\cH$ has an infinite DSL tree but does not have any NL tree of depth 2. 
\end{theorem}
Thus, the nonexistence of an infinite NL tree does not distinguish the near-linear rate and arbitrarily slow rates for infinite label space ($|\Y|=\infty$). 

We briefly comment on the $\log^2 n$ gap between the upper and lower bounds of the optimal universal rate in the second case (i.e., $\cH$ has an infinite Littlestone tree but does not have an infinite DSL tree) of Theorem \ref{thm:tree-rates}.
It is worth pointing out that the $\frac{\log^2 n}{n}$ universal rate follows from the $\frac{\log^2 n}{n}$ uniform rate in \eqref{eq:uniform-rate}. 
In fact, we prove in Theorem \ref{thm:universal_rate_DSL} that roughly speaking, a learning algorithm achieving some uniform rate for hypothesis classes with finite DS dimensions implies a learning algorithm achieving the same universal rate for any hypothesis class that does not have an infinite DSL tree. The $\frac{\log^2 n}{n}$ rate proved in \citet{brukhim2022characterization} is currently the sharpest uniform rate to our knowledge, and a sharper uniform rate will narrow the gap between the upper and lower bounds of the optimal universal rate. 
Nevertheless, the gap may also be narrowed by improving the lower bound. 
We list this problem as a future direction in Section \ref{sec:future}. 

Furthermore, we solve the first question in \citet[Open question 1]{kalavasis2022multiclass} which asks whether the existence of an infinite NL tree is equivalent to the existence of an infinite GL tree for finite label spaces ($|\Y|<\infty$). We prove that it is equivalent in the following theorem.
\begin{theorem} \label{thm:NL-GL}
Let $K\in\N\backslash\{1\}$, and let $\cH\subseteq [K]^{\X}$. Then, $\cH$ has an infinite NL tree if and only if it has an infinite GL tree.
\end{theorem}
Since it is not hard to see from definitions that a NL tree for $\cH$ can be converted into a DSL tree for $\cH$ of the same depth, and a DSL tree for $\cH$ can be converted into a GL tree for $\cH$ of the same depth, we immediately obtain the following corollary for $|\Y|<\infty$. 
\begin{corollary}
If $|\Y|<\infty$, then for any $\cH\subseteq\Y^{\X}$, the followings are equivalent: 
\begin{itemize}
\item $\cH$ has an infinite NL tree.
\item $\cH$ has an infinite DSL tree.
\item $\cH$ has an infinite GL tree.
\end{itemize}
\end{corollary}
Thus, the term ``infinite Natarajan-Littlestone tree'' in \citet[Theorem 2]{kalavasis2022multiclass} can be replaced with ``infinite DSL tree'' or ``infinite GL tree''.

\subsection{Future direction} \label{sec:future}
There are three immediate future directions following our current results in this paper. 
The first direction is to bridge the gap between the near-linear upper bound and linear lower bound of the optimal universal rate for hypothesis classes that have an infinite Littlestone tree but do not have an infinite DSL tree. 
As is already pointed out, tighter analysis of the uniform rate for hypothesis classes with finite DS dimensions would directly help in solving this problem. 
The second direction is to analyze the universal rates for uncountable label spaces. We believe that the major difficulty lies in proving the universal measurability of the learning algorithm constructed, and establishing the existence of a universally measurable learning algorithm that is universally consistent for general uncountable label spaces would shed light on this problem. 
Finally, it is an important next step to extend the results to the agnostic setting.

\section{Examples} \label{sec:examples}
In this section, we present three examples in multiclass learning with different universal rates. 
\begin{example}[Multiclass linear classifier on $\N^d$] \label{eg:exp}
For $d\in\N$, $K\in\N\backslash\{1\}$, $\X=\N^d$, and $\Y=[K]$, consider the following hypothesis class
{\small
\begin{align} \notag
\cH:=\Big\{&\X\rightarrow\Y,\bx\mapsto
\max(\argmax_{k\in[K]}\bw_{k}\cdot\bx-b_k):
\\& \label{eq:H_linear}
\bw_1=\textbf{\textup{0}}, (\bw_k)_j\le (\bw_{k+1})_j,\ \forall\ k\in[K],j\in[d],\ (b_1,\dots,b_K)\in(0,\infty)^K
\Big\}.
\end{align}
}
Consider any sequence $((\bx_i,y_i))_{i\in\N}\in (\X\times\Y)^\infty$ that is consistent with $\cH$; i.e., for any $n\in\N$ and $S_n:=((\bx_i,y_i))_{i\in[n]}$, there exists some $h_n\in\cH$ with $h_n(\bx_i)=y_i$ for all $i\in[n]$. 
For any $n\in\N$ and $\bx\in\X$, we define the set 
\begin{align*}
\mathsf{Y}_{S_n,\bx}:=\left\{k\in[K]:\exists \bz'\in[0,\infty)^d\ \textup{such that } \bx-\bz'\in \conv(\{\bx_i:(\bx_i,k)\in S_n, i\in[n]\})\right\}
\end{align*}
where $\conv(\emptyset):=\emptyset$ and for any $t\in\N$ and set $\{\bz_1,\dots,\bz_t\}\subseteq\X$, 
{\small
$$
\conv(\{\bz_1,\dots,\bz_t\}):=\left\{\sum_{i=1}^t\alpha_i\bz_t:(\alpha_1,\dots,\alpha_t)\in[0,1]^t,\sum_{i=1}^t\alpha_i=1\right\}
$$
}
denotes the convex hull of the set $\{z_1,\dots,z_t\}$.
Then, we define the data-dependent classifier $\wh h_n:\X\rightarrow\Y$ by
{\small
\begin{align} \label{eq:hn_eg}
\wh h_n(\bx):= 
\begin{cases}
\min\mathsf{Y}_{S_n,\bx},\quad \textup{if } \mathsf{Y}_{S_n,\bx}\neq\emptyset,
\\ 
1,\quad\textup{otherwise}.
\end{cases}
\end{align}
}
We prove the following proposition in Appendix \ref{sec:claim_eg}. 
\begin{proposition} \label{prop:hn_eg}
$(\wh h_n)_{n\in\N}$ defined in \eqref{eq:hn_eg} only makes finitely many mistakes for any consistent sequence $((\bx_n,y_n))_{n\in\N}$. Moreover, if $\wh h_n(\bx_{n+1})=y_{n+1}$, then we have $\wh h_{n+1}=\wh h_n$.
\end{proposition}
Thus, by the construction and proofs given in \citet[Section 4.1]{universal_learning}, such an adversarial algorithm implies an online learning algorithm with exponential rate. 
By Theorem \ref{thm:tree-rates}, $\cH$ is learnable with optimal rate $e^{-n}$ and $\cH$ does not have an infinite Littlestone tree.
\end{example}

\begin{example}[Multiclass linear classifier on $\reals^d$] \label{eg:linear}
For $d\in\N$, $K\in\N\backslash\{1\}$, $\X=\reals^d$, and $\Y=[K]$, consider the hypothesis class $\cH$ defined by \eqref{eq:H_linear}. 
Notice that the class of threshold functions constructed in \citet[Example 2.2]{universal_learning} can be obtained from $\cH$ by restricting $(\bw_k)_1=1$, $(\bw_k)_j=0$, and $b_k=b_K$ for all $j\in[d]\backslash\{1\}$ and $k\in[K]\backslash\{1\}$. 
Thus, $\cH$ has an infinite Littlestone tree. 
By \citet[Theorem 7]{daniely2014optimal}, we have that $\dim_N(\cH)<\infty$. 
By \citet{BENDAVID199574,daniely2014optimal}, we have 
\begin{equation} \label{eq:dim-bd}
\dim(\cH)\le\dim_G(\cH)\le 5\log_2(K)\dim_N(\cH)
\end{equation}
which actually holds for any hypothesis class. 
It follows that $\dim(\cH)<\infty$ and $\cH$ does not have an infinite DSL tree. 
Then, by Theorem \ref{thm:tree-rates}, $\cH$ is learnable with optimal rate in $\wt{\Theta}(\frac{1}{n})$.

\end{example}

\begin{example}[A class with an infinite DSL tree but no NL tree of depth 2] \label{eg:slow}
Theorem \ref{thm:NL-counter} guarantees the existence of a hypothesis class $\cH$ that has an infinite DSL tree but does not have any NL tree of depth 2 (see the proof of Theorem \ref{thm:NL-counter} in Appendix \ref{sec:proof-NL-counter} for the construction of $\cH$).
Then, by Theorem \ref{thm:tree-rates}, $\cH$ is learnable but requires arbitrarily slow rates.

\end{example}

\section{Technical Overview} \label{sec:tech}
In this section, we briefly discuss some key technical points in the proofs of our main results. 
\subsection{Exponential rates}
We sketch the proof of the following theorem in this subsection.  
\begin{theorem} \label{thm:tree-exp-rate}
For any nondegenerate measurable hypothesis class $\cH$, if $\cH$ does not have an infinite Littlestone tree, then $\cH$ is learnable with optimal rate $e^{-n}$.
\end{theorem}
The complete proof is provided in Appendix \ref{sec:exp_rate_proof}.
Since $\cH$ is nondegenerate, according to \citet[Lemma 4.2]{universal_learning} and its proof, we can show that $\cH$ is not learnable at rate faster than the exponential rate $e^{-n}$. 
The main point of the proof is to construct a learning algorithm that achieves the exponential universal rate if $\cH$ does not have an infinite Littlestone tree. 
We follow the framework in \citet{universal_learning} for the construction. 
First, we consider an adversarial online learning game $\wb \B$ played in rounds between an adversary $\wb P_a$ and a learner $\wb P_l$ defined in Appendix \ref{sec:adv_gam_exp}. 
If we prove that for $\cH$ that does not have an infinite Littlestone tree, there exists a universally measurable strategy for the learner $\wb P_l$ in $\wb \B$ that only makes finitely many mistakes against any adversary $\wb P_a$ and only changes its prediction function when a mistake happens, then by the analysis in \citet[Section 4.1]{universal_learning}, there is a learning algorithm that achieves the exponential universal rate. 

From \eqref{eq:littlestone}, we can naturally relate Littlestone trees to the following adversarial game $\B$ between two players $P_A$ and $P_L$. In each round $\tau\in\N$: 
\begin{itemize}
\item Player $P_A$ chooses a three-tuple $\xi_\tau=(x_\tau,y_\tau^{0},y_\tau^{1})\in\wt \X$ and shows it to Player $P_L$.
\item Player $P_L$ chooses a point $\eta_\tau\in\{0,1\}$.
\end{itemize}
Player $P_L$ wins the game in round $\tau\in\N$ if $\cH_{\xi_1(1),\xi_1(\eta_1+2),\dots,\xi_\tau(1),\xi_\tau(\eta_\tau+2)}=\emptyset$ (see Appendix \ref{sec:notation} for explanations of notation).
Player $P_A$ wins the game if the game continues indefinitely. 
We prove in Lemma \ref{lem:PA_winning} that a winning strategy of $P_A$ is equivalent to an infinite Littlestone tree of $\cH$. 
According to \citet[Theorem B.1]{universal_learning}, $P_L$ has a universally measurable winning strategy if $\cH$ has no infinite Littlestone tree.
However, this winning strategy cannot be directly applied for the construction of a strategy for $\wb P_l$ in $\wb \B$ as in \citet[Section 3.2]{universal_learning} because $P_A$ chooses two labels $y_\tau^0$ and $y_\tau^1$ in each round $\tau$ while $\wb P_a$ does not provide this information (for the binary case, $\{y_\tau^0,y_\tau^1\}$ is trivially $\{0,1\}$). 

We tackle this problem by first defining the value function on the positions of $\wb\B$, which extends the value function defined on the positions of $\B$ (see Section \ref{sec:adv_gam_exp} for the terminologies and definitions).
Then, by \citet[Proposition B.8]{universal_learning}, for each round in $\wb \B$, whatever point $\wb P_a$ picks, there is at most one point in $\Y$ such that the value function does not decrease. 
Then, we can define the function \eqref{eq:gt-exp} which informally speaking, maps the current position and a point $x\in\X$ to the point in $\Y$ that does not decrease the value function. 
For Polish $\X$, countable $\Y$, and measurable $\cH$, we prove that this function is universally measurable.
Moreover, when $\cH$ has no infinite Littlestone tree, we can prove that there is no infinite value-decreasing sequence of positions by the well-ordering of the ordinals \citep{karel2017introduction}. Then, by playing the strategy induced from that defined function, $\wb P_l$ will only make finitely many mistakes because otherwise there will be an infinite value-decreasing sequence.

\subsection{Near-linear rates}
In this subsection, we sketch the proof of the following theorem. 
\begin{theorem} \label{thm:near-linear-rates}
For any nondegenerate measurable hypothesis class $\cH$, if $\cH$ has an infinite Littlestone tree but does not have an infinite DSL tree, then $\cH$ is learnable at rate $\frac{\log^2 n}{n}$ and is not learnable at rate faster than $\frac{1}{n}$.
\end{theorem}
The complete proof is provided in Appendix \ref{sec:linear_rate_proof}. 
The fact that $\cH$ is not learnable at rate faster than $\frac{1}{n}$ if it has an infinite Littlestone tree can be proved by generalizing the techniques used in the proof of \citet[Theorem 4.6]{universal_learning}. 
The key difficulty is to construct a learning algorithm that achieves $\frac{\log^2 n}{n}$ universal rate when $\cH$ does not have an infinite DSL tree. 
As is discussed in Section \ref{sec:main-results}, we show in Theorem \ref{thm:universal_rate_DSL} that a learning algorithm achieving some uniform rate for any hypothesis class with a finite DS dimension implies a learning algorithm achieving the same universal rate for any hypothesis class that does not have an infinite DSL tree. 
Since a learning algorithm that achieves $O(\frac{\dim(H)^{3/2}\log^2(n)}{n})$ uniform error rate for any hypothesis class $H\subseteq\Y^{\X}$ has been constructed \citep{brukhim2022characterization}, it suffices to prove Theorem \ref{thm:universal_rate_DSL}. 
We follow the framework in \citet[Section 5]{universal_learning}. 
Similar to the case of exponential rates, we relate that the DSL tree to the following game $\fB$ between player $P_A$ and $P_L$. At each round $\tau\in\N$:
\begin{itemize}
\item Player $P_A$ chooses a sequence $\bx_{\tau}=(x_{\tau}^0,\dots,x_{\tau}^{\tau-1})\in\X^{\tau}$ and a set  $C_{\tau}\in\PC(\Y^{\tau})$. 
\item Player $P_L$ chooses a sequence $\by_{\tau}=(y_{\tau}^0,\dots,y_{\tau}^{\tau-1})\in \Y^{\tau}$. 
\end{itemize}
Player $P_L$ wins the game in round $\tau$ if 
\begin{itemize}
\item either $C_\tau \notin\PC(\cH|_{\bx_\tau})$
\item or $\by_{s}\in C_s$ for all $1\le s\le \tau$ and $\cH_{\bx_1,\by_1,\dots,\bx_{\tau},\by_{\tau}}=\emptyset$, where
$$
\mathcal{H}_{\bx_{1}, \by_{1}, \ldots, \bx_{\tau}, \by_{\tau}}:=\left\{h \in \mathcal{H}: h\left(x_{s}^{i}\right)=y_{s}^{i} \text { for } 0 \leq i<s, 1 \leq s \leq \tau\right\}.
$$
\end{itemize}
Player $P_A$ wins the game if the game continues indefinitely. We emphasize the subtlety in the winning rule of $P_L$. In this way, we can ensure that $\fB$ is a Gale-Stewart game and an infinite DSL tree is equivalent to a winning strategy for $P_A$ (Lemma \ref{lem:PA_winning_DSL}). 
Similar to the analysis of exponential rates, there exists a mismatch between a winning strategy of $P_L$ and a ``pattern avoidance function'' required in the template for constructing learning algorithms in the probabilistic setting in \citet[Section 5.2]{universal_learning}: in the adversarial learning problem, the adversary does not provide a pseudo-cube as $P_A$ does.
Thus, it is tricky to construct pattern avoidance functions which successfully rule out label patterns from their mappings of the feature patterns for any $\cH$-consistent sequence in a finite number of steps, and keeps unchanged after the success. 
We provide our definition of pattern avoidance functions in \eqref{eq:pattern-avoid-func}. Informally, for a consistent sequence, given the current position in $\fB$ as well as the current feature pattern and label pattern from the sequence, we traverse all pseudo-cubes contained in the projection of $\cH$ on the feature pattern, where by a feature (label) pattern we refer to a consecutive subsequence of the feature (label) sequence ending at the current point. If the value function defined on positions in $\fB$ deceases after adding the feature pattern, the current pseudo-cube, and label pattern into the position, we accept this new position, proceed one round in $\fB$, and stop the traverse. If the value function never decreases after the traverse, we still use the original position and does not change the round in $\fB$. Then, the feature pattern and label pattern are updated accordingly. 
Now, we define the current pattern avoidance function as the mapping from the current position and feature pattern to the set of all label patterns for which the position will be updated after traversing all the pseudo-cubes in the projection of $\cH$ on the feature pattern. 
Then, with the similar idea of showing contradiction with nonexistence of infinite value-decreasing sequences, we can prove the desired pattern avoidance property of the set functions we defined.
The next step is to show the universal measurability.
Unlike the pattern avoidance function in \citet{universal_learning,kalavasis2022multiclass}, our pattern avoidance function maps to a set of patterns. 
This increases the difficulty in proving the universal measurability of the pattern avoidance functions we define since then we need to pay attention to the topology on power sets. 
One key point to notice is that since pseudo-cubes are finite by definition, $\PC(\Y^\tau)$ is countable as the set of finite subsets of a countable set is also countable. We can use this point to show that certain sets served as the building blocks in the pull-back set of the pattern avoidance functions are analytic. 
We also note that the universal measurability of the winning strategy for $P_L$ or some value-decreasing function defined in $\fB$ does not obviously imply the universal measurability of the pattern avoidance functions since there are repetitions when feeding the data sequence as inputs to the game $\fB$; both \citet{universal_learning} and \citet{kalavasis2022multiclass} does not provide a proof for this step \citep[Remark 5.4]{universal_learning}. 
Thus, we provide an explicit and complete derivation 
of the universal measurability of the pattern avoidance functions that covers this step in our even more complicated setting, which also turns out to be very tricky. 

There are still several big technical obstacles in plugging the pattern avoidance functions and a learning algorithm $A$ with a uniform rate guarantee for hypothesis classes of finite DS dimensions into the template algorithm in \citet[Section 5.2]{universal_learning}. 
We first upper bound the DS dimension of the hypothesis class \eqref{eq:H_per} constructed through a pattern avoidance function with its length (i.e., the length of the pattern the function seeks to avoid) in Lemma \ref{lem:DS_dim_subclass}, where informally, \eqref{eq:H_per} consists of projections of hypotheses in $\cH$ on a given feature sequence such that any ordered subsequence of the projection is avoided by the pattern avoidance function. 
Then, we prove in Lemma \ref{lem:realizable_subclass} that informally, the uniform distribution over an independent and identically distributed (i.i.d.) data sequence that defines \eqref{eq:H_per} is realizable with the class \eqref{eq:H_per} almost surely if the pattern avoidance function that defines \eqref{eq:H_per} avoids an i.i.d. data sequence with probability 1. 
Then, we would like to apply the the uniform learning algorithm $A$ to \eqref{eq:H_per} with that uniform distribution as the realizable distribution. 
However, for the usage of $A$ in the template specified in Theorem \ref{thm:universal_rate_DSL}, informally, given a sequence $(X_1,Y_1,\dots,X_n,Y_n)$ and a feature $X_{n+1}$, the training data for $A$ are drawn from the uniform distribution only over $\{(1,Y_1),\dots,(n,Y_n)\}$ as we do not know $Y_{n+1}$, but the test data is always fixed, i.e., $n+1$. 
In Lemma \ref{lem:partial_sample}, we upper bound the error rate in this setting with twice the error rate in the standard setting (i.e., both the training data and the test data are drawn i.i.d. from the uniform distribution only over $\{(1,Y_1),\dots,(n+1,Y_{n+1})\}$ with $Y_{n+1}$ being the label of $X_{n+1}$). 
Similar to Theorem \ref{thm:universal_rate_DSL}, Lemma \ref{lem:partial_sample} is interesting in itself for dealing with partial training distributions.

\subsection{Arbitrarily slow rates}
In this subsection, we sketch the proof of the following theorem.
\begin{theorem} \label{thm:arb_slow}
If $\cH$ has an infinite DSL tree, then $\cH$ is learnable but requires arbitrary slow rates. 
\end{theorem}
The complete proof is provided in Appendix \ref{sec:arb_slow}. 
The proof follows the framework for the construction of distributions in \citet[Theorem 5.11]{universal_learning}. Since for DSL trees, the numbers of the children of the nodes are not fixed in each level, to even formulate a uniform distribution over the paths in the infinite DSL tree is non-trivial. 
The key for the proof is to show \eqref{eq:neq-prop}, which holds trivially for both VCL trees \citep{universal_learning} and NL trees \citep{kalavasis2022multiclass} since the labels of the edges connecting a node to its children consist a copy of the Boolean cube. 
However, such result for pseudo-cubes is novel; it actually implies an elegant proof for the $\Omega(\frac{\dim(\cH)}{n})$ lower bound of the uniform rate in \eqref{eq:uniform-rate}. 
There are two key steps to show \eqref{eq:neq-prop}. We first prove that for any pseudo-cube, any position, and any label, the proportion of hypotheses in the pseudo-cube  that maps that position to that label is at most half. 
Then, we prove that when restricting some arbitrary positions to some arbitrary pattern, a pseudo-cube, projected to the unrestricted positions, is still a pseudo-cube. 
Both steps follow from careful examination of the definition of pseudo-cubes.

Now, Theorem \ref{thm:tree-rates} directly follows from Theorem \ref{thm:tree-exp-rate}, Theorem \ref{thm:near-linear-rates}, and Theorem \ref{thm:arb_slow}.

\subsection{Proof sketch of Theorem \ref{thm:NL-counter}}
The complete proof of Theorem \ref{thm:NL-counter} is provided in Appendix \ref{sec:proof-NL-counter}.
We use the disjoint pseudo-cubes of all dimensions on disjoint finite label spaces constructed in the proof of \citet[Theorem 2]{brukhim2022characterization} as our starting point.
We first build an infinite complete tree using these pseudo-cubes as blocks and take the disjoint unions to construct a countable label space, a countable feature space, and a hypothesis class. Then, we add to the label space a unique new element $\star$ used for extending the domain of a hypothesis to the whole feature space. Specifically, in a top-down manner of the tree constructed, we extend the definition of a hypothesis which corresponds to an edge in the tree to be consistent with the the hypotheses in the path eliminating from the root to its edge. Then, we define its value to be $\star$ on any other features. 
The fact that this class has an infinite DSL tree directly follow from the tree we constructed and the way we extend the definitions of hypotheses. 
Then, we prove that the class has a NL dimension 1 by considering the projection of the class on two arbitrary features, which requires more sophisticated discussion compared with the proof of \citet[Theorem 2]{brukhim2022characterization}.

\subsection{Proof sketch of Theorem \ref{thm:NL-GL}}
The complete proof of Theorem \ref{thm:NL-GL} is provided in Appendix \ref{sec:thmNLGL}. 
The fact that a GL tree can be obtained from a NL tree is obvious. The key is to construct an infinite NL tree from an infinite GL tree. 
For each node in the infinite GL tree except the root, we can associate it with a hypothesis in $\cH$ that witnesses the requirement of the GL tree. 
Then, the rough idea is to construct for ``each'' node a distinct new sequence of labels for which each edge between this node and its children corresponds to a unique concept in the Boolean cube formed by this new sequence and the sequence provided by the GL tree, and the associated hypothesis of each descendant of this node along the path starting with this edge is consistent with the concept of the edge on this node. 
Here, by ``each'' we do not mean to construct for each element in the infinite GL tree, but we actually mean to select a node in the infinite GL tree for each position in the infinite NL tree to build. 

We first deal with the consistency. In an infinite GL tree, for a node and an edge between the node and one of its children, the associated hypotheses of its descendants along the path starting with the chosen edge can predict differently on the chosen node, and the prediction can be used to color each descendant of the chosen node starting with the chosen edge. Then, we obtain an infinite colored subtree. 
Since $|\Y|=K<\infty$, the total number of colorings is finite. Thus, by the Milliken's tree theorem \citep{MILLIKEN1979215}, there is a strongly embedded subtree whose edges have the same color. 
But we still need to prune this subtree so that it has the same structure as the original subtree, after which we replace the original subtree with the monochromatic subtree. Now, the prediction made by the associated hypothesis of each descendant along the path starting with the chosen edge is the same on the chosen node. 
This step is formally presented in Lemma \ref{lem:GL}. 

However, we still face the fact that the predictions specified for each edge of a given node in the previous step do not necessary make a copy of a Boolean cube.
For this problem, we observe that all the predictions make a hypothesis class with its Graph dimension greater than $d$ for $d$ denoting the length of the feature sequence of the given node. 
By \eqref{eq:dim-bd}, this class has a Natarajan dimension greater than $d/(5\log_2(K))$, which implies a Boolean cube of dimension greater than $d/(5\log_2(K))$. Thus, by skipping $\lceil5\log_2(K)\rceil$ levels in choosing nodes from the infinite GL tree in a top-down manner, we are able to ensure the existence of a copy of a Boolean cube of required dimension for each level of the NL tree constructed by some proper pruning. The proof is formalized by induction.

\acks{Shay Moran is a Robert J.\ Shillman Fellow; he acknowledges support by ISF grant 1225/20, by BSF grant 2018385, by an Azrieli Faculty Fellowship, by Israel PBC-VATAT, by the Technion Center for Machine Learning and Intelligent Systems (MLIS), and by the the European Union (ERC, GENERALIZATION, 101039692). Views and opinions expressed are however those of the author(s) only and do not necessarily reflect those of the European Union or the European Research Council Executive Agency. Neither the European Union nor the granting authority can be held responsible for them.}

\bibliography{learning}

\begin{thebibliography}{20}
\providecommand{\natexlab}[1]{#1}
\providecommand{\url}[1]{\texttt{#1}}
\expandafter\ifx\csname urlstyle\endcsname\relax
  \providecommand{\doi}[1]{doi: #1}\else
  \providecommand{\doi}{doi: \begingroup \urlstyle{rm}\Url}\fi

\bibitem[Bendavid et~al.(1995)Bendavid, Cesabianchi, Haussler, and
  Long]{BENDAVID199574}
S.~Bendavid, N.~Cesabianchi, D.~Haussler, and P.M. Long.
\newblock Characterizations of learnability for classes of \{0, ..., n\}-valued
  functions.
\newblock \emph{Journal of Computer and System Sciences}, 50\penalty0
  (1):\penalty0 74--86, 1995.
\newblock ISSN 0022-0000.
\newblock \doi{https://doi.org/10.1006/jcss.1995.1008}.
\newblock URL
  \url{https://www.sciencedirect.com/science/article/pii/S0022000085710082}.

\bibitem[Blumer et~al.(1989)Blumer, Ehrenfeucht, Haussler, and
  Warmuth]{10.1145/76359.76371}
Anselm Blumer, A.~Ehrenfeucht, David Haussler, and Manfred~K. Warmuth.
\newblock Learnability and the vapnik-chervonenkis dimension.
\newblock \emph{J. ACM}, 36\penalty0 (4):\penalty0 929–965, oct 1989.
\newblock ISSN 0004-5411.
\newblock \doi{10.1145/76359.76371}.
\newblock URL \url{https://doi.org/10.1145/76359.76371}.

\bibitem[Bousquet et~al.(2021)Bousquet, Hanneke, Moran, van Handel, and
  Yehudayoff]{universal_learning}
Olivier Bousquet, Steve Hanneke, Shay Moran, Ramon van Handel, and Amir
  Yehudayoff.
\newblock A theory of universal learning.
\newblock In \emph{Proceedings of the 53rd Annual ACM SIGACT Symposium on
  Theory of Computing}, STOC 2021, page 532–541, New York, NY, USA, 2021.
  Association for Computing Machinery.
\newblock ISBN 9781450380539.
\newblock \doi{10.1145/3406325.3451087}.
\newblock URL \url{https://doi.org/10.1145/3406325.3451087}.

\bibitem[Brukhim et~al.(2022)Brukhim, Carmon, Dinur, Moran, and
  Yehudayoff]{brukhim2022characterization}
Nataly Brukhim, Daniel Carmon, Irit Dinur, Shay Moran, and Amir Yehudayoff.
\newblock A characterization of multiclass learnability.
\newblock In \emph{2022 IEEE 63rd Annual Symposium on Foundations of Computer
  Science (FOCS)}, pages 943--955. IEEE, 2022.

\bibitem[Daniely and Shalev-Shwartz(2014)]{daniely2014optimal}
Amit Daniely and Shai Shalev-Shwartz.
\newblock Optimal learners for multiclass problems.
\newblock In \emph{Conference on Learning Theory}, pages 287--316. PMLR, 2014.

\bibitem[Daniely et~al.(2015)Daniely, Sabato, Ben-David, and
  Shalev-Shwartz]{daniely2015multiclass}
Amit Daniely, Sivan Sabato, Shai Ben-David, and Shai Shalev-Shwartz.
\newblock Multiclass learnability and the erm principle.
\newblock \emph{Journal of Machine Learning Research}, 16:\penalty0 2377--2404,
  2015.

\bibitem[Dietmann and Holm(2001)]{dietmann2001identification}
Sabine Dietmann and Liisa Holm.
\newblock Identification of homology in protein structure classification.
\newblock \emph{nature structural biology}, 8\penalty0 (11):\penalty0 953--957,
  2001.

\bibitem[Gale and Stewart(1953)]{MR0054922}
David Gale and F.~M. Stewart.
\newblock Infinite games with perfect information.
\newblock In \emph{Contributions to the theory of games, vol. 2}, Annals of
  Mathematics Studies, no. 28, pages 245--266. Princeton University Press,
  Princeton, N.J., 1953.

\bibitem[Hanneke et~al.(2021)Hanneke, Kontorovich, Sabato, and
  Weiss]{10.1214/20-AOS2029}
Steve Hanneke, Aryeh Kontorovich, Sivan Sabato, and Roi Weiss.
\newblock {Universal Bayes consistency in metric spaces}.
\newblock \emph{The Annals of Statistics}, 49\penalty0 (4):\penalty0 2129 --
  2150, 2021.
\newblock \doi{10.1214/20-AOS2029}.
\newblock URL \url{https://doi.org/10.1214/20-AOS2029}.

\bibitem[Hellerstein and Mendelsohn(1993)]{hellerstein1993theoretical}
Daniel Hellerstein and Robert Mendelsohn.
\newblock A theoretical foundation for count data models.
\newblock \emph{American journal of agricultural economics}, 75\penalty0
  (3):\penalty0 604--611, 1993.

\bibitem[Kalavasis et~al.(2022)Kalavasis, Velegkas, and
  Karbasi]{kalavasis2022multiclass}
Alkis Kalavasis, Grigoris Velegkas, and Amin Karbasi.
\newblock Multiclass learnability beyond the pac framework: Universal rates and
  partial concept classes.
\newblock \emph{arXiv preprint arXiv:2210.02297}, 2022.

\bibitem[Karel and Thomas(2017)]{karel2017introduction}
Hrbacek Karel and Jech Thomas.
\newblock \emph{Introduction to Set Theory: Revised and Expanded}.
\newblock Crc Press, 2017.

\bibitem[Milliken(1979)]{MILLIKEN1979215}
Keith~R. Milliken.
\newblock A ramsey theorem for trees.
\newblock \emph{Journal of Combinatorial Theory, Series A}, 26\penalty0
  (3):\penalty0 215--237, 1979.
\newblock ISSN 0097-3165.
\newblock \doi{https://doi.org/10.1016/0097-3165(79)90101-8}.
\newblock URL
  \url{https://www.sciencedirect.com/science/article/pii/0097316579901018}.

\bibitem[Natarajan(1989)]{natarajan1989learning}
Balas~K Natarajan.
\newblock On learning sets and functions.
\newblock \emph{Machine Learning}, 4:\penalty0 67--97, 1989.

\bibitem[Natarajan and Tadepalli(1988)]{natarajan1988two}
Balas~K Natarajan and Prasad Tadepalli.
\newblock Two new frameworks for learning.
\newblock In \emph{Machine Learning Proceedings 1988}, pages 402--415.
  Elsevier, 1988.

\bibitem[Rawat and Wang(2017)]{8016501}
Waseem Rawat and Zenghui Wang.
\newblock Deep convolutional neural networks for image classification: A
  comprehensive review.
\newblock \emph{Neural Computation}, 29\penalty0 (9):\penalty0 2352--2449,
  2017.
\newblock \doi{10.1162/neco_a_00990}.

\bibitem[Song and Croft(1999)]{song1999general}
Fei Song and W~Bruce Croft.
\newblock A general language model for information retrieval.
\newblock In \emph{Proceedings of the eighth international conference on
  Information and knowledge management}, pages 316--321, 1999.

\bibitem[Valiant(1984)]{10.1145/1968.1972}
L.~G. Valiant.
\newblock A theory of the learnable.
\newblock \emph{Commun. ACM}, 27\penalty0 (11):\penalty0 1134–1142, nov 1984.
\newblock ISSN 0001-0782.
\newblock \doi{10.1145/1968.1972}.
\newblock URL \url{https://doi.org/10.1145/1968.1972}.

\bibitem[Vapnik and Chervonenkis(1971)]{vapnik1971uniform}
VN~Vapnik and A~Ya Chervonenkis.
\newblock On the uniform convergence of relative frequencies of events to their
  probabilities.
\newblock \emph{Theory of Probability \& Its Applications}, 16\penalty0
  (2):\penalty0 264--280, 1971.

\bibitem[Young et~al.(2018)Young, Hazarika, Poria, and Cambria]{8416973}
Tom Young, Devamanyu Hazarika, Soujanya Poria, and Erik Cambria.
\newblock Recent trends in deep learning based natural language processing
  [review article].
\newblock \emph{IEEE Computational Intelligence Magazine}, 13\penalty0
  (3):\penalty0 55--75, 2018.
\newblock \doi{10.1109/MCI.2018.2840738}.

\end{thebibliography}

\appendix

\section{Preliminaries}
In this section, we describe the notation used in this paper and present the definitions of NL trees and GL trees in the general multiclass setting.
\subsection{Notation} \label{sec:notation}
We use the following notation throughout the paper. $\N$ denotes the set of positive integers. $\N_0$ denotes the set of non-negative integers. For any $n\in\N$, we define $[n]:=\{1,\dots,n\}$. 
For any $a,b\in\reals$, we define $a\land b:=\min\{a,b\}$ and $a\lor b:=\max\{a,b\}$.
For a set $A$, $|A|$ denotes its cardinality and $2^A$ denotes its power set. 
For any sets $X,Y$ and hypothesis class $F\subseteq Y^X$, let $\dim(F)$ denote the Daniely-Shalev-Shwartz (DS) dimension of $F$, $\dim_N(F)$ denote the Natarajan dimension of $F$, and $\dim_G$ denote the Graph dimension of $F$. 
For any $n\in\N$, any sequence $S=(x_1,\dots,x_n)\in X^n$, and any function $f:X\rightarrow Y$, we define the projection of $f$ to $S$ as $f|_S:=(f(x_1),\dots,f(x_n))\in Y^n$ and use $S(i)$ to denote the $i$-th element in $S$ (i.e., $S(i)=x_i$) for any $i\in[n]$. By convention, $f|_\emptyset=\emptyset$. Then, we define the projection of $F\subseteq Y^{X}$ to $S$ as $F|_S:=\{f|_S:f\in F\}\subseteq Y^n$. 
For any $x_1,\dots,x_n\in X$ and $y_1,\dots,y_n\in Y$,  we define 
$$
F_{x_1,y_1,\dots,x_n,y_n}:=\{f\in F:f(x_1)=y_1,\dots,f(x_n)=y_n\}.
$$

\subsection{NL trees and GL trees} \label{sec:add-def}
In this section, we define NL trees and GL trees for the general multiclass setting.
\begin{definition}[NL tree]
A Natarajan-Littlestone (NL) tree for $\cH\subseteq\Y^{\X}$ of depth $d\le\infty$ is the following collection
$$
\cup_{0\le n<d}\left\{(\bx_{\bu},\bs^{(0)}_{\bu},\bs^{(1)}_{\bu})\in\X^{n+1}\times\Y^{n+1}\times\Y^{n+1}:\bu\in\prod_{l=1}^n\{0,1\}^l\right\}
$$
such that for any $0\le n<d$ and $\bu=(u_1^0,(u_2^0,u_2^1),\dots,(u_n^0,\dots,u_n^{n-1}))\in\prod_{l=1}^{n}\{0,1\}^l$, the followings hold:
\begin{itemize}
\item $s_{\bu}^{(0)i}\neq s_{\bu}^{(1)i}$ for all $0\le i\le n$. 
\item If $n\ge 1$, then there exists some $h_{\bu}\in\cH$ such that $h_{\bu}(x_{\bu_{\le l}}^i)=s_{\bu_{\le l}}^{(0)i}$ if $u_{l+1}^i=0$ and $h_{\bu}(x_{\bu_{\le l}}^i)= s_{\bu_{\le l}}^{(1)i}$ otherwise for all $0\le i\le l$ and $0\le l<n$, where
\begin{align*}
\bu_{\le l}:=(u_1^0,(u_2^0,u_2^1),\dots,(u_l^{0},\dots,u_l^{l-1})),\ 
x_{\bu_{\le l}}:=(x_{\bu_{\le 1}}^0,\dots,x_{\bu_{\le 1}}^{l}).
\end{align*}
\end{itemize}
We call 
$$
\cup_{1\le n<d}\left\{h_{\bu}\in\cH:\bu\in\prod_{l=1}^n\{0,1\}^l\right\}
$$
the associated hypothesis set of the NL tree.
We say that $\cH$ has an infinite NL tree if it has a NL tree of depth $d=\infty$.
\end{definition}

\begin{definition}[GL tree]
A Graph-Littlestone (GL) tree for $\cH\subseteq\Y^{\X}$ of depth $d\le\infty$ is the following collection
$$
\cup_{0\le n<d}\left\{(\bx_{\bu},\bs_{\bu})\in\X^{n+1}\times\Y^{n+1}:\bu\in\prod_{l=1}^n\{0,1\}^l\right\}
$$
such that for any $0\le n<d$ and $\bu=(u_1^0,(u_2^0,u_2^1),\dots,(u_n^0,\dots,u_n^{n-1}))\in\prod_{l=1}^{n}\{0,1\}^l$, the following holds:
\begin{itemize}
\item If $n\ge 1$, then there exists some $h_{\bu}\in\cH$ such that $h_{\bu}(x_{\bu_{\le l}}^i)=s_{\bu_{\le l}}^{i}$ if $u_{l+1}^i=0$ and $h_{\bu}(x_{\bu_{\le l}}^i)\neq s_{\bu_{\le l}}^{i}$ otherwise for all $0\le i\le l$ and $0\le l<n$, where
\begin{align*}
\bu_{\le l}:=(u_1^0,(u_2^0,u_2^1),\dots,(u_l^{0},\dots,u_l^{l-1})),\ 
x_{\bu_{\le l}}:=(x_{\bu_{\le 1}}^0,\dots,x_{\bu_{\le 1}}^{l}).
\end{align*}
\end{itemize}
We call 
$$
\cup_{1\le n<d}\left\{h_{\bu}\in\cH:\bu\in\prod_{l=1}^n\{0,1\}^l\right\}
$$
the associated hypothesis set of the GL tree.
We say that $\cH$ has an infinite GL tree if it has a GL tree of depth $d=\infty$.
\end{definition}

\section{Exponential Rates} \label{sec:exp_rate_proof}
In this section, we prove Theorem \ref{thm:tree-exp-rate}. 
\subsection{Adversarial learning algorithm} \label{sec:adv_gam_exp}
We propose and analyze an adversarial learning algorithm in this section.
Define $\Yp:=\{(y,y')\in\Y^2:y\neq y'\}$ and $\wt \X:=\X\times \Yp$. For any $\xi\in\wt\X$, there exist $x\in\X$ and $y^0,\ y^1\in\Y$ such that $\xi=(x,y^0,y^1)$. Then, we let $\xi(1)$ denote $x$ and $\xi(i+2)$ denote $y^i$ for $i\in\{0,1\}$. 
Then, a Littlestone tree can be equivalently represented as the following collection
$$
\left\{\xi_{\bu}=(x_\bu,y_\bu^0,y_\bu^1)\in\wt\X:\bu\in\{0,1\}^k,0\le k< d\right\}\subseteq\wt \X
$$ 
such that for any $\bfeta\in\{0,1\}^d$ and $0\le n<d$, there exists a concept $h\in\cH$ such that
$h(\xi_{\bfeta_{\le k}}(1))=\xi_{\bfeta_{\le k}}(\eta_{k+1}+2)$.

For the multiclass online learning problem, we can define the following online learning game $\wb \B$ played in rounds between an adversary $\wb P_a$ and the learner $\wb P_l$. In each round $t\ge1$:
\begin{itemize}
\item The adversary $\wb P_a$ chooses a point $x_t\in\X$.
\item The learner $\wb P_l$ makes a prediction $\wh y_t\in\Y$.
\item The adversary $\wb P_a$ reveals the true label $y_t=h(x_t)$ for some concept $h\in\cH$ such that $h$ is consistent with the previous points: $y_1=h(x_1),\ \dots,\ y_{t-1}=h(x_{t-1})$. 
\end{itemize}
We would like to prove the following theorem.
\begin{theorem} \label{thm:exp_game}
Let $\X$ and $\Y$ be Polish spaces. For any hypothesis class $\cH\subseteq\Y^{\X}$, we have the following dichotomy.
\begin{itemize}
\item If $\cH$ has an infinite Littlestone tree, then there is a strategy for the adversary $\wb P_a$ in $\wb \B$ such that $\wh y_t\neq y_t$ in each round $t\ge1$ against any learner $\wb P_l$.
\item If $\cH$ does not have an infinite Littlestone tree, then there is a strategy for the learner $\wb P_l$ in $\wb \B$ that only makes finitely many mistakes against any adversary $\wb P_a$.
\end{itemize}
\end{theorem}

Consider the following game $\B$ between two players $P_A$ and $P_L$. In each round $\tau\in\N$:
\begin{itemize}
\item Player $P_A$ chooses a three-tuple $\xi_\tau=(x_\tau,y_\tau^{0},y_\tau^{1})\in\wt \X$ and shows it to Player $P_L$.
\item Player $P_L$ chooses a point $\eta_\tau\in\{0,1\}$.
\end{itemize}
We say that player $P_L$ wins the game in round $\tau\in\N$ if $\cH_{\xi_1(1),\xi_1(\eta_1+2),\dots,\xi_\tau(1),\xi_\tau(\eta_\tau+2)}=\emptyset$, where $\cH_{x_1,y_1,\dots,x_t,y_t}:=\{h\in\cH:h(x_1)=y_1,\dots,h(x_t)=y_t\}$ for any $x_1,\dots,x_t\in\X$ and $y_1,\dots,y_t\in\Y$.
We say that player $P_A$ wins the game if the game
continues indefinitely. 
We say a strategy for $P_A$ is winning if playing that strategy, $P_A$ wins the
game no matter what strategy $P_L$ plays. We define a winning strategy for $P_L$ analogously.  
According to the rule of $\B$, the set of winning sequence of $P_L$ is
\begin{align*}
\W:=\left\{(\bxi,\bfeta)\in\left(\wt \X\times\{0,1\}\right)^\infty:\cH_{\xi_1(1),\xi_1(\eta_1+2),\dots,\xi_\tau(1),\xi_\tau(\eta_\tau+2)}=\emptyset\text{ for some } \tau\in\N\right\}.
\end{align*}
Since $\W$ is finitely decidable (i.e., for any $(x_1,y_1,x_2,y_2,\dots)\in\W$, there exists $n\in\N$ such that $(x_1,y_1,\dots,x_n,y_n,x_{n+1}',y_{n+1}',\dots)\in\W$ for all $(x_{n+1}',y_{n+1}',x_{n+2}',y_{n+2}',\dots)\in(\X\times\Y)^\infty$), $\B$ is a Gale-Stewart game; then, either $P_A$ or $P_L$ has a winning strategy \citep{MR0054922}. 
We refer readers to \citet[Appendix A.1]{universal_learning} for detailed descriptions of the notion we use above and Gale-Stewart games. 

Then, we prove the following lemma that relates a winning strategy of $P_A$ to an infinite Littlestone tree of $\cH$.
\begin{lemma} \label{lem:PA_winning}
Player $P_A$ has a winning strategy in the game $\B$ if and only if $\cH$ has an infinite Littlestone tree.
\end{lemma}
\begin{proof}
Suppose that $\cH$ has an infinite Littlestone tree represented by 
$$
\left\{(x_{\bu},y^0_\bu,y^1_\bu):0\le k<\infty,\bu\in\{0,1\}^k\right\}.
$$ 
Define a strategy for $P_A$ by $\xi_{\tau}(\eta_1,\dots,\eta_{\tau-1}):=(x_{\eta_1,\dots,\eta_{\tau-1}},y^0_{\eta_1,\dots,\eta_{\tau-1}},y^1_{\eta_1,\dots,\eta_{\tau-1}})$ for any $\tau\in\N$. 
By the definition of Littlestone tree, we have that $\cH_{\xi_{1}(1),\xi_1(\eta_1+2),\dots,\xi_{\tau}(1),\xi_{\tau}(\eta_{\tau}+2)}\neq \emptyset$ for any $\tau\in\N$. Thus, $P_A$ has a winning strategy.

Suppose that $P_A$ has a winning strategy $\xi_\tau(\eta_1,\dots,\eta_{\tau-1})$ for any $\eta_1,\dots,\eta_{\tau-1}\in\{0,1\}$ and $1\le \tau<\infty$. Define an infinite binary tree represented by $\left\{(x_{\bu},y_\bu^0,y_\bu^1):\bu\in\{0,1\}^k,0\le k<\infty\right\}$ with
\begin{align*}
(x_{u_1,\dots,u_{k}},y^0_{u_1,\dots,u_k},y^1_{u_1,\dots,u_k}):=\xi_{k+1}(u_1,\dots,u_k).
\end{align*}
By the definition of winning strategy of $P_A$ in $\B$, the tree defined above is an infinite Littlestone tree of $\cH$.
\end{proof}

For any $n\in\N_0$, define $\ms P_n:=\left(\wt \X\times\{0,1\}\right)^n$ to be the set of positions of length $n$ in the game $\B$, where a position of a game is a finite sequence of plays made by the two players alternatively from the start to some round and $\ms P_0=\emptyset$ by convention. 
A position is called active if $P_L$ has not won yet after this position. Then, the set of active positions of length $n$ in the game $\B$ can be written as
$$
\ms A_n:=\cup_{\bw\in (\wt\X\times\{0,1\})^\infty}\{\bv\in\ms P_n:(\bv,\bw)\in\W^c\}.
$$ 
Then, we define $\ms P:=\cup_{0\le n<\infty}\ms P_n$ to be the set of all positions and $\ms A:=\cup_{0\le n<\infty}\ms A_n$ to be the set of all active positions in the game $\B$. 

Analogously, for any $n\in\N_0$, we define $\wmP_n:=\left(\X\times\Y\right)^n$ to be the set of all positions of length $n$ in the game $\wb\B$.
For notational convenience, we also define $\ms P_\infty:=\left(\wt \X\times\{0,1\}\right)^\infty$ and
$\wmP_\infty:=\left(\X\times\Y\right)^\infty$. 

As in \citet{universal_learning}, we need to describe for how many rounds the game can be kept active starting from an arbitrary position. The following definitions of decision trees and active decision trees are the direct restriction of \citet[Definition B.4]{universal_learning} for $\ms P$ in our setting. 
\begin{definition}[\citealt{universal_learning}, Definition B.4]
Given a position $\bv\in\ms P_k$ of length $k\in\N_0$:
\begin{itemize}
\item A decision tree of depth $n$ with starting position $\bv$ is a collection of points
$$
\bt=\left\{\xi_{\bfeta}\in\wt\X:\bfeta\in \{0,1\}^t,0\le t<n\right\}.
$$
By convention, we call $\bt=\emptyset$ a decision tree of depth 0.
\item $\bt$ is called active if $(\bv,\xi_{\emptyset},\eta_{k+1},\xi_{\eta_{k+1}},\eta_{k+2},\dots,\xi_{\eta_{k+1},\dots,\eta_{k+n-1}},\eta_{k+n})\in\ms A_{k+n}$ for all choices of $(\eta_{k+1},\dots,\eta_{k+n})\in\{0,1\}^{n}$. 
\item We denote by $\ms T_{\bv}$ the set of all decision trees with starting position $\bv$ (and any depth $n\in\N_0$), and by $\ms T_{\bv}^{\ms A}\subseteq \ms T_{\bv}$ the set of all active trees.
\end{itemize}
\end{definition}
Note that $\ms T_{\bv}=\ms T_{\bv'}$ for any $\bv,\bv'\in\ms P$ by the above definition. 

We use $\ORD$ to denote the set of of all ordinals. We use $-1$ to denote an element that is smaller than every ordinal and $\Omega$ to denote an element that is are larger than every ordinal. Define $\ORD^*:=\ORD\cup\{\Omega,-1\}$. We refer readers to \citet[Appendix A]{universal_learning} for brief introductions about the concepts of ordinals, well-founded relations, ranks, Polish spaces, universally measurability, analytic sets, etc.

For any $\bv\in\ms A$, we define a relation $\prec_{\bv}$ on $\ms T_{\bv}^{\ms A}$. For $\bt,\bt'\in\ms T_{\bv}^{\ms A}$, we say that $\bt'\prec_{\bv}\bt$ if and only if the tree $\bt$ is obtained from $\bt'$ by removing its leaves. 
Let $\rho_{\prec_{\bv}}:\ms T_{\bv}^{\ms A}\rightarrow\ORD$ denote the rank function of the relation $\prec_{\bv}$. 
Then, we define the following game value of on $\ms P$ as in \citet{universal_learning}. 
\begin{definition}[\citealt{universal_learning}, Definition B.5] \label{def:val-binary}
The \emph{game value} $\val:\ms P\rightarrow\ORD^*$ is defined as follows.
\begin{itemize}
\item $\val(\bv)=-1$ if $\bv\notin\ms A$.
\item $\val(\bv)=\Omega$ if $\bv\in\ms A$ and $\prec_{\bv}$ is not well-founded.
\item $\val(\bv)=\rho_{\prec_{\bv}}(\emptyset)$ if $\bv\in\ms A$ and $\prec_{\bv}$ is well-founded.
\end{itemize}
\end{definition}
According to Lemma \ref{lem:PA_winning} and Definition \ref{def:val-binary}, we have the following Lemma about $\val(\emptyset)$.
\begin{lemma} \label{lem:val_emptyset}
We have $\val(\emptyset)>-1$.
If $\cH$ does not have an infinite Littlestone tree, then $\val(\emptyset)<\Omega$.
\end{lemma}
\begin{proof}
Obviously, $\emptyset\in\ms A$. 
If $\cH$ does not have an infinite Littlestone tree, by Lemma \ref{lem:PA_winning}, $P_A$ does not have a winning strategy. Thus, $\prec_{\emptyset}$ is well-defined. By Definition \ref{def:val-binary}, we have $\val(\emptyset)<\Omega$.
\end{proof}

In order to define game values on $\wmP$, we prove the following lemma.
\begin{lemma} \label{lem:val_equal}
For any $k_1,k_2\in\N_0$, $\bv_a\in \ms P_{k_1}$, $\bv_b\in \ms P_{k_1}$, $x\in\X$, and $y,\ y',\ y''\in\Y$ such that $y'\neq y$ and $y''\neq y$, we have $\ms T_{(\bv_a,(x,y,y'),0,\bv_b)}^{\ms A}=\ms T_{(\bv_a,(x,y'',y),1,\bv_b)}^{\ms A}=\ms T_{(\bv_a,(x,y,y''),0,\bv_b)}^{\ms A}$. 
In particular,  $\val(\bv_a,(x,y,y'),0,\bv_b)=\val(\bv_a,(x,y'',y),1,\bv_b)=\val(\bv_a,(x,y,y''),0,\bv_b)$.
\end{lemma}
\begin{proof}
It suffices to show that $\ms T_{(\bv_a,(x,y,y'),0,\bv_b)}^{\ms A}=\ms T_{(\bv_a,(x,y'',y),1),\bv_b}^{\ms A}$ and $\val(\bv_a,(x,y,y'),0,\bv_b)=\val(\bv_a,(x,y'',y),1,\bv_b)$. 
Indeed, since $y\neq y''$, the above results immediately imply that 
$$
\ms T_{(\bv_a,(x,y,y''),0,\bv_b)}^{\ms A}=\ms T_{(\bv_a,(x,y'',y),1,\bv_b)}^{\ms A}=\ms T_{(\bv_a,(x,y,y'),0,\bv_b)}^{\ms A}
$$ 
and 
$$
\val(\bv_a,(x,y,y''),0,\bv_b)=\val(\bv_a,(x,y'',y),1,\bv_b)=\val(\bv_a,(x,y,y'),0,\bv_b).
$$

Let $k=k_1+k_2$.
Since $\bv_a\in\ms P_{k_1}$ and $\bv_b\in\ms P_{k_2}$, we have $\bv_a=(\xi_1,\eta_1,\dots,\xi_{k_1},\eta_{k_1})$ and $\bv_b=(\xi_{k_1+2},\eta_1,\dots,\xi_{k+1},\eta_{k+1})$ for some $(\xi_{1},\dots,\xi_{k_1})\in\wt \X^{k_1}$, $(\xi_{k_1+2},\dots,\xi_{k+1})\in\wt \X^{k_2}$, $(\eta_1,\dots,\eta_{k-1})\in\{0,1\}^{k_1}$, and $(\eta_{k_1+2},\dots,\eta_{k+1})\in\Y^{k_2}$. 
Define $\xi_{k_1+1}^0:=(x,y,y')$, $\eta_{k_1+1}^0:=0$, $\xi_{k_1+1}^1:=(x,y'',y)$, $\eta_{k_1+1}^1:=1$, $\bv_0:=(\bv_a,(x,y,y'),0,\bv_b)$, and $\bv_1:=(\bv_a,(x,y'',y),1,\bv_b)$.
For any decision tree 
$$
\bt=\left\{\xi_{\bfeta}\in\wt \X:\bfeta\in\{0,1\}^t,0\le t<n\right\}\in\ms T_{\bv_0}
$$ 
of depth $n$ $(0\le n< \infty)$, we have $\bt\in\ms T_{\bv_1}$. 

If $\bt\in\ms T^{\ms A}_{\bv_0}$, for any $\bfeta=(\eta_{k+2},\dots,\eta_{k+n+1})\in\{0,1\}^n$, we have 
\begin{align*}
\bv_{0,\bt,\bfeta}:=(\bv_0,\xi_{\emptyset},\eta_{k+2},\xi_{\eta_{k+2}},\dots, \xi_{\eta_{k+2},\dots,\eta_{k+n}},\eta_{k+n+1})\in\ms A_{k+n+1}.
\end{align*}
By the definition of $\ms A_{k+n+1}$, there exists 
$$
\bw=(\xi_{k+n+2},\eta_{k+n+2},\xi_{k+n+3},\eta_{k+n+3},\dots)\in\left(\wt \X\times\{0,1\}\right)^{\infty}
$$
such that $(\bv_{0,\bt,\bfeta},\bw)\in\ms W^c$.

For each $t\in[k_1]\cup\{k_1+2,\dots,k+1\}\cup\{k+n+2,k+n+3,\dots\}$,
define $x_t:=\xi_t(1)$ and $y_t:=\xi_t(\eta_t+2)$.
Define $x_{k_1+1}:=\xi_{k_1+1}^0(1)=x$ and $y_{k_1+1}:=\xi_{k_1+1}^0(\eta_{k_1+1}^0+2)=y$. 
Define $x_{k+2}:=\xi_{\emptyset}(1)$ and $y_{k+2}:=\xi_{\emptyset}(\eta_{k+2}+2)$.
For each $t\in\{k+3,\dots,k+n+1\}$, 
define $x_t:=\xi_{\eta_{k+2},\dots,\eta_{t-1}}(1)$ and $y_t:=\xi_{\eta_{k+2},\dots,\eta_{t-1}}(\eta_{t}+2)$. 

Since $(\bv_{0,\bt,\bfeta},\bw)\in\ms W^c$, by the definition of $\ms W$, 
for any $0\le \tau<\infty$, there exists $h\in\cH$ such that
$h(x_t)=y_t$ for any $1\le t\le \tau$.
Since $\xi_{k_1+1}^1(1)=x=\xi_{k_1+1}^0(1)=x_{k_1+1}$ and $\xi_{k_1+1}^1(\eta_{k_1+1}^1+2)=y=\xi_{k_1+1}^0(\eta_{k_1+1}^0+2)=y_{k_1+1}$, we have
$(\bv_{1,\bt,\bfeta},\bw)\in\ms W^c$ where
\begin{align*}
\bv_{1,\bt,\bfeta}:=
(\bv_1,\xi_{\emptyset},\eta_{k+2},\xi_{\eta_{k+2}},\dots, \xi_{\eta_{k+2},\dots,\eta_{k+n}},\eta_{k+n+1}).
\end{align*}
Thus, $\bv_{1,\bt,\bfeta}\in \ms A_{k+n+1}$ for any $\bfeta\in\{0,1\}^n$. 
By the definition of $\ms T_{\bv_1}^{\ms A}$, we have
$\bt\in\ms T_{\bv_1}^{\ms A}$. 
Since it holds for any $\bt\in\ms T_{\bv_0}^{\ms A}$, we have $\ms T_{\bv_0}^{\ms A}\subseteq\ms T_{\bv_1}^{\ms A}$.

By symmetry, we can also show that $\ms T_{\bv_1}^{\ms A}\subseteq\ms T_{\bv_0}^{\ms A}$, which implies that $\ms T_{\bv_0}^{\ms A}=\ms T_{\bv_1}^{\ms A}$. 
Since $\ms T_{\bv_0}=\ms T_{\bv_1}$, we also have $\val(\bv_0)=\val(\bv_1)$.
\end{proof}

Now, we can define game values on $\wmP$ using game values on $\ms P$. 
\begin{definition} \label{def:val_XY}
The $\emph{game value}$ $\val:$ $\wmP\rightarrow\ORD^*$ is defined as follows. 
For $\emptyset$, $\val(\emptyset)$ is defined by Definition \ref{def:val-binary}. 
For any $n\in\N$ and $\bz=(x_1,y_1,\dots,x_n,y_n)\in\wmP_n$, pick a sequence $y_1',\dots,y_n'$ such that $y_1'\neq y,\dots$, and $y_n'\neq y_n$. 
Define $\bv:=(\xi_1,\eta_1,\dots,\xi_n,\eta_n)\in\ms P_n$ with $\xi_i:=(x_i,y_i,y_i')$ and $\eta_i:=0$ for any $i\in[n]$. 
Define $\val(\bz):=\val(\bv)$.
\end{definition}
By Lemma \ref{lem:val_equal}, $\val$ is well-defined on $\wmP$ and the following corollary holds.
\begin{corollary} \label{coro:val_eq}
For any $0\le n<\infty$ and $(\xi_1,\eta_1,\dots,\xi_n,\eta_n)\in\left(\wt\X\times\{0,1\}\right)^n$, we have
\begin{align} \label{eq:val_eq}
\val(\xi_1(1),\xi_1(\eta_1+2),\dots,\xi_n(1),\xi_n(\eta_n+2))=\val(\xi_1,\eta_1,\dots,\xi_n,\eta_n).
\end{align}
\end{corollary}
\begin{proof}
For $n=0$, we have $\val(\emptyset)=\val(\emptyset)$. 
For any $n\ge 1$, by Definition \ref{def:val_XY} and Lemma \ref{lem:val_equal}, we have
\begin{align*}
&\val(\xi_1(1),\xi_1(\eta_1+2),\dots,\xi_n(1),\xi_n(\eta_n+2))
\\=&\val((\xi_1(1),\xi_1(\eta_1+2),\xi_1(3-\eta_1),0,\dots,(\xi_{n}(1),\xi_{n}(\eta_{n}+2),\xi_{n}(3-\eta_{n}),0)
\\=&\val((\xi_1(1),\xi_1(\eta_1+2),\xi_1(3-\eta_1),0,\dots,(\xi_n(1),\xi_n(2),\xi_n(3)),\eta_n)
\\=&\val((\xi_1(1),\xi_1(\eta_1+2),\xi_1(3-\eta_1),0,\dots,\xi_n,\eta_n)
\\ \vdots&
\\=&\val(\xi_1,\eta_1,\dots,\xi_n,\eta_n),
\end{align*}
which gives \eqref{eq:val_eq}.
\end{proof}

According to \citet[Lemma B.7]{universal_learning}, we have the following Lemma.
\begin{lemma} \label{lem:val_max}
If $\W$ is coanalytic, then for any $\bv\in\ms P$, either $\val(\bv)=\Omega$ or $\val(\bv)<\omega_1$. In particular, it follows that either $\val(\emptyset)=\Omega$ or $\val(\emptyset)<\omega_1$. 
\end{lemma}
According to \citet[Proposition B.8]{universal_learning}, we have the following Proposition.
\begin{proposition} \label{prop:val_decrease}
Fix $0\le n<\infty$ and $\bv\in\ms P_n$ such that $0\le \val(\bv)<\Omega$. For any $\xi=(x,y^0,y^1)\in\wt \X$, there exists $\eta\in\{0,1\}$ such that $\val(\bv,\xi,\eta)<\val(\bv)$.
\end{proposition}

For any $0\le n<\infty$, define 
\begin{align*}
\ms D_{n+1}:=\left\{(\bv,\xi,\eta)\in\ms P_{n+1}:\val(\bv,\xi,\eta)<\min\{\val(\bv),\val(\emptyset)\}\right\}
\end{align*}
and
\begin{align*}
\wmD_{n+1}:=\left\{(\bz,x,y)\in\wmP_{n+1}:\val(\bz,x,y)<\min\{\val(\bz),\val(\emptyset)\}\right\}.
\end{align*}
The following lemma relates $\wmD_{n+1}$ to $\ms D_{n+1}$.
\begin{lemma} \label{lem:D_eq}
For any $0\le n<\infty$, we have \begin{align*}
\wmD_{n+1}=\left\{(\xi_1(1),\xi_1(\eta_1+2),\dots, \xi_{n+1}(1),\xi_{n+1}(\eta_{n+1}+2)):\ (\xi_1,\eta_1,\dots,\xi_{n+1},\eta_{n+1})\in\ms D_{n+1}\right\}.
\end{align*}
\end{lemma}
\begin{proof}
For any $(\xi_1,\eta_1,\dots,\xi_{n+1},\eta_{n+1})\in\ms D_{n+1}$, we have 
$$
\val(\xi_1,\eta_1,\dots,\xi_{n+1},\eta_{n+1})<\min\{\val(\xi_1,\eta_1,\dots,\xi_{n},\eta_{n}),\val(\emptyset)\}.
$$ 
By Corollary \ref{coro:val_eq}, we have
\begin{align*}
&\val(\xi_1(1),\xi_1(\eta_1+2),\dots,\xi_{n+1}(1),\xi_{n+1}(\eta_{n+1}+2))
\\=&\val(\xi_1,\eta_1,\dots,\xi_{n+1},\eta_{n+1})
\\<&\min\{\val(\xi_1,\eta_1,\dots,\xi_{n},\eta_{n}),\val(\emptyset)\}.
\\=&\min\{\val(\xi_1(1),\xi_1(\eta_1+2),\dots,\xi_{n}(1),\xi_{n}(\eta_{n}+2),\val(\emptyset)\}
\end{align*}
which implies that $(\xi_1(1),\xi_1(\eta_1+2),\dots,\xi_{n+1}(1),\xi_{n+1}(\eta_{n+1}+2))\in\wmD_{n+1}$.

On the other hand, for any $(x_1,y_1,\dots,x_{n+1},y_{n+1})\in\wmD_{n+1}$, define $\xi_i:=(x_i,y_i,y_i')$ for arbitrary $y_i'\in\Y$ satisfying $y_i'\neq y_i$ and $\eta_i:=0$ for each $i\in[n]$. Then, by Corollary \ref{coro:val_eq}, we have
\begin{align*}
&\val(\xi_1,\eta_1,\dots,\xi_{n+1},\eta_{n+1})
\\=&\val(\xi_1(1),\xi_1(\eta_1+2),\dots,\xi_{n+1}(1),\xi_{n+1}(\eta_{n+1}+2))
\\=&\val(x_1,y_1,\dots,x_{n+1},y_{n+1})
\\<&\min\{\val(x_1,y_1,\dots,x_{n},y_{n})
,\min(\emptyset)\}
\\=&\min\{\val(\xi_1,\eta_1,\dots,\xi_{n},\eta_{n})
,\min(\emptyset)\}
\end{align*}
Thus, $(\xi_1,\eta_1,\dots,\xi_n,\eta_n)\in\ms D_{n+1}$ and $\xi_i(1)=x_i$, $\xi_i(\eta_i+2)=y_i$ for all $i\in[n]$. 
Therefore, 
\begin{align*}
&(x_1,y_1,\dots,x_{n+1},y_{n+1})
\\ \in&\left\{(\xi_1(1),\xi_1(\eta_1+2),\dots, \xi_{n+1}(1),\xi_{n+1}(\eta_{n+1}+2)):\ (\xi_1,\eta_1,\dots,\xi_{n+1},\eta_{n+1})\in\ms D_{n+1}\right\}.
\end{align*}
In conclusion, Lemma \ref{lem:D_eq} is proved. 
\end{proof}

According to \citet[Lemma B.10]{universal_learning}, we have the following lemma.
\begin{lemma} \label{lem:val_greater}
For any $0 \leq n<\infty, \bv \in \ms P_{n}$, and $\kappa \in \ORD$, we have $\val(\bv)>\kappa$ if and only if there exists $\xi \in \wt\X$ such that $\val(\bv, \xi, \eta) \ge \kappa$ for all $\eta \in \{0,1\}$.
\end{lemma}
Then, by Corollary \ref{coro:val_eq} and Lemma \ref{lem:val_greater}, the following corollary holds.
\begin{corollary} \label{coro:val_greater}
For any $0 \leq n<\infty$, $\bz \in \wmP_{n}$, and $\kappa \in \ORD$, we have $\val(\bz)>\kappa$ if and only if there exist $x \in \wt\X$ and $(y,y')\in\Yp$ such that $\val(\bz, x, y) \ge \kappa$ and $\val(\bz, x, y) \ge \kappa$ and $\val(\bz, x, y') \ge \kappa$.
\end{corollary}

Define 
\begin{align*}
\wmW:=\{(x_1,y_1,\dots)\in(\X\times\Y)^{\infty}:\cH_{x_1,y_1,\dots,x_\tau,y_\tau}=\emptyset\textup{ for some }0\le \tau<\infty\}.
\end{align*}
Then, we can show that $\wmW$ is coanalytic under the assumption that $\cH$ is measurable.
\begin{lemma} \label{lem:coanalytic}
If $\X$ and $\Y$ are Polish and $\cH$ is measurable, then $\wmW$ is coanalytic.
\end{lemma}
\begin{proof}
According to Definition \ref{def:measurable}, we have
\begin{align*}
\wmW^c=&\left\{(x_1,y_1,\dots)\in(\X\times\Y)^{\infty}:\cH_{x_1,y_1,\dots,x_\tau,y_\tau}\neq\emptyset\textup{ for all }\tau<\infty\right\}
\\=&\cap_{\tau=1}^\infty\cup_{\theta\in\Theta}\cap_{t=1}^{\tau}\left\{(x_1,y_1,\dots)\in(\X\times\Y)^{\infty}:\sfh(\theta,x_t)=y_t\right\}.
\end{align*}
For any $h\in\cH$ and $1\le t<\infty$, define $\wt h_t:\Theta\times(\X\times\Y)^{\infty}\rightarrow\Y$, $(\theta,x_1,y_1,\dots)\mapsto h(\theta,x_t)$ and $l_t:\Theta\times(\X\times\Y)^{\infty}\rightarrow\mathbb{R}$, $(\theta,x_1,y_1,\dots)\mapsto \indi\{y_t\neq \wt h_t(\theta,x_1,y_1,\dots)\}$.

Since $\cH$ is measurable, $h\in\cH$ is Borel-measurable. Thus, $\wt h_t$ is also Borel-measurable. 
Since the mapping $\Theta\times(\X\times\Y)^{\infty}\rightarrow\Y$, $(\theta,x_1,y_1,\dots)\mapsto y_t$ is Borel-measurable, the mapping $\Theta\times(\X\times\Y)^{\infty}\rightarrow\Y^2$, $(\theta,x_1,y_1,\dots)\mapsto (y_t,\wt h_t(\theta,x_1,y_1,\dots))$ is also Borel-measurable, which, together with the fact that the mapping $\Y^2\rightarrow\{0,1\}$, $(y,y')\mapsto\indi\{y=y'\}$ is Borel-measurable, implies that $l_t$ is Borel-measurable.
Since
\begin{align*}
\left\{(\theta,x_1,y_1,\dots)\in\Theta\times(\X\times\Y)^{\infty}:h(\theta,x_t)=y_t\right\}
=l_t^{-1}(\{1\}),
\end{align*}
we have that $\left\{(\theta,x_1,y_1,\dots)\in\Theta\times(\X\times\Y)^{\infty}:h(\theta,x_t)=y_t\right\}$ is Borel for any $1\le t<\infty$ and $h\in \cH$. Thus, for any $1\le \tau<\infty$, $\cap_{t=1}^\tau\left\{(\theta,x_1,y_1,\dots)\in\Theta\times(\X\times\Y)^{\infty}:h(\theta,x_t)=y_t\right\}$ is Borel.
Since the union over $\theta\in\Theta$ corresponds to a projection and the intersection over $\tau$ is countable, the set $\wmW^c$ is analytic. 
\end{proof}

Define
\begin{align*}
\wmA_n:=\cup_{\bw\in(\X\times \Y)^{\infty}}
\left\{\bz\in\wmP_n:(\bz,\bw)\in\wmW^c\right\}.
\end{align*}
Since $\X$ and $\Y$ are Polish, we have that $\wmP_n$ is Polish for any $0\le n\le \infty$. 
If $\wmW$ is coanalytic, then $\wmA_n$ is an analytic subset of $\wmP_n$ for any $0\le n<\infty$. 

Define $\wmA:=\cup_{0\le n<\infty}\wmA_n$. We have the following lemma.
\begin{lemma} \label{lem:val-1}
$\val(\bz)> -1$ for any $\bz\in\wmA$.
\end{lemma}
\begin{proof}
Since $\bz\in\wmA$, there exists $0\le n<\infty$ such that $\bz\in\wmA_n$. There exist $(x_1,y_1,\dots,x_n,y_n)\in\wmP_n$ and $(x_{n+1},y_{n+1},x_{n+2},y_{n+2},\dots)\in(\X\times\Y)^{\infty}$ such that $\bz=(x_1,y_1,\dots,x_n,y_n)$ and 
$$
\cH_{x_1,y_1,\dots,x_{\tau},y_{\tau}}\neq\emptyset
$$ 
for all $1\le \tau<\infty$.

For any $1\le i<\infty$, define $\eta_i=0$ and $\xi_i=(x_i,y_i,y_i')$ for arbitrary $y_i'\in\Y$ such that $y_i'\neq y_i$. 
It follows that
$\cH_{\xi_1(1),\xi_1(\eta_1+2),\dots,\xi_\tau(1),\xi_\tau(\eta_\tau+2)}\neq \emptyset$ for all $1\le \tau<\infty$. Thus, 
$$
\bv:=(\xi_1,\eta_1,\dots,\xi_n,\eta_n)\in\ms A_n
$$ 
and by Definition \ref{def:val-binary} and Definition \ref{def:val_XY}, $\val(\bz)=\val(\bv)>-1$.
\end{proof}
Now, the corollary below holds.
\begin{corollary} \label{coro:A_analytic}
If $\wmW$ is coanalytic, then the set 
\begin{align*}
\wmA^{\kappa}_n:=\{\bz\in\wmA_n:\val(\bz)>\kappa\}
\end{align*}
is analytic for every $0\le n<\infty$ and $-1\le \kappa<\omega_1$.
\end{corollary}
\begin{proof}
Since $\wmW$ is coanalytic, we have that $\wmA_n$ is analytic.
For $\kappa=-1$, since $\val(\bz)> -1$ for any $\bz\in\wmA_n$ by Lemma \ref{lem:val-1}, we have that $\wmA_n^{-1}=\wmA_n$ is analytic for any $0\le n<\infty$.
For $\kappa>-1$,
suppose that for all $-1\le \lambda<\kappa$, $\wmA_n^{\lambda}$ is analytic for every $0\le n<\infty$. 
According to Corollary \ref{coro:val_greater}, for any $0\le n<\infty$, we have
\begin{align*}
\wmA_n^{\kappa}=&\cup_{(x,y,y')\in\X\times\Yp}
\left\{\bz\in\wmA_n:\val(\bz,x,y)\ge \kappa\textup{ and }\val(\bz,x,y')\ge \kappa\right\}
\\=&
\cup_{(x,y,y')\in\X\times\Y^2}\left(
\left\{\bz\in\wmA_n:\val(\bz,x,y)\ge\kappa\textup{ and }\val(\bz,x,y')\ge\kappa\textup{ and }y\neq y'\right\}
\right)
\end{align*}
Consider the function $f:\ \wmP_n\times\X\times\Y^2\rightarrow\wmP_n\times\X\times\Y^2$, $(\bz,x,y^0,y^1)\mapsto(\bz,x,y^1,y^0)$. 
$f$ is a continuous function. 
Since by the induction hypothesis, $\wmA_{n+1}^{\lambda}$ is analytic for any $-1\le \lambda<\kappa$, we have that $f(\wmA_{n+1}^{\lambda})$ is also analytic.
Thus, 
\begin{align*}
&\left\{(\bz,x,y,y')\in\wmA_n\times\X\times\Y^2:\val(\bz,x,y)\ge\kappa\textup{ and }\val(\bz,x,y')\ge\kappa\textup{ and }y\neq y'\right\}
\\=&
\cap_{-1\le \lambda<\kappa}
\left(\wmA_{n+1}^{\lambda}\cap f(\wmA_{n+1}^{\lambda})\right)
\cap\wmA_n\times\X\times\Yp
\end{align*} 
is also analytic.
Since $\X$ and $\Y$ are Polish spaces, we have that
$\wmA_n^{\kappa}$ is analytic.
By induction, $\wmA_n^{\kappa}$ for any $0\le n<\infty$ and $-1\le \kappa<\omega_1$.
\end{proof}

For any $0\le n<\infty$, define 
\begin{align*}
\wmD_{n+1}:=\left\{(\bz,x,y)\in\wmP_{n+1}:\val(\bz,x,y)<\min\{\val(\bz),\val(\emptyset)\}\right\}.
\end{align*}
Then, we can show the following corollary.
\begin{corollary}
If $\val(\emptyset)<\omega_1$ and $\wmW$ is coanalytic, then $\wmD_{n+1}$ is universally measurable for any $0\le n<\infty$.
\end{corollary}
\begin{proof}
By the definition of $\wmD_{n+1}$, we have
\begin{align*}
\wmD_{n+1}=&\cup_{-1\le \kappa<\val(\emptyset)}\left\{(\bz,x,y)\in\wmP_{n+1}:\val(\bz,x,y)\le\kappa\textup{ and }\val(\bz)>\kappa\right\}
\\=&
\cup_{-1\le \kappa<\val(\emptyset)}\left\{(\bz,x,y)\in\wmP_{n+1}:(\bz,x,y)\in(\wmA_{n+1}^{\kappa})^c\textup{ and }\bz\in\wmA_n^{\kappa}\right\}
\\=&
\cup_{-1\le \kappa<\val(\emptyset)}\left((\wmA_{n+1}^{\kappa})^c\cap\wmA_n^{\kappa}\times\X\times\Y\right)
\end{align*}
with $\wmA_n^\kappa$ defined in Corollary \ref{coro:A_analytic}. According to Corollary \ref{coro:A_analytic}, $(\wmA_{n+1}^{\kappa})^c\cap\wmA_n^{\kappa}\times\X\times\Y$ is universally measurable for any $-1\le \kappa<\omega_1$. 
Since $\val(\emptyset)<\omega_1$, the union over $-1\le\kappa<\val(\emptyset)$ is countable. Thus, $\wmD_{n+1}$ is universally measurable.
\end{proof}

However, for the universal measurability of the learning strategy we defined, the above corollary does not directly apply. We need more refined analysis of the projection set of $\wmD_{n}$. 
For any $0\le n<\infty$ and $y\in\Y$, define 
\begin{align*}
\wmD_{n+1}^y:=&\left\{(\bz,x)\in\wmP_{n}\times\X:\val(\bz,x,y)<\min\{\val(\bz),\val(\emptyset)\}\right\}
\\=&\left\{(\bz,x)\in\wmP_{n}\times\X:(\bz,x,y)\in\wmD_{n+1}\right\}.
\end{align*}
Then, we can proceed to show the following corollary.
\begin{corollary} \label{coro:wmDy_meas}
If $\val(\emptyset)<\omega_1$ and $\wmW$ is coanalytic, then $\wmD^y_{n+1}$ is universally measurable for any $0\le n<\infty$ and $y\in\Y$.
\end{corollary}
\begin{proof}
By the definition of $\wmD_{n+1}^y$, we have
\begin{align*}
\wmD_{n+1}^y=&\cup_{-1\le \kappa<\val(\emptyset)}\left\{(\bz,x)\in\wmP_{n}\times\X:\val(\bz,x,y)\le\kappa\textup{ and }\val(\bz)>\kappa\right\}
\\=&
\cup_{-1\le \kappa<\val(\emptyset)}\left\{(\bz,x)\in\wmP_{n}\times\X:(\bz,x,y)\in(\wmA_{n+1}^{\kappa})^c\textup{ and }\bz\in\wmA_n^{\kappa}\right\}
\\=&
\cup_{-1\le \kappa<\val(\emptyset)}\left(\left\{(\bz,x)\in\wmP_{n}\times\X:(\bz,x,y)\in(\wmA_{n+1}^{\kappa})^c\right\}\cap\wmA_n^{\kappa}\times\X\right)
\end{align*}
with $\wmA_n^\kappa$ defined in Corollary \ref{coro:A_analytic}. 
Note that
\begin{align*}
&\left\{(\bz,x)\in\wmP_{n}\times\X:(\bz,x,y)\in(\wmA_{n+1}^{\kappa})^c\right\}
\\=&
\left\{(\bz,x)\in\wmP_{n}\times\X:(\bz,x,y)\in\wmA_{n+1}^{\kappa}\right\}^c
\\=&
\left(\cup_{y'\in\Y}\left\{(\bz,x)\in\wmP_{n}\times\X:(\bz,x,y')\in\wmA_{n+1}^{\kappa}\cap\wmP_n\times\X\times\{y\}\right\}\right)^c
\end{align*}
According to Corollary \ref{coro:A_analytic}, $\wmA_{n}^{\kappa}\times\X$ is an analytic subset of $\wmP_n\times\X$ and
$\wmA^{\kappa}_{n+1}\cap\wmP_n\times\X\times\{y\}$ is an analytic subset of $\wmP_{n+1}$ for any $-1\le \kappa<\omega_1$. 
Thus, $\left\{(\bz,x)\in\wmP_{n}\times\X:(\bz,x,y)\in(\wmA_{n+1}^{\kappa})^c\right\}$ is coanalytic for any $-1\le \kappa<\omega_1$.
Since $\val(\emptyset)<\omega_1$, the union over $-1\le\kappa<\val(\emptyset)$ is countable. Thus, $\wmD_{n+1}^y$ is universally measurable.
\end{proof}

For any $0\le n<\infty$, $\bz\in\wmP_n$, and $y\in\Y$, define 
\begin{align*}
\wmD_{n+1}^{\bz,y}:=&\left\{
x\in\X: \val(\bz,x,y)<\min\{\val(\bz),\val(\emptyset)\}\right\}
\\=&\left\{
x\in\X: (\bz,x,y)\in\wmD_{n+1}\right\}.
\end{align*}
Then, we can show the following corollary.
\begin{corollary} \label{coro:wmDzy_meas}
If $\val(\emptyset)<\omega_1$ and $\wmW$ is coanalytic, then $\wmD^{\bz,y}_{n+1}$ is universally measurable for any $0\le n<\infty$, $\bz\in\wmP_n$, and $y\in\Y$.
\end{corollary}
\begin{proof}
By the definition of $\wmD_{n+1}^{\bz,y}$, we have
\begin{align*}
\wmD_{n+1}^{\bz,y}=&\cup_{-1\le\kappa<\val(\emptyset)}\left\{x\in\X:\val(\bz,x,y)\le\kappa\textup{ and }\val(\bz)>\kappa\right\}
\\=&
\cup_{-1\le \kappa<\val(\emptyset)}\left\{x\in\X:(\bz,x,y)\in(\wmA_{n+1}^{\kappa})^c\textup{ and }\bz\in\wmA_n^{\kappa}\right\}
\\=&
\cup_{\kappa:-1\le \kappa<\val(\emptyset),\ \bz\in\wmA_n^{\kappa}}
\left\{x\in\X:(\bz,x,y)\in(\wmA_{n+1}^{\kappa})^c\right\}
\end{align*}
with $\wmA_n^\kappa$ defined in Corollary \ref{coro:A_analytic}. 
Note that
\begin{align*}
&\left\{x\in\X:(\bz,x,y)\in(\wmA_{n+1}^{\kappa})^c\right\}
\\=&
\left\{x\in\X:(\bz,x,y)\in\wmA_{n+1}^{\kappa}\right\}^c
\\=&
\left(\cup_{(\bw,y')\in\wmP_n\times\Y}\left\{x\in\X:(\bw,x,y')\in\wmA_{n+1}^{\kappa}\cap\{\bz\}\times\X\times\{y\}\right\}\right)^c
\end{align*}
According to Corollary \ref{coro:A_analytic}, $\wmA^{\kappa}_{n+1}\cap\{\bz\}\times\X\times\{y\}$ is an analytic subset of $\wmP_{n+1}$ for any $-1\le \kappa<\omega_1$. 
Thus, $\left\{x\in\X:(\bz,x,y)\in(\wmA_{n+1}^{\kappa})^c\right\}$ is coanalytic for any $-1\le \kappa<\omega_1$.
Since $\val(\emptyset)<\omega_1$, the union over $\kappa$ is countable. Thus, $\wmD_{n+1}^{\bz,y}$ is universally measurable.
\end{proof}

Now, we are ready to define a value-decreasing function. 
For any $1\le t<\infty$, $\bz\in\wmP_{t-1}$, and $x\in\X$, define the set $G_{t,\bz,x}:=\left\{y\in\Y:(\bz,x,y)\notin\wmD_t\right\}$. 
When $\Y$ is uncountable, define the mapping $g_t:\wmP_{t-1}\times\X\rightarrow\Y$ by
\begin{align} \label{eq:gt-exp}
g_t(\bz,x):=
\begin{cases}
& \textup{arbitrary } y\in G_{t,\bz,x},\quad \textup{if }G_{t,\bz,x}\neq\emptyset,\\
& \textup{arbitrary } y\in\Y, \quad\textup{if }G_{t,\bz,x}=\emptyset.
\end{cases}
\end{align}
When $\Y$ is countable, we can enumerate $\Y$ as $\{y^1,y^2,y^3,\dots,\}$. Then, the mapping $g_t:\wmP_{t-1}\times\X\rightarrow\Y$ is defined as
\begin{align*}
g_t(\bz,x):=&
\begin{cases}
& y^i,\quad \textup{if } G_{t,\bz,x}\neq \emptyset \textup{ and } y^j\notin G_{t,\bz,x} \textup{ for all }1\le j\le i-1,\ y^i\in G_{t,\bz,x}, \\
& y^1,\quad \textup{if } G_{t,\bz,x}=\emptyset.
\end{cases}
\\=&
\begin{cases}
& y^i,\quad \textup{if } (\bz,x)\in\wmD_t^{y^j}
\textup{ for all }1\le j\le i-1\textup{ and } (\bz,x)\notin\wmD_t^{y^i}, \\
& y^1,\quad \textup{if } G_{t,\bz,x}=\emptyset.
\end{cases}
\\=&
\begin{cases}
& y^i,\quad \textup{if } (\bz,x)\in\left(\cap_{j=1}^{i-1}\wmD_t^{y^j}\right)\cap\left(\wmD_t^{y^i}\right)^c, \\
& y^1,\quad \textup{if } (\bz,x)\in\cap_{j=1}^{\infty}\wmD_t^{y^j}.
\end{cases}
\end{align*}
\begin{corollary} \label{coro:gt_meas}
If $\Y$ is countable, $\val(\emptyset)<\omega_1$, and $\wmW$ is coanalytic, then $g_t$ is universally measurable for any $1\le t<\infty$. 
\end{corollary}
\begin{proof}
For any $2\le i<\infty$, we have
\begin{align*}
g_t^{-1}(y^i)=
\left(\cap_{j=1}^{i-1}\wmD_t^{y^j}\right)\cap\left(\wmD_t^{y^i}\right)^c
\end{align*}
which is universally measurable by Corollary \ref{coro:wmDy_meas}. 
For $i=1$, we have
\begin{align*}
g_t^{-1}(y^1)=
\left(\cap_{j=1}^{\infty}\wmD_t^{y^j}\right)\cup\left(\wmD_t^{y^1}\right)^c
\end{align*}
which is also universally measurable by Corollary \ref{coro:wmDy_meas}. 
\end{proof}

For any $1\le t<\infty$, $\bz\in\wmP_{t-1}$, and $x\in\X$, define the mapping $g_{t,\bz}:\X\rightarrow\Y$, $x\mapsto g_t(\bz,x)$. Then, we have the following corollary.
\begin{corollary} \label{coro:gtz_meas}
If $\Y$ is countable, $\val(\emptyset)<\omega_1$, and $\wmW$ is coanalytic, then $g_{t,\bz}$ is universally measurable for any $1\le t<\infty$ and $\bz\in\wmP_{t-1}$.
\end{corollary}
\begin{proof}
By the definition of $g_{t,\bz}$, we have
\begin{align*}
g_{t,\bz}(x)=
\begin{cases}
&y^i,\quad \textup{if } x\in \left(\cap_{j=1}^{i-1}\wmD_{t}^{\bz,y^j}\right)\cap\left(\wmD_{t}^{\bz,y^i}\right)^c,\\
&y^1,\quad \textup{if } x\in \cap_{j=1}^{\infty}\wmD_{t}^{\bz,y^j}.
\end{cases}
\end{align*}
Thus, for $2\le i<\infty$, we have
\begin{align*}
g_{t,\bz}^{-1}(y^i)=\left(\cap_{j=1}^{i-1}\wmD_{t}^{\bz,y^j}\right)\cap\left(\wmD_{t}^{\bz,y^i}\right)^c
\end{align*}
which is universally measurable by Corollary \ref{coro:wmDzy_meas}. 
For $i=1$, we have
\begin{align*}
g_{t,\bz}^{-1}(y^1)=\left(\cap_{j=1}^{\infty}\wmD_{t}^{\bz,y^j}\right)\cup\left(\wmD_{t}^{\bz,y^1}\right)^c
\end{align*}
which is also universally measurable by Corollary \ref{coro:wmDzy_meas}. 
\end{proof}

For any $1\le t<\infty$, define the mapping $\wb g_t:\X^{t}\rightarrow\Y$, 
$$
(x_1,x_2,\dots,x_t)\mapsto g_t(x_1,g_1(x_1),x_2,g_2(x_1,g_1(x_1),x_2),\dots,x_{t-1},g_{t-1}(x_1,g_1(x_1),\dots,x_{t-1}),x_t)
$$
We can show the following lemma.
\begin{lemma} \label{lem:wbgt_meas}
For any $1\le t<\infty$, if $g_{i}$ is universally measurable for all $i\in[t]$, then $\wb g_t$ is also universally measurable.
\end{lemma}
\begin{proof}
For each $i\in[t-1]$, define the mapping $\wt g_i:\wmP_{i-1}\times\X^{t-i+1}\rightarrow\wmP_{i}\times\X^{t-i}$,
\begin{align*}
&(x_1,y_1,\dots,x_{i-1},y_{i-1},x_i,x_{i+1},\dots,x_{t})
\\ \mapsto&
(x_1,y_1,\dots,x_{i},g_i(x_1,y_1,\dots,x_{i-1},y_{i-1},x_i),x_{i+1},x_{i+2},\dots,x_{t}).
\end{align*}
Then, we have $\wb g_t=g_t\circ\wt g_{t-1}\circ\cdots\circ\wt g_1$. Since $g_t$ is universally measurable, it suffices to show that $\wt g_{i}$ is universally measurable for each $i\in[t-1]$. 
For any Polish space $\mathcal{E}_1$ and $\mathcal{E}_2$, let $\F(\mathcal{E}_1)$ denote the Borel $\sigma$-field of $\mathcal{E}_1$ and $\F(\mathcal{E}_1)\times\F(\mathcal{E}_2)$ denote the product $\sigma$-field of $\F(\mathcal{E}_1)$ and $\F(\mathcal{E}_2)$ on $\mathcal{E}_1\times\mathcal{E}_2$. 
Since $\X$ and $\Y$ are Polish spaces, we have $\F(\wmP_j\times\X^k)=(\F(\X)\times\F(\Y))^j\times\F(\X)^k$ for any $0\le j,k<\infty$. Thus, it suffices to show that $\wt g_i^{-1}((\prod_{j=1}^iA_j\times B_j)\times (\prod_{k=i+1}^t A_k))$ is universally measurable in $\wmP_{i-1}\times\X^{t-i+1}$ for any $A_j\in\F(\X)$ with $j\in[t]$, any $B_j\in\F(\Y)$ with $j\in[i]$, and any $i\in[t-1]$. By the definition of $\wt g_i$, we have
\begin{align*}
&\wt g_i^{-1}\left(\left(\prod_{j=1}^iA_j\times B_j\right)\times \left(\prod_{k=i+1}^t A_k\right)\right)
\\=&
\Big\{(x_1,y_1,\dots,x_{i-1},y_{i-1},x_i,x_{i+1},\dots,x_{t})\in \left(\prod_{j=1}^{i-1}A_j\times B_j\right)\times \left(\prod_{k=i}^t A_k\right) : 
\\&\quad\quad
g_i(x_1,y_1,\dots,x_{i-1},y_{i-1},x_i)\in B_i\Big\}
\\=&
\left(\prod_{j=1}^{i-1}A_j\times B_j\right)\times \left(\prod_{k=i}^t A_k\right)
\cap g_i^{-1}(B_i)\times\left(\prod_{k=i+1}^t A_k\right)
\end{align*}
Since $g_i$ is universally measurable, we have that $g_i^{-1}(B_i)$ is a universally measurable subset of $\wmP_{i-1}\times\X$. It follows that $g_i^{-1}(B_i)\times\left(\prod_{k=i+1}^t A_k\right)$ is universally measurable in $\wmP_{i-1}\times\X^{t-i+1}$. Thus, $\wt g_i^{-1}\left(\left(\prod_{j=1}^iA_j\times B_j\right)\times \left(\prod_{k=i+1}^t A_k\right)\right)$ is universally measurable in $\wmP_{i-1}\times\X^{t-i+1}$.
\end{proof}
The following corollary immediately follows from Corollary \ref{coro:gt_meas} and Lemma \ref{lem:wbgt_meas}.
\begin{corollary} \label{coro:wbgt_meas}
If $\Y$ is countable, $\val(\emptyset)<\omega_1$, and $\wmW$ is coanalytic, then $\wb g_t$ is universally measurable for any $1\le t<\infty$.
\end{corollary}

We have the following lemma.
\begin{lemma} \label{lem:val_A}
For any $1\le t<\infty$, $\bz=(x_1,y_1,\dots,x_t,y_t)\in\wmP_t$, we have 
\begin{align*}
\val(\bz)=-1\Longleftrightarrow \cH_{\bz}=\emptyset.
\end{align*}
\end{lemma}
\begin{proof}
Define $\xi_i:=(x_i,y_i,y_i')$ for arbitrary $y_i'\in\Y\backslash\{y_i'\}$ and $\eta_i:=0$ for each $i\in[t]$. Define $\bv:=(\xi_1,\eta_1,\dots,\xi_t,\eta_t)$.

Assume that $\cH_{\bz}=\emptyset$. Then, for any $\bw\in(\wt \X\times\{0,1\})^{\infty}$, we have 
$$
\cH_{\xi_1(1),\xi_1(\eta_1+2),\dots,\xi_t(1),\xi_t(\eta+2)}=\cH_{\bz}=\emptyset
$$
which implies that $(\bv,\bw)\in\ms W$. Thus, we have $\bv\notin\ms A$. By Definition \ref{def:val-binary}, we have $\val(\bv)=-1$. By Corollary \ref{coro:val_eq}, we have $\val(\bz)=\val(\bv)=-1$.

For the other direction, assume that $\val(\bz)=-1$. By Corollary \ref{coro:val_eq}, we have $\val(\bv)=\val(\bz)=-1$. By Definition \ref{def:val-binary}, we have $(\bv,\bw)\in\ms W$ for any $\bw\in(\wt\X\times\{0,1\})^\infty$. Suppose that $\cH_{\bz}\neq\emptyset$. Choose arbitrary $h\in\cH_\bz$. Since $h\in\Y^\X$, there exists $x\in\X$ and $y\in\Y$ such that $h(x)=y$. Choose arbitrary $y'\in\Y\backslash\{y\}$. Define $\xi_{t+i}:=(x,y,y')$ and $\eta_{t+i}:=0$ for any $1\le i<\infty$, and  $\bw:=(\xi_{t+1},\eta_{t+1},\xi_{t+2},\eta_{t+2},\dots)\in(\wt\X\times\{0,1\})^{\infty}$. Then, for any $0\le \tau<\infty$, we have
$h\in\cH_{\xi_1(1),\xi_1(\eta_1+1),\dots,\xi_\tau(1),\xi_\tau(\eta_\tau+1)}$. Then, we have $(\bv,\bw)\notin\ms W$. A contradiction. Thus, $\cH_{\bz}=\emptyset$.
\end{proof}

Then, we can prove the following guarantee for $g_t$.
\begin{proposition} \label{prop:gt}
For any $(x_1,x_2,\dots)\in\X^{\infty}$, any $y_1\in\Y\backslash\{g_1(x_1)\}$, and any $y_t\in\Y$ such that $y_t\neq g_t(x_1,y_1,\dots,x_{t-1},y_{t-1},x_t)$ with $2\le t<\infty$, 
if $\val(\emptyset)<\Omega$, then there exists some positive integer $\tau$ $(1\le \tau<\infty)$ such that $\cH_{x_1,y_1,\dots,x_\tau,y_\tau}=\emptyset$.
\end{proposition}
\begin{proof}
By Lemma \ref{lem:val_emptyset}, we have $\val(\emptyset)\ge0$.
Define $\xi_t:=(x_t,y_t,\wb g_t(x_1,\dots,x_t))$, $\bv_t=(\xi_1,0,\dots,\xi_t,0)$, and $\bz_{t}:=(x_1,y_t,\dots,x_t,y_t)$ for any $0\le t<\infty$ (when $t=0$, we have $\bv_0=\emptyset$ and $\bz_0=\emptyset$). 

We claim that for any $1\le t<\infty$, if $0\le \val(\bz_{t-1})\le \val(\emptyset)$, we have $\val(\bz_t)<\val(\bz_{t-1})$. 
Indeed, by the definition of $g_t$, we have either $\val(\bz_{t-1},x_t,g_t(\bz_{t-1},x_t))\ge \min\{\val(\bz_{t-1}),\val(\emptyset)\}$ or $\val(\bz_{t-1},x_t,y)<\min\{\val(\bz_{t-1}),\val(\emptyset)\}$ for all $y\in\Y$. 

If $\val(\bz_{t-1},x_t,y)<\min\{\val(\bz_{t-1}),\val(\emptyset)\}$ for all $y\in\Y$, it obviously follows that $\val(\bz_{t})<\min\{\val(\bz_{t-1}),\val(\emptyset)\}$. 
If $\val(\bz_{t-1},x_t,g(\bz_{t-1},x_t))\ge\min\{\val(\bz_{t-1}),\val(\emptyset)\}$, since $\val(\bz_{t-1})\le \val(\emptyset)$ by our assumption, we have $\val(\bz_{t-1},x_t,g(\bz_{t-1},x_t))\ge\val(\bz_{t-1})$. By Corollary \ref{coro:val_eq}, we have
\begin{align*}
\val(\bv_{t-1},\xi_t,1)\ge\val(\bv_{t-1}).
\end{align*}
Then, by Proposition \ref{prop:val_decrease} and Corollary \ref{coro:val_eq}, we must have 
\begin{align*}
\val(\bz_{t})=\val(\bv_{t})=\val(\bv_{t-1},\xi_t,0)<\val(\bv_{t-1})=\val(\bz_{t-1}).
\end{align*}
Thus, the above claim holds.

Now we claim that $\val(\bz_t)\le\val(\emptyset)$ for $t=0$ and $\val(\bz_t)<\val(\emptyset)$ for any $1\le t<\infty$. Indeed, when $t=0$, we have $\val(\bz_0)=\val(\emptyset)$. 
Suppose $\val(\bz_{t-1})\le\val(\emptyset)$ for some $1\le t<\infty$. If $\val(\bv_{t-1})=-1$, by Lemma \ref{lem:val_A},
we have $\val(\bz_t)=\val(\bz_{t-1})=-1<\val(\emptyset)$. If $\val(\bz_{t-1})\ge0$,
we have $\val(\bz_{t})<\val(\bz_{t-1})\le \val(\emptyset)$ by the first claim. Thus, by induction, the claim holds.

By the two claims, we can conclude that $\val(\emptyset)>\val(\bz_1)>\val(\bz_2)>\cdots>\val(\bz_t)$ as long as $\val(\bz_{t})>-1$. If $\val(\emptyset)<\Omega$, by the well-ordering of \ORD, there exists some finite positive integer $\tau$ such that $\val(\bz_{\tau})=-1$. Thus, by Lemma \ref{lem:val_A}, we have $\cH_{\bz_{\tau}}=\emptyset$.
\end{proof}

Now, we can present the proof of Theorem \ref{thm:exp_game}.
\begin{proof}[of Theorem \ref{thm:exp_game}]
Assume that $\cH$ has an infinite Littlestone tree $\{\xi_{\bu}:0\le k<\infty,\bu\in\{0,1\}^k\}$. Define the following strategy for the adversary $\wb P_a$: in each round $t\ge1$, $\wb P_a$ chooses $x_t:=\xi_{(\eta_1,\dots,\eta_{t-1})}(1)\in\X$ with $\eta_i\in\{0,1\}\ (1\le i\le t-1)$ defined later (when $t=1$, we have $\xi_{(\eta_1,\dots,\eta_{t-1})}:=\xi_{\emptyset}$). 
After the learner $\wb P_l$ makes the prediction $\wh y_t$, define
\begin{align*}
\eta_t:=\begin{cases}
&0,\quad \textup{if } \xi_{(\eta_1,\dots,\eta_{t-1})}(2)\neq \wh y_t,\\
&1,\quad \textup{otherwise.} \\
\end{cases}
\end{align*}
Then, $\wb P_a$ reveals the true label $y_t:=\xi_{(\eta_1,\dots,\eta_{t-1})}(\eta_t+2)$.

Since $\xi_{(\wh y_1,\dots,\wh y_{t-1})}(2)\neq \xi_{(\wh y_1,\dots,\wh y_{t-1})}(3)$, we have $y_t\neq \wh y_t$ for each $t\ge1$. 
Besides, by the definition of Littlestone tree (Definition \ref{def:littlestone_tree}), $(x_1,y_1,\dots,x_t,y_t)$ is consistent with $\cH$ for any $(\wh y_1,\dots,\wh y_t)\in\Y^t$ and $t\ge1$. 

Assume that $\cH$ does not have an infinite Littlestone tree.
Consider the following strategy for the learner $\wb P_l$. 
\begin{itemize}
\item Initialize $\tau\leftarrow1$ and $f(x)\leftarrow g_1(x)$.
\item For $t\leftarrow 1,2,3,\dots:$
\begin{itemize}
    \item Predict $\wh y_t=f(x_t)$.
    \item If $\wh y_t\neq y_t:$
    
    set $\wt x_\tau=x_t$, $\wt y_\tau=y_t$, $f(x)\leftarrow g_{\tau+1}(\wt x_1,\wt y_1,\dots,\wt x_\tau,\wt y_\tau,x)$, and $\tau\leftarrow\tau+1$.
\end{itemize}
\end{itemize}
Suppose that there exists some adversary $\wb P_a$ such that $\wb P_l$ makes infinitely many mistakes at $t_1, t_2,\dots$ adopting the above strategy. Then according to Proposition \ref{prop:gt}, there exists some $1\le k<\infty$ such that $\cH_{x_{t_1},y_{t_1},\dots,x_{t_k},y_{t_k}}=\emptyset$. However, this contradicts the rule of the online learning game $\wb \B$ because $\cH_{x_1,y_1,\dots,x_{t_k},y_{t_k}}=\emptyset$.
\end{proof}
Also, the universal measurability of the learning strategy can be proved. 
\begin{corollary}
If $\X$ is a Polish space, $\Y$ is countable, $\cH$ is measurable as defined in Definition \ref{def:measurable}, and $\cH$ does not have an infinite Littlestone tree, then the learning strategy of $\wb P_l$ specified in Theorem \ref{thm:exp_game} is universally measurable.
\end{corollary}
\begin{proof}
Since $\cH$ does not have an infinite Littlestone tree, according to Lemma \ref{lem:val_emptyset}, we have $\val(\emptyset)<\Omega$. Then, by Lemma \ref{lem:val_max}, we have $\val(\emptyset)<\omega_1$. 
Since $\cH$ is measurable, by Lemma \ref{lem:coanalytic}, $\wmW$ is coanalytic. Then, according to Corollary \ref{coro:gt_meas}, $g_t$ is universally measurable for any $1\le t<\infty$. 
According to Corollary \ref{coro:gtz_meas}, $f(x)$ is also universally measurable for any $1\le t<\infty$.
Thus, the learning strategy for $\wb P_l$ specified in Theorem \ref{thm:exp_game} is universally measurable.
\end{proof}

\subsection{Concluding proof} 
\begin{proof}[of Theorem \ref{thm:tree-exp-rate}]
First, according to \citet[Lemma 4.2]{universal_learning}, $\cH$ is not learnable at rate faster than the exponential rate $e^{-n}$. Thus, the proof is completed once we construct a learning algorithm which, for $\cH$ without an infinite Littlestone tree, achieves
exponential rate for any realizable distribution $P$. 
We use the learning algorithm constructed in \citet[Section 4.1]{universal_learning}. 
According to \citet[Lemma 4.3, Lemma 4.4., and Corollary 4.5]{universal_learning} and their proofs, for the learning algorithm to achieve exponential rate, it suffices to have an adversarial online learning algorithm with the properties that it only makes finitely many mistakes against any adversary and it only changes when a mistake is made. 
According to Theorem \ref{thm:exp_game} and its proof, for $\cH$ without an infinite Littlestone tree, the winning strategy constructed in the proof only makes finitely many mistakes against any adversary and changes only when a mistake happens. 
Then, the same proofs of \citet[Lemma 4.3, Lemma 4.4., and Corollary 4.5]{universal_learning} can be applied to show that the constructed online learning algorithm achieves exponential rate. 
In conclusion, if $\cH$ does not have an infinite Littlestone tree, then $\cH$ is learnable with optimal rate $e^{-n}$.
\end{proof}

\section{Near-Linear Rates} \label{sec:linear_rate_proof}
In this section, we prove Theorem \ref{thm:near-linear-rates}. 
\subsection{Slower than exponential is not faster than linear}
In this subsection, we prove the following theorem.
\begin{theorem} \label{thm:linear_rate_lower}
If $\cH$ has an infinite Littlestone tree, then for any learning algorithm $A$, there exists a $\cH$-realizable distribution $P$ such that for infinitely many $n$, $\bE[\er(
\wh h_n)]\ge\frac{1}{33n}$ where $\wh h_n=A(\cH,S_n)$ with $S_n\sim P^n$. Thus, $\cH$ is not learnable at rate faster than $\frac{1}{n}$.
\end{theorem}
\begin{proof}
Suppose that $\cH$ has an infinite Littlestone tree 
$$
\left\{\xi_{\bu}=(x_\bu,y_\bu^0,y_\bu^1):0\le k< d,\bu\in\{0,1\}^k\right\}.
$$ 
Fix an arbitrary learning algorithm $A$. 
Let $\bu=\{u_1,u_2,\dots\}$ be a sequence of i.i.d. Bernoulli$(\frac{1}{2})$ random variables. Conditional on $\bu$, define the distribution $P_{\bu}$ on $\X\times\Y$ by 
\begin{align*}
P_{\bu}(\{x_{\bu_{\le k}},y_{\bu_{\le k}}^{u_{k+1}}\})=2^{-k-1}, \ \forall\ k\ge 0.
\end{align*}
Note that the mapping $\bu\mapsto P_{\bu}$ is measurable. 

By the definition of Littlestone tree, for any $n\ge 0$, there exists a hypothesis $h_n\in\cH$ such that $h_n(x_{\bu_{\le k}})=y_{\bu_{\le k}}^{u_{k+1}}$ for any $0\le k\le n$. Thus, we have
\begin{align*}
\er_{\bu}(h_n):=P_{\bu}(\{x,y)\in\X\times\Y: h_n(x)\neq y\})\le \sum_{k=n+1}^\infty 2^{-k-1}=2^{-n-1}.
\end{align*}
Then, $\inf_{h\in\cH}\er_{\bu}(h)=0$ and $P_{\bu}$ is $\cH$-realizable. 

Let $T,T_1,T_2,\dots$ be i.i.d. random variables with distribution Geometric$(\frac{1}{2})$ (starting from 0). 
Define $X:=x_{\bu_{\le T}}$, $Y:=y_{\bu_{\le T}^{u_{T+1}}}$, $X_i:=x_{\bu_{\le T_i}}$, and $Y_i:=y_{\bu_{\le T_i}^{u_{T_i+1}}}$ for any $i\ge 1$. 
Then, conditional on $\bu$, by the definition of $P_\bu$, we know that $(X,Y),(X_1,Y_1),(X_2,Y_2),\dots$ is a sequence of i.i.d. random variables with distribution $P_\bu$. 
Now, define $\wh h_n=A(\cH,((X_1,Y_1),\dots,(X_n,Y_n)))$.
For any $k\ge 1$, since $u_1,u_2,\dots$ are i.i.d. Bernoulli$(\frac{1}{2})$ random variables, we have
\begin{align*}
&\pr(\wh h_n(X)\neq Y,T=k,\max\{T_1,\dots,T_n\}< k)
\\=&
\pr(\wh h_n(X)\neq y_{\bu_{\le k}}^{u_{k+1}},T=k,\max\{T_1,\dots,T_n\}< k)
\\=&
\bE[\pr(\wh h_n(X)\neq y_{\bu_{\le k}}^{u_{k+1}}|X,T,T_1,\dots,T_n)\indi\{T=k,\max\{T_1,\dots,T_n\}< k\}]
\\=&
\bE\left[\frac{1}{2}\indi\{T=k,\max\{T_1,\dots,T_n\}< k\}\right]
\\=&
\frac{1}{2}\pr(T=k,\max\{T_1,\dots,T_n\}< k\})
\\=&
2^{-k-2}(1-2^{-k})^n.
\end{align*}
Define $k_n:=\lceil 1+\log_2(n)\rceil$ for $n\ge1$. Then, we have $2^{-k_n-2}> \frac{1}{16n}$ and $(1-2^{-k_n})^n\ge (1-\frac{1}{2n})^{n}\ge \frac{1}{2}$, which, together with the above result, implies that
\begin{align*}
\pr(\wh h_n(X)\neq Y,T=k_n)
\ge &
\pr(\wh h_n(X)\neq Y,T=k_n,\max\{T_1,\dots,T_n\}< k_n)
\\ \ge &
2^{-k_n-2}(1-2^{-k_n})^n
\\ > &
\frac{1}{32n}.
\end{align*}
Since 
$$
n\pr(\wh h_n(X)\neq Y,T=k_n|\bu)\le n\pr(T=k_n|\bu)=n\pr(T=k_n)=n2^{-k_n-1}\le \frac{1}{4}\ \ \textup{a.s.},
$$
by Fatou's lemma, we have
\begin{align*}
\bE[\limsup_{n\rightarrow\infty}n\pr(\wh h_n(X)\neq Y,T=k_n|\bu)]
\ge \limsup_{n\rightarrow\infty}n\pr(\wh h_n(X)\neq Y,T=k_n)
\ge \frac{1}{32}.
\end{align*}
Since 
\begin{align*}
\pr(\wh h_n(X)\neq Y,T=k_n|\bu)\le 
\pr(\wh h_n(X)\neq Y|\bu)=\bE[\er_{\bu}(\wh h_n)|\bu]\ \ \textup{a.s.},
\end{align*}
we have $\bE[\limsup_{n\rightarrow\infty}n\bE[\er_\bu(\wh h_n)|\bu]]\ge\frac{1}{32}>\frac{1}{33}$. Thus, there exists $\bu'\in\{0,1\}^{\infty}$ such that $\bE[\er_{\bu'}(\wh h_n)] \ge \frac{1}{33n}$ infinitely often. The proof is completed by setting $P=P_{\bu'}$. 
\end{proof}

\subsection{Pattern avoidance functions} \label{sec:patern-avoid}
In this subsection, we design pattern avoidance functions in the adversarial setting and analyze their properties.
For any $n\in\N$ and hypothesis class $\cH\subseteq\Y^n$, denote the collection of all $n$-dimensional pseudo-cubes of $\cH$ with $\PC(\cH)$. For any $m\in\N$, denote the collection of all $n$-dimensional pseudo-cubes of $\cH$ of size $m$ with $\PC_m(\cH)$. Then, we have $\PC(\cH)=\cup_{m=1}^\infty\PC_m(\cH)$.
For any hypothesis class $F\subseteq\Y^n$, let $Q(F)$ denote the union of all the pseudo-cubes of dimension $n$ that are subsets of $F$. 

Consider the following game $\fB$ between player $P_A$ and $P_L$. At each round $\tau\ge1$:
\begin{itemize}
\item Player $P_A$ chooses a sequence $\bx_{\tau}=(x_{\tau}^0,\dots,x_{\tau}^{\tau-1})\in\X^{\tau}$ and a set  $C_{\tau}\in\PC(\Y^{\tau})$. 
\item Player $P_L$ chooses a sequence $\by_{\tau}=(y_{\tau}^0,\dots,y_{\tau}^{\tau-1})\in \Y^{\tau}$. 
\item Player $P_L$ wins the game in round $\tau$ if 
\begin{itemize}
\itemsep0em
\item either $C_\tau \notin\PC(\cH|_{\bx_\tau})$,
\item or $\by_{s}\in C_s$ for all $1\le s\le \tau$ and $\cH_{\bx_1,\by_1,\dots,\bx_{\tau},\by_{\tau}}=\emptyset$, where

$\mathcal{H}_{\bx_{1}, \by_{1}, \ldots, \bx_{\tau}, \by_{\tau}}:=\left\{h \in \mathcal{H}: h\left(x_{s}^{i}\right)=y_{s}^{i} \text { for } 0 \leq i<s, 1 \leq s \leq \tau\right\}.$
\end{itemize}
\end{itemize}
The set of winning sequences of $P_L$ in $\fB$ is 
\begin{align*}
\W_{\fB}:=\Bigg\{&(\bx_1,C_1,\by_1,\dots)\in\prod_{t=1}^\infty(\X^t\times\PC(\Y^t)\times\Y^t):\
\exists \tau\in\N \textup{ such that}
\\&\quad\quad
\textup{either }
C_\tau \nsubseteq\PC(\cH|_{\bx_{\tau}}),
\textup{ or }
\by_t\in C_t \textup{ for all } t\in[\tau] \textup{ and } \cH_{\bx_1,\by_1,\dots,\bx_\tau,\by_\tau}=\emptyset \Bigg\}
\end{align*}
Obviously, $\W_{\fB}$ is finitely decidable, which implies that $\fB$ is a Gale-Stewart game and according to \citet{MR0054922}, either $P_A$ or $P_L$ has a winning strategy. 

With regard to the universal measurability of the winning strategy, we assume that $\X$ is a Polish space, $\Y$ is countable, and $\cH$ is measurable in the sense of Definition \ref{def:measurable}. We first prove the following lemma.
\begin{lemma} \label{lem:PC_countable}
For any $t\in\N$ and $\bx_t\in\X^t$, $\PC(\Y^t)$ and $\PC(\cH|_{\bx_t})$ are countable sets.
\end{lemma}
\begin{proof}
Since $\Y$ is countable, $\Y^t$ and $\cH|_{\bx_t}$ are also countable. By the definition of pseudo-cube, any pseudo-cube is a finite subset of the hypothesis class. Since the set of all finite subsets of a countable set is countable, $\PC(\cH|_{\bx_t})$ and $\PC(\Y^t)$ are countable sets. 
\end{proof}
For any $t\in\N$, define the set 
\begin{align*}
\XPC_t:=\cup_{\bx_t\in\X^t}\{\bx_t\}\times \PC(\cH|_{\bx_t})\subseteq\X^t\times\PC(\Y^t)
\end{align*}
Then, we can prove the following property of $\XPC_t$.
\begin{lemma} \label{lem:XPC_analytic}
For any $t\in\N$, $\XPC_t$ is an analytic subset of the Polish space $\X^t\times\PC(\Y^t)$.
\end{lemma}
\begin{proof}
According to Lemma \ref{lem:PC_countable}, $\PC(\Y^t)$ is countable. Thus, $\X^t\times\PC(\Y^t)$ is a Polish space.
For any $t\in\N$, we have 
\begin{align*}
\XPC_t=\cup_{n=1}^\infty\Big(&
\Big(\X^t\times\PC_n(\Y^t)\Big)\cap 
\\&
\cup_{(\theta_1,\dots,\theta_n)\in\Theta^n}\Big\{(\bx,\by^1,\dots,\by^n)\in\X^t\times\Y^{tn}:\sfh(\theta_i,\bx)=\by^i \textup{ for all }i\in[n]\Big\}
\Big)
\end{align*}
where by $\sfh(\theta,(x_1,\dots,x_t))=(y_1,\dots,y_t)$, we mean that $\sfh(\theta,x_\tau)=y_\tau$ for all $\tau\in[t]$.
Indeed, for any $(\bx,C)\in\XPC_t$, we have $\bx\in\X^t$ and $C\in\PC(\cH|_{\bx})$. Then, by the definition of pseudo-cubes, there exists a finite $n\in\N$ such that $C\in\PC_n(\cH|_{\bx})$. Since $\cH|_{\bx}\subseteq\Y^t$, we have $C\in\PC_n(\Y^t)$ and $(\bx,C)\in\X^t\times\PC_n(\Y^t)$. Moreover, since $C\subseteq\cH|_{\bx}$ with $|C|=n$, we can write $C=\{\by^1,\dots,\by^n\}$ such that there exist $(h_1,\dots,h_n)\in\cH^n$ satisfying $h_i(\bx)=\by^i$ for any $i\in[n]$. By Definition \ref{def:measurable}, there exist $(\theta_1,\dots,\theta_n)\in\Theta^n$ such that $\sfh(\theta_i,\bx)=\by^i$ for all $i\in[n]$.

On the other hand, if $(\bx,\{\by^1,\dots,\by^n\})\in\X^t\times\PC_n(\Y^t)$ is such that there exist $(\theta_1,\dots,\theta_n)\in\Theta^n$ satisfying $\sfh(\theta_i,\bx)=\by^i$ for all $i\in[n]$, we have $C:=\{\by^1,\dots,\by^n\}\subseteq\cH|_{\bx}$ and $C$ is a pseudo-cube of dimension $t$. Thus, $C\in\PC(\cH|_{\bx})$ and $(\bx,\{\by^1,\dots,\by^n\})\in\XPC_t$.

We claim that the set 
\begin{align*}
S_{t,n}:=\left\{(\theta_1,\dots,\theta_n,\bx,\by^1,\dots,\by^n)\in\Theta^n\times\X^t\times\Y^{tn}:\sfh(\theta_i,\bx)=\by^i\textup{ for all }i\in[n]\right\}
\end{align*}
is a Borel set. The reason is as follows. 
For any $i\in[n]$, define the function 
$$
l:\Theta^n\times\X^t\times\Y^{tn}\rightarrow \{0,1,\dots,nt\},\ \ 
(\theta_1,\dots,\theta_n,\bx,\by^1,\dots,\by^n)\mapsto \sum_{i=1}^n\sum_{\tau=1}^t
\indi\{\sfh(\theta_i,x_\tau)\neq y^i_\tau\}.
$$
Since $\sfh$ is Borel-measurable, we can conclude that $l$ is also Borel-measurable with the argument analogous to that in the proof of Lemma \ref{lem:coanalytic}. Thus, $S_{t,n}=l^{-1}(\{0\})$ is a Borel set. Then, the set 
\begin{align*}
&\cup_{(\theta_1,\dots,\theta_n)\in\Theta^n}\left\{(\bx,\by^1,\dots,\by^n)\in\X^t\times\Y^{tn}:\sfh(\theta_i,\bx)=\by^i \textup{ for all }i\in[n]\right\}
\\=&
\cup_{(\theta_1,\dots,\theta_n)\in\Theta^n}\left\{(\bx,\by^1,\dots,\by^n)\in\X^t\times\Y^{tn}:(\theta_1,\dots,\theta_n,\bx,\by^1,\dots,\by^n)\in S_{t,n}\right\}
\end{align*}
is an analytic set for any $t,n\in\N$. 
Since $\PC(\Y^t)$ is countable, we know that $\PC_n(\Y^t)$ is countable. Since $\X$ is a Polish space, we have that $\X^t\times\PC_n(\Y^t)$ is an analytic set. In conclusion, $\XPC_t$ is an analytic set for any $t\in\N$.
\end{proof}

Now, for any $0\le n\le \infty$, define $\ms P_n:=\prod_{t=1}^n(\X^t\times\PC(\Y^t)\times\Y^t)$ which is the set of positions of length $n$ of the game $\fB$ and $\wt{\ms P}_n:=\prod_{t=1}^n(\XPC_t\times\Y^t)\subseteq \ms P_n$. 
Define $\ms P:=\cup_{n=0}^\infty\ms P_n$ (with $\ms P_0:=\emptyset$) to be the set of all positions of the Gale-Stewart game $\fB$. 
We can show the following results according to Lemma \ref{lem:XPC_analytic}
\begin{corollary} \label{coro:P_n_analytic}
For any $0\le n\le \infty$, $\wt{\ms P}_n$ is an analytic subset of the Polish space $\ms P_n$.
\end{corollary}
\begin{proof}
Since $\Y^t$ and $\PC(\Y^t)$ are countable and $\X$ is a Polish space, $\ms P_n=\prod_{t=1}^n(\X^t\times\PC(\Y^t)\times\Y^t)$ is also a Polish space for any $0\le n\le \infty$.
By Lemma \ref{lem:XPC_analytic}, we know that $\XPC_t\times\Y^t$ is an analytic subset of $\X^t\times\PC(\Y^t)\times\Y^t$ for any $0\le t<\infty$. Then, we have that $\wt{\ms P}_n$ is an analytic subset of $\ms P_n$ for any $0\le n<\infty$.
For $n=\infty$, we have that
\begin{align*}
\wt{\ms P}_\infty=\cap_{n=1}^\infty\left(\wt{\ms P}_n\times\prod_{t=n+1}^\infty(\X^t\times\PC(\Y^t)\times\Y^t)\right).
\end{align*}
Since $\wt{\ms P}_n\times\prod_{t=n+1}^\infty(\X^t\times\PC(\Y^t)\times\Y^t)$ is an analytic subset of $\ms P_{\infty}=\prod_{t=1}^\infty(\X^t\times\PC(\Y^t)\times\Y^t)$ for any $1\le n<\infty$, we have that $\wt{\ms P}_{\infty}$ is also an analytic subset of $\ms P_{\infty}$. 
\end{proof}

Then, we can proceed to show that
\begin{lemma} \label{lem:coanalytic_linear}
$\ms P_\infty\backslash\W_\fB$ is an analytic set.
\end{lemma}
\begin{proof}
We have
\begin{align*}
\ms P_\infty\backslash\W_\fB=&
\Big\{
(\bx_1,C_1,\by_1,\dots)\in\prod_{t=1}^\infty(\XPC_t\times\Y^t):
\textup{ for all }\tau\in\N, 
\textup{ either } \cH_{\bx_1,\by_1,\dots,\bx_\tau,\by_\tau}\neq\emptyset 
\\&\quad\quad\quad\quad\quad\quad\quad\quad\quad\quad\quad\quad\quad\quad\quad\quad
\textup{ or } \exists\ s\in[\tau] \textup{ s.t. } \by_s\notin C_s
\Big\}
\\=&
\cap_{\tau=1}^\infty
\Bigg(\Big(\cup_{\theta\in\Theta}
\cap_{s=1}^{\tau}\Big\{(\bx,C_1,\by_1,\dots,)\in\prod_{t=1}^\infty(\XPC_t\times\Y^t): \sfh(\theta,\bx_s)=\by_s\Big\}\Big)
\\&\cup \Big(\cup_{s=1}^\tau
\Big\{(\bx,C_1,\by_1,\dots,)\in\prod_{t=1}^\infty(\XPC_t\times\Y^t):\by_s\notin C_s\Big\}\Big)\Bigg).
\end{align*}
By Lemma \ref{lem:PC_countable}, Lemma \ref{lem:XPC_analytic}, and Definition \ref{def:measurable}, for any $s\in\N$,
\begin{align*}
&\Big\{(\theta,\bx,C_1,\by_1,\dots,)\in\Theta\times\prod_{t=1}^\infty(\XPC_t\times\Y^t): \sfh(\theta,\bx_s)=\by_s\Big\}
\\=&
\Big(\Theta\times\prod_{t=1}^\infty(\XPC_t\times\Y^t)\Big)\cap
\Big\{(\theta,\bx,C_1,\by_1,\dots,)\in\Theta\times\prod_{t=1}^\infty(\X^t\times\PC(\Y^t)\times\Y^t): \sfh(\theta,\bx_s)=\by_s\Big\}
\end{align*}
is an analytic set and
\begin{align*}
&\Big\{(\bx,C_1,\by_1,\dots,)\in\prod_{t=1}^\infty(\XPC_t\times\Y^t):\by_s\notin C_s\Big\}
\\=&
\Big(\prod_{t=1}^\infty(\XPC_t\times\Y^t)\Big)\cap 
\Big\{(\bx,C_1,\by_1,\dots,)\in\prod_{t=1}^\infty(\X^t\times\PC(\Y^t)\times\Y^t):\by_s\notin C_s\Big\}
\end{align*}
is also an analytic set. Thus, we have
$\ms P_\infty\backslash\W_{\fB}$ is an analytic set.
\end{proof}

We have the following lemma relating a winning strategy of $P_A$ in $\fB$ to an infinite DSL tree of $\cH$. 
\begin{lemma} \label{lem:PA_winning_DSL}
$P_A$ has a winning strategy in $\fB$ if and only if $\cH$ has an infinite DSL tree.
\end{lemma}
\begin{proof}
Suppose that $P_A$ has a winning strategy $\xi_\tau:\prod_{t=1}^{\tau-1}\Y^t\rightarrow\X^\tau\times\PC(\Y^{\tau})$ for all $\tau\in\N$ in $\fB$. 
Specifically, for any $(\by_1,\cdots,\by_{\tau-1})\in\prod_{t=1}^{\tau-1}\Y^t$, we have $\xi_\tau(\by_1,\cdots,\by_{\tau-1})=(\bx_{\tau},C)$ for some $\bx_\tau\in\X^\tau$ and $C\in\PC(\cH|_{\bx_\tau})$.
For notational convenience, let  $\xi_\tau(\by_1,\cdots,\by_{\tau-1})(1)$ denote $\bx_{\tau}$ and let $\xi_\tau(\by_1,\cdots,\by_{\tau-1})(2)$ denote $C$. Now, define the following infinite tree by induction. 
\begin{itemize}
\item Let the root node of the tree be labelled with $\xi_{1}(\emptyset)(1)\in\X$ and have $|\xi_{1}(\emptyset)(2)|$ children such that each edge between the root node and its children is labelled with a unique element in $\xi_{1}(\emptyset)(2)$. 
\item Suppose that for some $\tau\in\N$, all the nodes in level $0,1,\dots,\tau$ have been defined, all the nodes in level $0,1,\dots,\tau-1$ have been labelled, and the edges between each node in level $k$ and its children have been labelled for all $k\in\{0,1,\cdots,\tau-1\}$.

Then, for each node $v$ in level $\tau$, denote the labels of the edges along the path eliminating from the root node to node $v$ with $\by_1\in\Y^{1}$, $\by_2\in\Y^{2}$, \dots, and $\by_{\tau}\in\Y^{\tau}$. Now, let node $v$ be labelled with $\xi_{\tau+1}(\by_1,\cdots,\by_\tau)(1)$ and have $|\xi_{\tau+1}(\by_1,\cdots,\by_\tau)(1)|$ children such that each edge between node $v$ to one of its children is labelled with a unique element in $\xi_{\tau+1}(\by_1,\cdots,\by_\tau)(2)$. 
\end{itemize}
By the definition of the winning strategy of $P_A$ in $\fB$, the infinite tree defined above is an infinite DSL tree for $\cH$.

For the other direction, suppose that $\cH$ has an infinite DSL tree. For any $k\in\N_0$, denote the set of nodes in the level $k$ of the infinite DSL tree with $V_{k}$. 
Note that if there exists some $\tau\in\N$ such that $C_{t}\subseteq \cH|_{\bx_t}$ for all $t\in[\tau]$, $\cH|_{\bx_1,\by_1,\dots,\bx_{\tau-1},\by_{\tau-1}}\neq \emptyset$, and $\by_\tau\notin C_\tau$, then $P_A$ wins in round $\tau$ of $\fB$.
Define the following strategy $\xi_{\tau}:\prod_{t=1}^{\tau-1}\Y^t\rightarrow\X^{\tau}$ for $P_A$ in $\fB$ and a corresponding node mapping $v_\tau:\prod_{t=1}^{\tau-1}\Y^t\rightarrow V_{\tau-1}$ by induction for all $\tau\in\N$. 
\begin{itemize}
\item For $\tau=1$, let $v_{1}(\emptyset)$ denote the root node, $\xi_1(\emptyset)(1)$ denote the label of the root node of the DSL tree, and $\xi_1(\emptyset)(2)$ denote the pseudo-cube in $\cH|_{\xi_1(\emptyset)(1)}$ consisting of the labels of all the edges between the root node and its children. 

\item Suppose that for some $\tau\in\{2,3,\dots\}$, $\eta_t$ and $v_t$ has been defined for all $t\in[\tau-1]$. For any $\by_1\in\Y^1,\dots, \by_{\tau-1}\in\Y^{\tau-1}$, there are two cases.
\begin{itemize}
\item If $P_A$ has not won in round $\tau-1$, define $v_\tau(\by_1,\dots,\by_{\tau-1})$ to be the node in $V_{\tau-1}$ which is the ending node of the path in the DSL tree eliminating from the root along the edges labelled with $\by_1,\dots,\by_{\tau-1}$.
Define $\xi_{\tau}(\by_1,\dots,\by_{\tau-1})(1)$ to be the label of $v_\tau(\by_1,\dots,\by_{\tau-1})$. 
Define $\xi_{\tau}(\by_1,\dots,\by_{\tau-1})(2)$ to be the pseudo-cube in $\cH|_{\xi_{\tau}(\by_1,\dots,\by_{\tau-1})(1)}$ consisting of the labels of all the edges between $v_\tau(\by_1,\dots,\by_{\tau-1})$ and its children. 

\item If $P_A$ has already won, define $v_\tau(\by_1,\dots,\by_{\tau-1})$ to be the first child node of $v_{\tau-1}(\by_1,\dots,\by_{\tau-2})$.
Define $\xi_\tau(\by_1,\dots,\by_{\tau-1})(1)$ to be the label of $v_\tau(\by_1,\dots,\by_{\tau-1})$. 
Define $\xi_{\tau}(\by_1,\dots,\by_{\tau-1})(2)$ to be the pseudo-cube in $\cH|_{\xi_{\tau}(\by_1,\dots,\by_{\tau-1})(1)}$ consisting of the labels of all the edges between $v_\tau(\by_1,\dots,\by_{\tau-1})$ and its children.
\end{itemize}
\end{itemize}
According to the definition of DSL trees and the rules of $\fB$, $\{\xi_\tau\}_{\tau\in\N}$ is a winning strategy of $P_A$ in $\fB$. 
\end{proof}

Moreover, we can ensure that there is a universally measurable winning strategy for $P_L$ in $\fB$ when $\cH$ does not have an infinite DSL tree. 
\begin{proposition} \label{prop:DSL_winning}
If $\cH$ does not have an infinite DSL tree, then there is a universally measurable winning strategy for $P_L$ in $\fB$.
\end{proposition}
\begin{proof}
Since $\fB$ is a Gale-Stewart game, according to Lemma \ref{lem:PA_winning_DSL} and \citet[Theorem A.1]{universal_learning}, we have that if $\cH$ does not have an infinite DSL tree, then there is a winning strategy for $P_L$ in $\fB$.

According to Lemma \ref{lem:PC_countable}, Corollary \ref{coro:P_n_analytic}, and Lemma \ref{lem:coanalytic_linear}, we know that $\fB$ is a Gale-Stewart game such that the action sets of $P_A$ $(\X^t\times\PC(\Y^t),\ t\in\N)$ are Polish spaces, the action sets of $P_L$ $(\Y^t,\ t\in\N)$ are countable, and the set of winning sequences $\W_{\fB}$ for $P_L$ is coanalytic. Then, according to \citet[Theorem B.1]{universal_learning}, 
$P_L$ has a universally measurable winning strategy ($g_t:\ms P_{t-1}\times\X^t\times\PC(\Y^t)\rightarrow\Y^t$, $t\in\N$) for $P_L$ in $\fB$ if $\cH$ does not have an infinite DSL tree.
For completeness, we provide the explicit definition of $g_t$ below for $t\in\N$ according to the proof of \citet[Theorem B.1]{universal_learning}. For any $\bv\in\ms P_{t-1}$, $\bx\in\X^t$, and $C\in\PC(\Y^t)$, enumerate $\Y^t$ as $\{\by^{(t,i)}\}_{i\in\N}$ and define
\begin{align*}
g_t(\bv,\bx,C):=
\begin{cases}
&\by^{(t,i)}\quad \textup{if } \val(\bv,\bx,C,\by^{(t,j)})\ge\min\{\val(\bv),\val(\emptyset)\} \textup{ for all } 1\le j<i 
\\& \ \ \quad\quad\quad 
\textup{ and }
\val(\bv,\bx,C,\by^{(t,i)})<\min\{\val(\bv),\val(\emptyset)\},
\\&
\by^{(t,1)} \quad \textup{if }
\val(\bv,\bx,C,\by^{(t,j)})\ge\min\{\val(\bv),\val(\emptyset)\} \textup{ for all } j\in\N.
\end{cases}
\end{align*}
\end{proof}

From now on we assume that $\cH$ does not have an infinite DSL tree. 
Analogous to Definition \ref{def:val-binary}, we define the game value $\val:\ms P\rightarrow\ORD^*$ according to \citet[Definition B.5]{universal_learning}. 
For any $\tau\in\N$, define the mapping $\eta_{\tau}:\prod_{t=1}^{\tau} \left(\XPC_t\times\Y^t\right)\rightarrow \{0,1\}$ by
\begin{align*}
\eta_{\tau}(\bv,\bx,C,\by):=\begin{cases}
&1\quad \textup{if } \val(\bv,\bx,C,\by)<\min\{\val(\bv),\val(\emptyset)\},\\
&0\quad \textup{otherwise,}
\end{cases}
\end{align*}
for any $\bv\in\prod_{t=1}^{\tau-1} \left(\XPC_t\times\Y^t\right)$, $(\bx,C)\in\XPC_{\tau}$, and $\by\in\Y^\tau$. 
Define the following online algorithm which given a sequence of feature-label pairs $(x_1,y_1,x_2,y_2,\dots)\in(\X\times\Y)^\infty$ chooses a sequence of elements in $\cup_{\tau=0}^\infty\prod_{t=1}^\tau(\XPC_t\times\Y^t)$ (``patterns''):
\begin{itemize}
\item Initialize $\tau_0\leftarrow1$.
\item At every time step $t\in\N$:
\begin{itemize}
    \item Let $\tau_t\leftarrow\tau_{t-1}$.
    \item For each $C\in\PC(\cH|_{(x_{t-\tau_{t-1}+1},\dots,x_t)})$:
    \begin{itemize}
        \item If 
        \begin{align*}
        \eta_{\tau_{t-1}}\Big(&\wb\bx_1,\wb C_1,\wb\by_1,\dots,\wb\bx_{\tau_{t-1}-1},\wb C_{\tau_{t-1}-1},\wb\by_{\tau_{t-1}-1},
        \\&
        (x_{t-\tau_{t-1}+1},\dots,x_t),C, (y_{t-\tau_{t-1}+1},\dots,y_t)\Big)=1:
        \end{align*}
        \begin{itemize}
            \item Let $\wb\bx_{\tau_{t-1}}\leftarrow(x_{t-\tau_{t-1}+1},\dots,x_t)$, $\wb C_{\tau_{t-1}}\leftarrow C$, $\wb\by_{\tau_{t-1}}\leftarrow(y_{t-\tau_{t-1}+1},\dots,y_t)$, and $\tau_{t}\leftarrow\tau_{t-1}+1$.
            \item Break.
        \end{itemize}
    \end{itemize}
\end{itemize}
\end{itemize}

We use $\wh \by_t$ to denote the `` pattern avoidance mapping'' defined after time step $t$ of the above algorithm; specifically, we define
\begin{align*}
\wh \by_t(x'_1,\dots,x'_{\tau_{t}}):=&\cup_{C\in\PC(\cH|_{(x'_1,\dots,x'_{\tau_{t}})})}
\\&\left\{
\by'\in C:\eta_{\tau_{t}}(\wb\bx_1,\wb C_1,\wb\by_1,\dots,\wb\bx_{\tau_{t}-1},\wb C_{\tau_{t}-1},\wb\by_{\tau_{t}-1},(x'_1,\dots,x'_{\tau_t}),C,\by')=1
\right\}
\end{align*}
for any $t\ge0$ and $(x'_1,\dots,x'_{\tau_{t}})\in\X^{\tau_{t}}$.
From the above algorithm, we can also define the following functions for any $t\ge0$,
\begin{align*}
T_{t}:(\X\times\Y)^{t}\rightarrow\{1,\dots,t+1\},\ \ 
(x_1,y_1,\dots,x_{t},y_{t})\mapsto \tau_{t},
\end{align*}
and 
\begin{align} \label{eq:pattern-avoid-func}
\wh \bY_{t}:(\X\times\Y)^t\times\cup_{s=1}^{t+1}\X^s\rightarrow \cup_{s=1}^{t+1}2^{\Y^s},\ \ 
(x_1,y_1,\dots,x_t,y_t,x'_1,\dots,x'_{\tau_{t}})\mapsto \wh \by_t(x'_1,\dots,x'_{\tau_{t}}).
\end{align}
We have the following proposition.
\begin{proposition} \label{prop:pattern_avoid}
For any sequence $x_1,y_1,x_2,y_2,\dots$ that is consistence with $\cH$, we have
\begin{align*}
(y_{t-\tau_{t-1}+1},\dots,y_{t})
\notin
\wh\by_{t-1}(x_{t-\tau_{t-1}+1},\dots,x_{t}),\ \tau_{t-1}=\tau_{t}<\infty,\ \textup{and }\wh\by_{t-1}=\wh\by_{t}
\end{align*}
for all sufficiently large $t$.
\end{proposition}
\begin{proof}
Suppose that there is an infinite sequence of times $1\le t_1< t_2<\cdots$ such that 
$$
(y_{t_i-\tau_{t_i-1}+1},\dots,y_{t_i})\in \wh\by_{t_i-1}(x_{t_i-\tau_{t_i-1}+1},\dots,x_{t_i})
$$
for any $i\in\N$.
Define $\bx_i:=(x_{t_i-\tau_{t_i-1}+1},\dots,x_{t_i})$, $\by_i:=(y_{t_i-\tau_{t_i-1}+1},\dots,y_{t_i})$, $C_{i}:=\wb C_{\tau_{t_i-1}}$, and $\bv_{i}:=(\bx_1,C_1,\by_1,\dots,\bx_i,C_i,\by_i)$ for $i\in\N$. 
Since $\cH$ does not have an infinite DSL tree and $\W_{\fB}$ is coanalytic, we have $\val(\emptyset)<\omega_1$. Thus, there is no infinite value-decreasing sequence, which, together with the definition of $\wh\by_t$, implies that $\val(\bv_{k})=-1$ for some $k\in\N$; i.e., $P_L$ wins at round $k$ of $\fB$ under the sequence of positions $\bv_k$. 
Since we have ensured that $C_i\subseteq \PC(\cH|_{\bx_i})$ for all $i\in\N$, we must have $\cH_{\bx_1,\by_1,\dots,\bx_k,\by_k}=\emptyset$ by the winning rule of $P_L$ in $\fB$. 
However, this contracts the assumption that the sequence $(x_1,y_1,x_2,y_2,\dots)$ is consistent with $\cH$. 
Thus, there exists some $t_0\in\N$ such that 
$$
(y_{t-\tau_{t-1}+1},\dots,y_t)\notin\wh\by_{t-1}(x_{t-\tau_{t-1}+1},\dots,x_t)
$$ 
for all $t\ge t_0$. Then, according to the definition of $\tau_t$ and $\wh\by_t$, we have $\tau_{t-1}=\tau_t\le t_0<\infty$ and $\wh\by_{t-1}=\wh\by_t$ for all $t\ge t_0$.
\end{proof}

\subsection{Universal measurability}
In this section, we prove the following proposition about the universal measurability of the functions $T_{t}$ and $\wh\bY_t$ defined in the previous section.
\begin{proposition} \label{prop:meas_pattern}
For any $t\ge0$, $T_{t}$ and $\wh\bY_t$ are universally measurable. 
\end{proposition}
We start with some definitions of the building blocks for analyzing the universal measurability. 
For any $t\in\N$ and any $s\in[t]$, fix an arbitrary sequence $1\le j_1\le j_2\le\dots\le j_s\le t-s+1$.
Define $j_0:=0$ and $J_0:=0$. For any $i\in[s]$, define 
\begin{align*}
J_i:=&1+\sum_{k=2}^i \left[k-\left((j_{k-1}+k-2-j_k+1)\lor0\right)\right]\\
=&\frac{i(i+1)}{2}-\sum_{k=2}^i \left((j_{k-1}+k-1-j_k)\lor0\right)
\end{align*}
and
\begin{align*}
I_i:=&J_{i-1}+1-\left((j_{i-1}+i-2-j_i+1)\lor 0\right)=
J_{i-1}+\left((j_i-j_{i-1}-i+2)\land 1\right).
\end{align*}

In this section, for any $k\in\N$, $i,j\in[k]$, and $k$-tuple $\bz=(z^1,z^2,\dots,z^k)$, let $z^{i:j}$ denote the subtuple $(z^i,z^{i+1},\dots,z^j)$ if $i\le j$ and denote $\emptyset$ if $i>j$. We assume the convention that $\emptyset=\emptyset$. 

For any $0\le i\le s$, define
\begin{align*}
\ms F_{j_1,\dots,j_i}:=
\big\{&(\bx_1,C_1,\by_1\dots,\bx_i,C_i,\by_i)\in\prod_{k=1}^i(\X^k\times\PC(\Y^k)\times\Y^k):
\\&
x_k^{1:(j_{k-1}+k-1-j_k)}=x_{k-1}^{(j_k-j_{k-1}+1):(k-1)} \textup{ and }
\\&
y_k^{1:(j_{k-1}+k-1-j_k)}=y_{k-1}^{(j_k-j_{k-1}+1):(k-1)}
\textup{ for all }2\le k\le i
\big\}.
\end{align*}
Then, $\ms F_{j_1,\dots,j_i}$ is a closed subset of $\prod_{k=1}^i(\X^k\times\PC(\Y^k)\times\Y^k)$ and is also analytic. 
Define
\begin{align*}
\wt{\ms F}_{j_1,\dots,j_i}:=\ms F_{j_1\dots,j_i}\times \prod_{k=i+1}^\infty(\X^k\times\PC(\Y^k)\times\Y^k).
\end{align*}
Then, $\wt{\ms F}_{j_1,\dots,j_i}$ is an analytic subset of $\prod_{k=1}^\infty(\X^k\times\PC(\Y^k)\times\Y^k)$.

For any $C_1\in\PC(\Y^1),\dots,C_i\in\PC(\Y^i)$, define
Define
\begin{align*}
\ms Z_{C_1,\dots,C_i}:=
\left(\prod_{k=1}^i(\X^k\times\{C_k\}\times\Y^k)\right)
\end{align*}
and
\begin{align*}
\wt{\ms Z}_{C_1,\dots,C_i}:=
\left(\prod_{k=1}^i(\X^k\times\{C_k\}\times\Y^k)\right)\times\left(\prod_{k=i+1}^\infty(\X^k\times\PC(\Y^k)\times\Y^k)\right).
\end{align*}
Since $\PC(\Y^t)$ is countable by Lemma \ref{lem:PC_countable}, we have that $\ms Z_{C_1,\dots,C_i}$ is an analytic subset of $\prod_{k=1}^i(\X^k\times\PC(\Y^k)\times\Y^k)$. Thus, $\wt{\ms Z}_{C_1,\dots,C_i}$ is an analytic subset of $\prod_{k=1}^\infty(\X^k\times\PC(\Y^k)\times\Y^k)$.

Define
\begin{align*}
\XY_{j_1,\dots,j_i,C_1,\dots,C_i}:=\Big\{&(x_1,y_1,\dots,x_{J_i},y_{J_i})\in(\X\times\Y)^{J_i}:
\\&
(x_1,C_1,y_1,x_{I_2:(I_2+1)},C_2,y_{I_2:(I_2+1)},\dots, x_{I_i:(I_i+i-1)},C_i,y_{I_i:(I_i+i-1)})\in \wt{\ms P}_i
\Big\}.
\end{align*}
Then, we have
\begin{align*}
&\XY_{j_1,\dots,j_i,C_1,\dots,C_i}
\\=&
\bigcup_{(C_1',\dots,C_i')\in\prod_{k=1}^i\PC(\Y^k)}
\\&
\bigcup_{\left(x_{2}^{1:((j_{1}+1-j_2)\lor0)},y_{2}^{1:((j_{1}+1-j_2)\lor0)},\dots,x_{i}^{1:((j_{i-1}+i-1-j_i)\lor0)},y_{i}^{1:((j_{i-1}+i-1-j_i)\lor0)}\right) \in \prod_{k=1}^i(\X\times\Y)^{1:((j_{k-1}+k-1-j_k)\lor0)}}
\\&
\Big\{(x_1,y_1,\dots,x_{J_i},y_{J_i})\in(\X\times\Y)^{J_i}:
\\&
\Big(x_1,C_1',y_1,
\big(x_{2}^{1:((j_{1}+1-j_2)\lor0)},x_{(J_1+1):J_2}\big),
C_2',
\big(y_{2}^{1:((j_{1}+1-j_2)\lor0)},y_{(J_1+1):J_2} \big),\ \cdots, 
\\&\quad 
\big(x_{i}^{1:((j_{i-1}+i-1-j_i)\lor0)},x_{(J_{i-1}+1):J_i}\big),
C_i',
\big(y_{i}^{1:((j_{i-1}+i-1-j_i)\lor0)}, y_{(J_{i-1}+1):J_i}\big)\Big)
\\&\in 
\wt{\ms P}_i \cap \ms F_{j_1,\dots,j_i}
\cap \ms Z_{C_1,\dots,C_i}
\Big\}.
\end{align*}
Since $\ms P_i$, $\ms F_{j_1,\dots,j_i}$, and $\ms Z_{C_1,\dots,C_i}$ are analytic sets, we can conclude that $\XY_{j_1,\dots,j_i,C_1,\dots,C_i}$ is an analytic subset of $(\X\times\Y)^{J_i}$.

Define the set
\begin{align*}
&\ms A_{j_1,\dots,j_i,C_1,\dots,C_i}\\:=&
\bigcup_{\bw\in\prod_{t=i+1}^\infty(\XPC_t\times\Y^t)}
\big\{
(x_1,y_1,\dots,x_{J_i},y_{J_i})\in (\X\times\Y)^{J_i}:
\\&
(x_1,C_1,y_1,x_{I_2:(I_2+1)},C_2,y_{I_2:(I_2+1)},\dots, x_{I_i:(I_i+i-1)},C_i,y_{I_i:(I_i+i-1)},\bw)\in \ms P_\infty\backslash\ms W_\fB\big\}
\\=&
\bigcup_{\bw\in\prod_{t=i+1}^\infty(\XPC_t\times\Y^i)}
\bigcup_{(C_1',\dots,C_i')\in\prod_{k=1}^i\PC(\Y^k)}
\\&
\bigcup_{\left(x_{2}^{1:((j_{1}+1-j_2)\lor0)},y_{2}^{1:((j_{1}+1-j_2)\lor0)},\dots,x_{i}^{1:((j_{i-1}+i-1-j_i)\lor0)},y_{i}^{1:((j_{i-1}+i-1-j_i)\lor0)}\right) \in \prod_{k=1}^i(\X\times\Y)^{1:((j_{k-1}+k-1-j_k)\lor0)}}
\\&
\Big\{
(x_1,y_1,\dots,x_{J_i},y_{J_i})\in (\X\times\Y)^{J_i}:
\\&
\Big(x_1,C_1',y_1,
\big(x_{2}^{1:((j_{1}+1-j_2)\lor0)},x_{(J_1+1):J_2}\big),
C_2',
\big(y_{2}^{1:((j_{1}+1-j_2)\lor0)},y_{(J_1+1):J_2} \big),\ \cdots, 
\\&
\big(x_{i}^{1:((j_{i-1}+i-1-j_i)\lor0)},x_{(J_{i-1}+1):J_i}\big),
C_i',
\big(y_{i}^{1:((j_{i-1}+i-1-j_i)\lor0)}, y_{(J_{i-1}+1):J_i}\big),
\bw\Big)
\\&\in 
(\ms P_\infty\backslash \ms W_\fB) \cap \wt{\ms F}_{j_1,\dots,j_i}
\cap \wt{\ms Z}_{C_1,\dots,C_i}
\Big\}.
\end{align*}
By Lemma \ref{lem:coanalytic_linear} and the analysis above, $\ms P_\infty\backslash \ms W_\fB$, $\wt{\ms F}_{j_1,\dots,j_i}$, and $\wt{\ms Z}_{C_1,\dots,C_i}$ are analytic subsets of $\ms P_{\infty}$. Moreover, $\prod_{k=1}^i(\X\times\Y)^{1:((j_{k-1}+k-1-j_k)\lor0)}$ is a Polish space. Thus, $\ms A_{j_1,\dots,j_i,C_1,\dots,C_i}$ is an analytic subsets of $(\X\times\Y)^{J_i}$. 

For any $\kappa\in\ORD\cup\{-1\}$, any $i\in\{2,3,\dots,s\}$, and any $C_1\in\PC(\Y^1),\dots,C_{i}\in\PC(\Y^{i})$, define the sets
\begin{align*}
\ms A_i:=
\bigcup_{\bw\in\prod_{t=i+1}^\infty(\XPC_t\times\Y^t)}
\{\bv\in\ms P_i: (\bv,\bw)\in\ms P_{\infty}\backslash\W_{\fB}\},
\end{align*} 
\begin{align*}
\ms A_i^{\kappa}:=
\{\bv\in\ms A_{i}:\val(\bv)>\kappa\},
\end{align*}
and
\begin{align*}
&\ms A_{j_1,\dots,j_i,C_1,\dots,C_i}^{\kappa}
\\:=&
\Big\{(x_1,y_1,\dots,x_{J_i},y_{J_i})\in\ms A_{j_1,\dots,j_i,C_1,\dots,C_i}:
\\& \quad
\val\big(x_1,C_1,y_1,x_{I_2:(I_2+1)},C_2,y_{I_2:(I_2+1)},\dots, x_{I_i:(I_i+i-1)},C_i,y_{I_i:(I_i+i-1)}\big)>\kappa
\Big\}.
\end{align*}
When $\kappa=-1$, we have $\ms A_{j_1,\dots,j_i,C_1,\dots,C_i}^{-1}=\ms A_{j_1,\dots,j_i,C_1,\dots,C_i}$ and $\ms A_{i}^{-1}=\ms A_i$. 
According to \citet[Corollary B.11]{universal_learning}, $\ms A_i^{\kappa}$ is an analytic subset of $\ms P_i$ for any $-1\le \kappa<\omega_1$. Then, since
\begin{align*}
&\ms A_{j_1,\dots,j_i,C_1,\dots,C_i}^\kappa
\\=&
\bigcup_{(C_1',\dots,C_i')\in\prod_{k=1}^i\PC(\Y^k)}
\\&
\bigcup_{\left(x_{2}^{1:((j_{1}+1-j_2)\lor0)},y_{2}^{1:((j_{1}+1-j_2)\lor0)},\dots,x_{i}^{1:((j_{i-1}+i-1-j_i)\lor0)},y_{i}^{1:((j_{i-1}+i-1-j_i)\lor0)}\right) \in \prod_{k=1}^i(\X\times\Y)^{1:((j_{k-1}+k-1-j_k)\lor0)}}
\\&
\Big\{
(x_1,y_1,\dots,x_{J_i},y_{J_i})\in (\X\times\Y)^{J_i}:
\\&
\Big(x_1,C_1',y_1,
\big(x_{2}^{1:((j_{1}+1-j_2)\lor0)},x_{(J_1+1):J_2}\big),
C_2',
\big(y_{2}^{1:((j_{1}+1-j_2)\lor0)},y_{(J_1+1):J_2}\big),\ \cdots,
\\&
\big(x_{i}^{1:((j_{i-1}+i-1-j_i)\lor0)},x_{(J_{i-1}+1):J_i}\big),
C_i',
\big(y_{i}^{1:((j_{i-1}+i-1-j_i)\lor0)}, y_{(J_{i-1}+1):J_i}\big)\Big)
\\&\in 
\ms A_i^\kappa \cap \ms F_{j_1,\dots,j_i}
\cap \ms Z_{C_1,\dots,C_i}
\Big\},
\end{align*}
we can conclude that $\ms  A_{j_1,\dots,j_i,C_1,\dots,C_i}^\kappa$ is analytic subset of $(\X\times\Y)^{J_i}$ for any $-1\le \kappa<\omega_1$. 

According to Lemma \ref{lem:PA_winning_DSL} and the definition of the game value \citep[Definition B.5]{universal_learning}, we have $\val(\emptyset)<\Omega$ under the assumption that $\cH$ does not have an infinite DSL tree.
Then, by \citet[Lemma B.7]{universal_learning}, we immediately have the following lemma.
\begin{lemma} \label{lem:DSL_empty_count}
$\val(\emptyset)<\omega_1$ when $\cH$ does not have an infinite DSL tree. 
\end{lemma}
Now, for any $m\in\N$ with $j_i\le m\le j_{i+1}$, any $(C_1,\dots,C_{i+1})\in\prod_{k=1}^{i+1}\PC(\Y^{k})$, and any $\by'=(y_1',\dots,y_{(m-j_i+1)\land (i+1)}')\in\Y^{(m-j_i+1)\land (i+1)}$, define the set
\begin{align*}
&\ms D_{j_1,\dots,j_i,m,C_1,\dots,C_i,C_{i+1},\by'}
\\:=&
\Big\{(x_1,y_1,\dots,x_{J_i},y_{J_i},x'_{1:((m-j_i+1)\land(i+1))})\in (\X\times\Y)^{J_i}\times\X^{(m-j_i+1)\land (i+1)}:
\\&
(x_1,y_1,\dots,x_{J_i},y_{J_i},(x'_k,y'_k)_{k=1}^{((m-j_i+1)\land(i+1))})
\in\XY_{j_1,\dots,j_i,m,C_1,\dots,C_i,C_{i+1}},
\\&
\val\big(x_1,C_1,y_1,\dots,
x_{I_i:(I_i+i-1)},C_i,y_{I_i:(I_i+i-1)},
\\&
\big(x_{(J_i+1-((j_i+i-m)\lor0)):J_i},x'_{1:((m-j_i+1)\land (i+1))}\big),C_{i+1},
\big(y_{(J_i+1-((j_i+i-m)\lor0)):J_i},y'_{1:((m-j_i+1)\land (i+1))}\big)
\big)
\\&<
\min\big\{\val(\emptyset),
\val\big(
x_1,C_1,y_1,\dots,
x_{I_i:(I_i+i-1)},C_i,y_{I_i:(I_i+i-1)}
\big)
\big\}
\Big\}.
\end{align*}
We prove the following result about the above set.
\begin{lemma}
For any $m\in\N$ with $j_i\le m\le j_{i+1}$, any $(C_1,\dots,C_{i+1})\in\prod_{k=1}^{i+1}\PC(\Y^{k})$, and any $\by'=(y_1',\dots,y_{(m-j_i+1)\land (i+1)}')\in\Y^{(m-j_i+1)\land (i+1)}$, $\ms D_{j_1,\dots,j_i,m,C_1,\dots,C_i,C_{i+1},\by'}$ is universally measurable.
\end{lemma}
\begin{proof}
We can write
\begin{align*}
&\ms D_{j_1,\dots,j_i,m,C_1,\dots,C_i,C_{i+1},\by'}
\\=&
\bigcup_{-1\le\kappa<\val(\emptyset)}
\Big\{
(x_1,y_1,\dots,x_{J_i},y_{J_i},x'_{1:((m-j_i+1)\land(i+1))})\in (\X\times\Y)^{J_i}\times\X^{(m-j_i+1)\land (i+1)}:
\\&
(x_1,y_1,\dots,x_{J_i},y_{J_i},(x_k',y_k')_{k=1}^{(m-j_i+1)\land(i+1)})
\in\XY_{j_1,\dots,j_i,m,C_1,\dots,C_i,C_{i+1}},
\\&
\val\big(x_1,C_1,y_1,\dots,
x_{I_i:(I_i+i-1)},C_i,y_{I_i:(I_i+i-1)},
\big(x_{(J_i+1-((j_i+i-m)\lor0)):J_i},x'_{1:((m-j_i+1)\land (i+1))}\big),
\\&
C_{i+1},
\big(y_{(J_i+1-((j_i+i-m)\lor0)):J_i},y'_{1:((m-j_i+1)\land (i+1))}\big)\big)\le\kappa,
\\&
\val\big(x_1,C_1,y_1,\dots,
x_{I_i:(I_i+i-1)},C_i,y_{I_i:(I_i+i-1)}\big)>\kappa
\Big\}
\\=&
\bigcup_{-1\le\kappa<\val(\emptyset)}
\Big\{
(x_1,y_1,\dots,x_{J_i},y_{J_i},x'_{1:((m-j_i+1)\land(i+1))})\in (\X\times\Y)^{J_i}\times\X^{(m-j_i+1)\land (i+1)}:
\\&
\big(x_1,y_1,\dots,x_{J_i},y_{J_i},(x'_k,y'_k)_{k=1}^{(m-j_i+1)\land(i+1)}\big)\in
\XY_{j_1,\dots,j_i,m,C_1,\dots,C_i,C_{i+1}}\backslash\ms A_{j_1,\dots,j_i,m,C_1,\dots,C_i,C_{i+1}}^\kappa,
\\&
\big(x_1,y_1,\dots,x_{J_i},y_{J_i}\big)\in\ms A_{j_1,\dots,j_i,C_1,\dots,C_i}^\kappa
\Big\}
\\=&
\bigcup_{-1\le\kappa<\val(\emptyset)}
\Big(\Big\{
(x_1,y_1,\dots,x_{J_i},y_{J_i},x'_{1:((m-j_i+1)\land(i+1))})\in (\X\times\Y)^{J_i}\times\X^{(m-j_i+1)\land (i+1)}:
\\&
\big(x_1,y_1,\dots,x_{J_i},y_{J_i},(x'_k,y'_k)_{k=1}^{(m-j_i+1)\land(i+1)}\big)\in
\XY_{j_1,\dots,j_i,m,C_1,\dots,C_i,C_{i+1}}\backslash\ms A_{j_1,\dots,j_i,m,C_1,\dots,C_i,C_{i+1}}^\kappa
\Big\}
\\& \bigcap
\left(\ms A_{j_1,\dots,j_i,C_1,\dots,C_i}^\kappa\times\X^{(m-j_i+1)\land (i+1)}\right)\Big).
\end{align*}
By Lemma \ref{lem:DSL_empty_count} and the previous results, for any $-1\le \kappa<\val(\emptyset)<\omega_1$, we have that $\ms A_{j_1,\dots,j_i,C_1,\dots,C_i}^\kappa\times\X^{(m-j_i+1)\land (i+1)}$ is an analytic subset of $(\X\times\Y)^{J_i}\times\X^{(m-j_i+1)\land (i+1)}$. 
Moreover, for any $-1\le \kappa<\val(\emptyset)<\omega_1$, we have
\begin{align*}
&\Big\{
(x_1,y_1,\dots,x_{J_i},y_{J_i},x'_{1:((m-j_i+1)\land(i+1))})\in (\X\times\Y)^{J_i}\times\X^{(m-j_i+1)\land (i+1)}:
\\&
\big(x_1,y_1,\dots,x_{J_i},y_{J_i},(x'_k,y'_k)_{k=1}^{(m-j_i+1)\land(i+1)}\big)\in
\XY_{j_1,\dots,j_i,m,C_1,\dots,C_i,C_{i+1}}\backslash\ms A_{j_1,\dots,j_i,m,C_1,\dots,C_i,C_{i+1}}^\kappa
\Big\}
\\=&
\Big\{
(x_1,y_1,\dots,x_{J_i},y_{J_i},x'_{1:((m-j_i+1)\land(i+1))})\in (\X\times\Y)^{J_i}\times\X^{(m-j_i+1)\land (i+1)}:
\\&
\big(x_1,y_1,\dots,x_{J_i},y_{J_i},(x'_k,y'_k)_{k=1}^{(m-j_i+1)\land(i+1)}\big)\in
\XY_{j_1,\dots,j_i,m,C_1,\dots,C_i,C_{i+1}}\Big\}
\\& \Big\backslash
\Big\{
(x_1,y_1,\dots,x_{J_i},y_{J_i},x'_{1:((m-j_i+1)\land(i+1))})\in (\X\times\Y)^{J_i}\times\X^{(m-j_i+1)\land (i+1)}:
\\&
\quad\big(x_1,y_1,\dots,x_{J_i},y_{J_i},(x'_k,y'_k)_{k=1}^{(m-j_i+1)\land(i+1)}\big)\in
\ms A_{j_1,\dots,j_i,m,C_1,\dots,C_i,C_{i+1}}^\kappa
\Big\}
\end{align*}
with 
\begin{align*}
&\Big\{
(x_1,y_1,\dots,x_{J_i},y_{J_i},x'_{1:((m-j_i+1)\land(i+1))})\in (\X\times\Y)^{J_i}\times\X^{(m-j_i+1)\land (i+1)}:
\\&
\big(x_1,y_1,\dots,x_{J_i},y_{J_i},(x'_k,y'_k)_{k=1}^{(m-j_i+1)\land(i+1)}\big)\in
\XY_{j_1,\dots,j_i,m,C_1,\dots,C_i,C_{i+1}}\Big\}
\\=&
\bigcup_{\by''\in\Y^{(m-j_i+1)\land(i+1)}}
\Big\{
(x_1,y_1,\dots,x_{J_i},y_{J_i},x'_{1:((m-j_i+1)\land(i+1))})\in (\X\times\Y)^{J_i}\times\X^{(m-j_i+1)\land (i+1)}:
\\&
\big(x_1,y_1,\dots,x_{J_i},y_{J_i},(x'_k,y''_k)_{k=1}^{(m-j_i+1)\land(i+1)}\big)\in
\XY_{j_1,\dots,j_i,m,C_1,\dots,C_i,C_{i+1}}
\\&\quad\quad\quad\quad\quad\quad\quad\quad\quad\quad\quad\quad\quad\quad\quad\quad\quad\quad\quad \bigcap \Big((\X\times\Y)^{J_i}\times\prod_{k=1}^{(m-j_i+1)\land i}(\X\times\{y'_k\})\Big)\Big\}
\end{align*}
and 
\begin{align*}
&\Big\{
(x_1,y_1,\dots,x_{J_i},y_{J_i},x'_{1:((m-j_i+1)\land(i+1))})\in (\X\times\Y)^{J_i}\times\X^{(m-j_i+1)\land (i+1)}:
\\&
\big(x_1,y_1,\dots,x_{J_i},y_{J_i},(x'_k,y'_k)_{k=1}^{(m-j_i+1)\land(i+1)}\big)\in
\ms A_{j_1,\dots,j_i,m,C_1,\dots,C_i,C_{i+1}}^\kappa\Big\}
\\=&
\bigcup_{\by''\in\Y^{(m-j_i+1)\land(i+1)}}
\Big\{
(x_1,y_1,\dots,x_{J_i},y_{J_i},x'_{1:((m-j_i+1)\land(i+1))})\in (\X\times\Y)^{J_i}\times\X^{(m-j_i+1)\land (i+1)}:
\\&
\big(x_1,y_1,\dots,x_{J_i},y_{J_i},(x'_k,y''_k)_{k=1}^{(m-j_i+1)\land(i+1)}\big)\in
\ms A_{j_1,\dots,j_i,m,C_1,\dots,C_i,C_{i+1}}^\kappa
\\&\quad\quad\quad\quad\quad\quad\quad\quad\quad\quad\quad\quad\quad\quad\quad\quad\quad\quad\quad \bigcap \Big((\X\times\Y)^{J_i}\times\prod_{k=1}^{(m-j_i+1)\land i}(\X\times\{y'_k\})\Big)\Big\}.
\end{align*}
Since we have proved that $\XY_{j_1,\dots,j_i,m,C_1,\dots,C_i,C_{i+1}}$ and $\ms A_{j_1,\dots,j_i,m,C_1,\dots,C_i,C_{i+1}}^\kappa$ are analytic subsets of $(\X\times\Y)^{J_i+((m-j_i+1)\land (i+1))}$, we can conclude that the set
\begin{align*}
&\Big\{
(x_1,y_1,\dots,x_{J_i},y_{J_i},x'_{1:((m-j_i+1)\land(i+1))})\in (\X\times\Y)^{J_i}\times\X^{(m-j_i+1)\land (i+1)}:
\\&
\big(x_1,y_1,\dots,x_{J_i},y_{J_i},(x'_k,y'_k)_{k=1}^{(m-j_i+1)\land(i+1)}\big)\in
\XY_{j_1,\dots,j_i,m,C_1,\dots,C_i,C_{i+1}}\Big\}
\end{align*}
and the set
\begin{align*}
&\Big\{
(x_1,y_1,\dots,x_{J_i},y_{J_i},x'_{1:((m-j_i+1)\land(i+1))})\in (\X\times\Y)^{J_i}\times\X^{(m-j_i+1)\land (i+1)}:
\\&
\big(x_1,y_1,\dots,x_{J_i},y_{J_i},(x'_k,y'_k)_{k=1}^{(m-j_i+1)\land(i+1)}\big)\in
\ms A_{j_1,\dots,j_i,m,C_1,\dots,C_i,C_{i+1}}^\kappa\Big\}
\end{align*}
are both analytic subsets of $(\X\times\Y)^{J_i}\times\X^{(m-j_i+1)\land (i+1)}$. Therefore, the set
\begin{align*}
&\Big\{
(x_1,y_1,\dots,x_{J_i},y_{J_i},x'_{1:((m-j_i+1)\land(i+1))})\in (\X\times\Y)^{J_i}\times\X^{(m-j_i+1)\land (i+1)}:
\\&
\big(x_1,y_1,\dots,x_{J_i},y_{J_i},(x'_k,y'_k)_{k=1}^{(m-j_i+1)\land(i+1)}\big)\in
\XY_{j_1,\dots,j_i,m,C_1,\dots,C_i,C_{i+1}}\backslash\ms A_{j_1,\dots,j_i,m,C_1,\dots,C_i,C_{i+1}}^\kappa
\Big\}
\end{align*}
is universally measurable. 
It follows from the fact that $\val(\emptyset)<\omega_1$ that
$\ms D_{j_1,\dots,j_i,m,C_1,\dots,C_i,C_{i+1},\by'}$ is universally measurable.
\end{proof}

Next, we define the set
\begin{align*}
&\ms D_{j_1,\dots,j_i,m,C_1,\dots,C_i,C_{i+1}}
\\:=&
\Big\{(x_1,y_1,\dots,x_{J_i},y_{J_i},(x'_k,y'_k)_{k=1}^{((m-j_i+1)\land(i+1))})
\in\XY_{j_1,\dots,j_i,m,C_1,\dots,C_i,C_{i+1}}:
\\&
\val\big(x_1,C_1,y_1,\dots,
x_{I_i:(I_i+i-1)},C_i,y_{I_i:(I_i+i-1)},
\\&
\big(x_{((J_i+1-((j_i+i-m)\lor0)):J_i)},x'_{1:((m-j_i+1)\land (i+1))}\big),C_{i+1},
\big(y_{((J_i+1-((j_i+i-m)\lor0)):J_i)},y'_{1:((m-j_i+1)\land (i+1))}\big)
\big)
\\&<
\min\big\{\val(\emptyset),
\val\big(
x_1,C_1,y_1,\dots,
x_{I_i:(I_i+i-1)},C_i,y_{I_i:(I_i+i-1)}
\big)
\big\}
\Big\}
\end{align*}
and prove the following lemma.
\begin{lemma}
For any $m\in\N$ with $j_i\le m\le j_{i+1}$ and any $(C_1,\dots,C_{i+1})\in\prod_{k=1}^{i+1}\PC(\Y^{k})$, $\ms D_{j_1,\dots,j_i,m,C_1,\dots,C_i,C_{i+1}}$ is universally measurable.
\end{lemma}
\begin{proof}
We have
\begin{align*}
&\ms D_{j_1,\dots,j_i,m,C_1,\dots,C_i,C_{i+1}}
\\=&
\bigcup_{-1\le\kappa<\val(\emptyset)}
\Big\{
(x_1,y_1,\dots,x_{J_i},y_{J_i},(x'_k,y'_k)_{k=1}^{((m-j_i+1)\land(i+1))})
\in\XY_{j_1,\dots,j_i,m,C_1,\dots,C_i,C_{i+1}}:
\\&
\val\big(x_1,C_1,y_1,\dots,
x_{I_i:(I_i+i-1)},C_i,y_{I_i:(I_i+i-1)},
\big(x_{((J_i+1-((j_i+i-m)\lor0)):J_i)},x'_{1:((m-j_i+1)\land (i+1))}\big),
\\&
C_{i+1},
\big(y_{((J_i+1-((j_i+i-m)\lor0)):J_i)},y'_{1:((m-j_i+1)\land (i+1))}\big)\big)\le\kappa,
\\&
\val\big(x_1,C_1,y_1,\dots,
x_{I_i:(I_i+i-1)},C_i,y_{I_i:(I_i+i-1)}\big)>\kappa
\Big\}
\\=&
\bigcup_{-1\le\kappa<\val(\emptyset)}
\Big\{
\big(x_1,y_1,\dots,x_{J_i},y_{J_i},(x'_k,y'_k)_{k=1}^{(m-j_i+1)\land(i+1)}\big)\in
\\& \quad\quad\ \ 
\XY_{j_1,\dots,j_i,m,C_1,\dots,C_i,C_{i+1}}\backslash\ms A_{j_1,\dots,j_i,m,C_1,\dots,C_i,C_{i+1}}^\kappa:
\big(x_1,y_1,\dots,x_{J_i},y_{J_i}\big)\in\ms A_{j_1,\dots,j_i,C_1,\dots,C_i}^\kappa
\Big\}
\\=&
\bigcup_{-1\le\kappa<\val(\emptyset)}
\Big(\big(
\XY_{j_1,\dots,j_i,m,C_1,\dots,C_i,C_{i+1}}\backslash\ms A_{j_1,\dots,j_i,m,C_1,\dots,C_i,C_{i+1}}^\kappa\big) 
\\&\quad\quad\quad\quad\quad\quad\bigcap
\big(\ms A_{j_1,\dots,j_i,C_1,\dots,C_i}^\kappa\times\X^{(m-j_i+1)\land (i+1)}\big)\Big).
\end{align*}
Since we have proved that $\XY_{j_1,\dots,j_i,m,C_1,\dots,C_i,C_{i+1}}$, $\ms A_{j_1,\dots,j_i,m,C_1,\dots,C_i,C_{i+1}}^\kappa$, and $\ms A_{j_1,\dots,j_i,C_1,\dots,C_i}^\kappa$ are analytic and $\val(\emptyset)<\omega_1$ (Lemma \ref{lem:DSL_empty_count}), we can conclude that $\ms D_{j_1,\dots,j_i,m,C_1,\dots,C_i,C_{i+1}}$ is universally measurable.
\end{proof}
Then, the following set
\begin{align*}
\ms D_{j_1,\dots,j_i,m,C_1,\dots,C_i}:=&
\bigcup_{C_{i+1}\in\PC(\Y^{i+1})}
\ms D_{j_1,\dots,j_i,m,C_1,\dots,C_i,C_{i+1}}
\end{align*}
is also universally measurable because $\PC(\Y^{i+1})$ is countable. 

Before proceeding to the next step, we will need the following lemmas regarding universal measurability. 
For any measure spaces $(A,\Ascr)$ and $(B,\Bscr)$, let $\Ascr^*$ and $\Bscr^*$ denote the universal completion of $\Ascr$ and $\Bscr$, respetively. 
Let $\Ascr\times \Bscr$ denote the product $\sigma$-algebra of $\Ascr$ and $\Bscr$ on $(A\times B)$. 
Note that when $A$ and $B$ are Polish spaces, $\Ascr\times\Bscr$ is also the Borel $\sigma$-algebra of $A\times B$. 
A function $f:A\rightarrow B$ is called $\Ascr/\Bscr$-measurable if $f^{-1}(E)\in\Ascr$ for any $E\in\Bscr$. 
We prove the following lemmas. 
\begin{lemma} \label{lem:universal_measurable_equiv}
For any two measurable space $(A,\Ascr)$ and $(B,\Bscr)$, any function $f: A\rightarrow B$ is $\Ascr^*/\Bscr^*$-measurable if and only if $f$ is $\Ascr^*/\Bscr$-measurable.
\end{lemma}
\begin{proof}
Suppose that $f$ is $\Ascr^*/\Bscr$-measurable.
For any probability measure $\mu_A$ on $(A,\Ascr)$, let $(A,\Ascr_{\mu_A}^*,\mu_A^*)$ denote the completion of $(A,\Ascr,\mu_A)$. Then, $(A,\Ascr^*,\mu_A^*)$ is also a probability space because $\Ascr^*\subseteq\Ascr^*_{\mu_A}$. 
Since $f$ is $\Ascr^*/\Bscr$-measurable, we can define $\mu_B:\Bscr\rightarrow [0,1],\ E\mapsto\mu_A^*(f^{-1}(E))$ which is the pushforward of $\mu_A^*$ by $f$. 

For any $S\in\Bscr^*$, there exist $U,V\in\Bscr$ such that $U\subseteq S\subseteq V$ and $\mu_B(V\backslash U)=0$. 
Then, we have $f^{-1}(U)\subseteq f^{-1}(S)\subseteq f^{-1}(V)$ and $f^{-1}(U),f^{-1}(V)\in\Ascr^*$ which implies that there exist $U_l,U_u,V_l,V_u\in \Ascr$ such that $U_l\subseteq f^{-1}(U)\subseteq U_u$, $V_l\subseteq f^{-1}(V)\subseteq V_u$ and $\mu_A(U_u\backslash U_l)=\mu_A(V_u\backslash V_l)=0$. 
Moreover, it follows from the definition of $\mu_B$ that
$$
\mu_A^*(f^{-1}(V)\backslash f^{-1}(U))=
\mu_A^*(f^{-1}(V\backslash U))=
\mu_B(V\backslash U)=0.
$$
Since $\mu_A^*$ is the completion of $\mu_A$, there exists $K\in\Ascr$ with $K\supseteq f^{-1}(V)\backslash f^{-1}(U)$ and $\mu_A(K)=0$. 
Therefore, we have $U_l\subseteq f^{-1}(U)\subseteq f^{-1}(S)\subseteq f^{-1}(V) \subseteq V_u$ and 
\begin{align*}
V_u\backslash U_l\subseteq 
(V_u\backslash V_l)\cup
(U_u\backslash U_l)\cup
K.
\end{align*}
Since $\mu_A(V_u\backslash V_l)=\mu_A(U_u\backslash U_l)=\mu_A(K)=0$, we have $\mu_A(V_u\backslash U_l)=0$. 
Thus, by the arbitrariness of $\mu$, we can conclude that $f^{-1}(S)\in\Ascr^*$, which implies that $f$ is $\Ascr^*/\Bscr^*$. 

The other direction is trivial. 
\end{proof}

\begin{lemma} \label{lem:meas_coord}
Consider any $n\in\N$ Polish space $A_1,\dots,A_n$ with their Borel $\sigma$-algebras denoted as $\Ascr_1,\dots,\Ascr_n$, respectively. For any $m\in[n]$, any sequence $1\le i_1<i_2<\dots<i_m\le n$, and any set $E\in (\prod_{k=1}^m\Ascr_{i_k})^*$, then we have
\begin{align*}
\wt E:=\left\{(x_1,\dots,x_n)\in\prod_{j=1}^n A_j:(x_{i_1},\dots,x_{i_m})\in E\right\}\in\left(\prod_{j=1}^n\Ascr_j\right)^*.
\end{align*}
\end{lemma}
\begin{proof}
Consider the following collections of subsets of $\prod_{j=1}^n A_j$
\begin{align*}
\mathfrak{G}:=\left\{\left\{(x_1,\dots,x_n)\in\prod_{j=1}^n A_j:x_{i_1}\in B_1,\dots,x_{i_m}\in B_m\right\}:
B_1\in\Ascr_{i_1},\dots,B_m\in\Ascr_{i_m}\right\}.
\end{align*}
It is easy to see that $\mathfrak{G}$ is a $\pi$-system on $\prod_{j=1}^n A_j$.  
Define $\G:=\sigma\left(\mathfrak{G}\right)$ to be the $\sigma$-algebra generated by $\mathfrak{G}$ and define following the collection of subsets of $\prod_{j=1}^n A_{j}$
\begin{align*}
\mathcal{C}:=\left\{\left\{(x_1,\dots,x_n)\in\prod_{j=1}^n A_j:(x_{i_1},\dots,x_{i_m})\in S\right\}:S\in\prod_{j=1}^m\Ascr_{i_k}\right\}.
\end{align*}
It is obvious that $\mathcal{C}$ is a $\sigma$-algebra on $\prod_{j=1}^n A_{j}$. 
Since $\mathfrak{G}\subseteq \mathcal{C}$, by Dynkin's $\pi$-$\lambda$ theorem, we have $\G=\sigma(\mathfrak{G})\subseteq \mathcal{C}$. 

Next, define the following collection of subsets of $\prod_{k=1}^m A_{i_k}$
\begin{align*}
\mathfrak{F}:=\left\{G|_{\prod_{k=1}^m A_{i_k}}:G\in\G\right\}
\end{align*}
where for any $G\subseteq\prod_{j=1}^nA_j$, we define
\begin{align} \label{eq:proj_intersect}
G|_{\prod_{k=1}^m A_{i_k}}:=\left\{(x_{i_1},\dots,x_{i_m}):\exists(x'_1,\dots,x'_{n})\in G \textup{ s.t. }x'_{i_k}=x_{i_k}, \ \forall\ j\in[n]\right\}.
\end{align} 
We now show that $\mathfrak{F}$ is a $\sigma$-algebra on $\prod_{k=1}^m A_{i_k}$. 
\begin{itemize}
\item Since $\prod_{j=1}^nA_j\in \mathfrak{G}$, we have $\prod_{k=1}^m A_{i_k}\in \mathfrak{F}$.
\item For any $G_1',G_2'\in \mathfrak{F}$ with $G_1'\subseteq G_2'$, there exists $G_l\in\G$ such that $G_l'=G_l|_{\prod_{k=1}^m A_{i_k}}$ for $l=1,2$. 
By \eqref{eq:proj_intersect} and the facts that $G_1'\subseteq G_2'$ and $\G\subseteq \mathcal{C}$, we must have $G_1\subseteq G_2$. Since $\G$ is a $\sigma$-algebra, we have $G_2\backslash G_1\in\G$. 
By $\G\subseteq \mathcal{C}$ again, we have $(G_2\backslash G_1)|_{\prod_{k=1}^m A_{i_k}}=G_2'\backslash G_1'$, which implies that $G_2\backslash G_1\in \mathfrak{F}$. 
\item For any $G_1',G_2',\dots\in\mathfrak{F}$, 
there exists $G_1,G_2,\dots\in \G$ such that $G_l'=G_l|_{\prod_{k=1}^m A_{i_k}}$ for all $l\in\N$. Then, we have $\cup_{l=1}^\infty G_l\in\G$. Since $\G\subseteq\mathcal{C}$, we have $\cup_{l=1}^\infty G_l'=(\cup_{l=1}^\infty G_l)|_{\prod_{k=1}^m A_{i_k}}\in\mathfrak{F}$. 
\end{itemize} 
Thus, $\mathfrak{F}$ is a $\sigma$-algebra on $\prod_{k=1}^m A_{i_k}$.

By the definition of product $\sigma$-algebras, we have $\prod_{k=1}^m \Ascr_{i_k}=\sigma(\mathfrak{C})$ where 
\begin{align*}
\mathfrak{C}:=\left\{\prod_{k=1}^mB_k:B_1\in\Ascr_{i_1},\dots,B_m\in\Ascr_{i_m}\right\}
\end{align*}
is a $\pi$-system on $\prod_{k=1}^m A_{i_k}$. 
Obviously, $\mathfrak{C}\subseteq\mathfrak{F}$. Then, by Dynkin's $\pi$-$\lambda$ theorem, we have $\prod_{k=1}^m \Ascr_{i_k}=\sigma(\mathfrak{C})\subseteq\mathfrak{F}$. 
By the definition of $\mathfrak{F}$ and $\mathcal{C}$ as well as the fact that $\mathcal{G}\subseteq\mathcal{C}$, we have
$\mathcal{C}\subseteq\G$. 
It follows that $\mathcal{C}=\G$. 
Since $\mathfrak{G}$ is a subset of the collection of all rectangles on $\prod_{j=1}^nA_j$, we have $\mathcal{G}\subseteq \prod_{j=1}^n\Ascr_j$. Thus, $\mathcal{C}\subseteq \prod_{j=1}^n\Ascr_j$.

For any probability measure $\mu$ on $\prod_{j=1}^n \Ascr_j$, consider its projection $\mu_{i_1,\dots,i_m}:=\mu|_{\prod_{k=1}^m A_{i_k}}$ on $\prod_{k=1}^m A_{i_k}$ defined by 
\begin{align} \label{eq:mu_proj_def}
\mu_{i_1,\dots,i_m}(S):=\mu\left(\left\{(x_1,\dots,x_n)\in\prod_{j=1}^n A_j:(x_{i_1},\dots,x_{i_m})\in S\right\}\right)
\quad \forall S\in \prod_{k=1}^m\Ascr_{i_k}.
\end{align}
Since we have proved $\mathcal{C}\subseteq \prod_{j=1}^n\Ascr_j$, the above definition makes sense and 
$$
\left(\prod_{k=1}^mA_{i_k}, \prod_{k=1}^m\Ascr_{i_k},\mu_{i_1,\dots,i_m}\right)
$$ 
is indeed a probability space. 

Since $E\in (\prod_{k=1}^m\Ascr_{i_k})^*$ and $\mu_{i_1,\dots,i_m}$ is a probability measure on $(\prod_{k=1}^m A_{i_k}, \prod_{k=1}^m \Ascr_{i_k})$, there exist sets $U,V\in\prod_{k=1}^m \Ascr_{i_k}$ such that $U\subseteq E\subseteq V$ and $\mu_{i_1,\dots,i_m}(V\backslash U)=0$. 
We define 
\begin{align*}
\wt U:=\left\{(x_1,\dots,x_n)\in\prod_{j=1}^n A_j:(x_{i_1},\dots,x_{i_m})\in U\right\}
\end{align*}
and
\begin{align*}
\wt V:=\left\{(x_1,\dots,x_n)\in\prod_{j=1}^n A_j:(x_{i_1},\dots,x_{i_m})\in V\right\}.
\end{align*}
By definition, we have $\wt U,\wt E\in\mathcal{C}\subseteq\prod_{j=1}^n \Ascr_{j}$ with $\wt U\subseteq\wt E\subseteq\wt V$. 
Moreover, by \eqref{eq:mu_proj_def}, we have
\begin{align*}
\mu(\wt V\backslash\wt U)
=\mu_{i_1,\dots,i_m}(V\backslash U)=0.
\end{align*}
Thus, we can conclude that $\wt E\in \left(\prod_{j=1}^n\Ascr_j\right)^*$. 
\end{proof}

Now, we define the sets
\begin{align*}
&\ms D^{+}_{t,(j_1,\dots,j_i),m,(C_1,\dots,C_i),C_{i+1}}
\\:=&\Big\{(x_1,y_1,\dots,x_t,y_t)\in(\X\times\Y)^t:
\big(x_{j_1},y_{j_1},x_{((j_1+1)\lor j_2):(j_2+1)},y_{((j_1+1)\lor j_2):(j_2+1)},\dots,
\\&
x_{((j_{i-1}+i-1)\lor j_i):(j_i+i-1)},y_{((j_{i-1}+i-1)\lor j_i):(j_i+i-1)},
x_{((j_{i}+i)\lor m):(m+i)},y_{((j_{i}+i)\lor m):(m+i)}
\big)\in
\\&
\ms D_{j_1,\dots,j_i,m,C_1,\dots,C_i,C_{i+1}}\Big\},
\end{align*}
\begin{align*}
&\ms D^{-}_{t,(j_1,\dots,j_i),m,(C_1,\dots,C_i)}
\\:=&\Big\{(x_1,y_1,\dots,x_t,y_t)\in(\X\times\Y)^t:
\big(x_{j_1},y_{j_1},x_{((j_1+1)\lor j_2):(j_2+1)},y_{((j_1+1)\lor j_2):(j_2+1)},\dots,
\\&
x_{((j_{i-1}+i-1)\lor j_i):(j_i+i-1)},y_{((j_{i-1}+i-1)\lor j_i):(j_i+i-1)},
x_{((j_{i}+i)\lor m):(m+i)},y_{((j_{i}+i)\lor m):(m+i)}
\big)\in
\\&
\left(\XY_{j_1,\dots,j_i,C_1,\dots,C_i}\times(\X\times\Y)^{(m-j_i+1)\land(i+1)}\right)
\backslash\ms D_{j_1,\dots,j_i,m,C_1,\dots,C_i}\Big\},
\end{align*}
and 
\begin{align*}
&\ms D^{\vee}_{t,j_1,\dots,j_i,C_1,\dots,C_i,\by'}
\\:=&\bigcup_{C_{i+1}\in\PC(\Y^{i+1})}\Big\{(x_1,y_1,\dots,x_{t},y_{t},\bx')\in(\X\times\Y)^{t}\times\X^{i+1}:
\\&
\big(x_{j_1},y_{j_1},x_{((j_1+1)\lor j_2):(j_2+1)},
y_{((j_1+1)\lor j_2):(j_2+1)},\dots,
\\&
x_{((j_{i-1}+i-1)\lor j_i):(j_i+i-1)},y_{((j_{i-1}+i-1)\lor j_i):(j_i+i-1)},
\bx'\big)\in
\ms D_{j_1,\dots,j_i,t+1,C_1,\dots,C_i,C_{i+1},\by'}\Big\}.
\end{align*}
Since we have proved that the sets $\ms D_{j_1,\dots,j_i,m,C_1,\dots,C_i,C_{i+1}}$, $\XY_{j_1,\dots,j_i,C_1,\dots,C_i}$, $\ms D_{j_1,\dots,j_i,m,C_1,\dots,C_i}$, and $\ms D_{j_1,\dots,j_i,t+1,C_1,\dots,C_i,C_{i+1},\by'}$ are universally measurable,
by Lemma \ref{lem:meas_coord}, we know that 
$$
\ms D^{+}_{t,(j_1,\dots,j_i),m,(C_1,\dots,C_i),C_{i+1}},\ 
\ms D^{-}_{t,(j_1,\dots,j_i),m,(C_1,\dots,C_i)},\ \textup{and }
\ms D^{\vee}_{t,j_1,\dots,j_i,C_1,\dots,C_i,\by'}
$$
are also universally measurable.

Define the following sets
\begin{align*}
\ms V_{t,j_1,\dots,j_s,C_1,\dots,C_s}:=&
\big((\ms D^-_{1,\emptyset,1,\emptyset})^{j_1-1}\times \ms D^+_{1,\emptyset,1,\emptyset,C_{1}}\times (\X\times\Y)^{t-j_1}\big)
\\&
\bigcap\big[\cap_{i=1}^{s-1}
\big(\big(\cap_{m=j_i}^{j_{i+1}-1}\ms D^{-}_{t,(j_1,\dots,j_i),m,(C_1,\dots,C_i)}\big)\cap
\ms D^+_{t,(j_1,\dots,j_i),j_{i+1},(C_1,\dots,C_i),C_{i+1}}\big)\big]
\\&
\bigcap\big(\cap_{m=j_s}^{t-s}\ms D^{-}_{t,(j_1,\dots,j_s),m,(C_1,\dots,C_s)}\big)
\end{align*}
and
\begin{align*}
\ms V_{t,j_1,\dots,j_s}:=&\bigcup_{(C_1,\dots,C_s)\in\prod_{k=1}^s\PC(\Y^k)}
\ms V_{t,j_1,\dots,j_s,C_1,\dots,C_s}.
\end{align*}
By the results above, we can conclude that $\ms V_{t,j_1,\dots,j_s,C_1,\dots,C_s}$ is universally measurable. Thus, $\ms V_{t,j_1,\dots,j_s}$ is also universally measurable.

Finally, we can complete the proof.
\begin{proof}[of Proposition \ref{prop:meas_pattern}]
First note that the function $T_0:\emptyset\rightarrow\{1\}$ is obviously universally measurable. 

For any $t\ge 1$ and any $1\le s\le t+1$, we have
\begin{align*}
T_t^{-1}(s)=\cup_{1\le j_1\le\dots\le j_{s-1}\le t-s+2}\ms V_{t,j_1,\dots,j_{s-1}}
\end{align*}
which is a universally measurable set according to the results proved above. 
Thus, $T_t$ is universally measurable. 

For any $t\ge1$, any $1\le s\le t$ , any $\by^*=(y^*_1,\dots,y^*_s)\in \Y^s$, define $\cS_{\by^*}=\{S\subseteq\Y^s:\by^*\in S\}$. Then, we have
\begin{align*}
\wh \bY_{t-1}^{-1}(\cS_{\by^*})=&
\bigcup_{1\le j_1\le\dots\le j_{s-1}\le t-s+1}
\bigcup_{(C_1,\dots,C_{s-1})\in\prod_{k=1}^{s-1}\PC(\Y^k)}
\\&
\left[\big(\ms V_{t-1,j_1,\dots,j_{s-1},C_1,\dots,C_{s-1}}\times\X^s\big)
\cap
\ms D^{\vee}_{t-1,j_1,\dots,j_{s-1},C_1,\dots,C_{s-1},\by^*}\right]
\end{align*}
which is a universally measurable set according to the results proved above. 
Then, by Lemma \ref{lem:universal_measurable_equiv}, $\wh \bY_{t-1}$ is universally measurable. 
\end{proof}

\subsection{Uniform rate implies universal rate}
Now, we apply the pattern avoidance functions defined in the previous section into a template for building learning algorithms in the probabilistic setting. Any learning algorithm with some guaranteed uniform rate for finite DS dimensional hypothesis classes can be plugged into this template to construct a learning algorithm that achieves the same universal rate for classes without an infinite DSL tree. 

For any $k\ge1$, any $n\ge k$, any function $g:\X^k\rightarrow2^{\Y^k}$, and any sequence $S=(x_1,\dots,x_n)\in\X^n$, define the concept set
\begin{align*}
\cH(S,g):=\{h\in\cH|_{S}:(h(i_1),\dots,h(i_k))\notin g(x_{i_1},\dots,x_{i_k})\textup{ for all distinct } 1\le i_1,\dots,i_k\le n\}.
\end{align*}
For any $t\ge0$, $n\ge \tau_{t}$, and any sequence $S=(x_1,\dots,x_n)\in\X^n$, define the concept set
\begin{align} \label{eq:H_per}
\cH(S,\wh\by_t):=\{h\in\cH|_{S}:(h(i_1),\dots,h({i_{\tau_{t}}}))\notin\wh\by_t(x_{i_1},\dots,x_{i_{\tau_{t}}})\textup{ for all distinct } 1\le i_1,\dots,{i_{\tau_{t}}}\le n\}.
\end{align}

We have the following lemma.
\begin{lemma} \label{lem:DS_dim_subclass}
For any $t\ge0$ and any sequence $(x_1,y_1,\dots,x_t,y_t)\in(\X\times\Y)^t$ (where we say $(x_1,y_1,\dots,x_t,y_t)=\emptyset$ if $t=0$) that is consistent with $\cH$, any $n\ge \tau_t$, and any $S:=(x'_1,\dots,x'_n)\in\X^n$, we have $\dim(\cH(S,\wh\by_{t}))<\tau_t$, where $\dim(\cH(S,\wh\by_{t}))$ denotes the DS dimension of $\cH(S,\wh\by_{t})$.
\end{lemma}
\begin{proof}
Define $k:=\tau_{t}$.
If $\dim(\cH(S,\wh\by_{t}))\ge k$, then there exists a sequence $(i_1,\dots,i_k)$ and pseudo-cube $C\in\PC(\cH(S,\wh\by_t)|_{(i_1,\dots,i_k)})$. Define $\wb \bx_{k}=(x_{i_1},\dots,x_{i_k})$. Then, by the definition of $\wh\by_t$, for any $\by'\in C$, we have that
\begin{align}
\notag
&\val(\wb\bx_1,\wb C_1,\wb\by_1,\dots,\wb\bx_{k-1},\wb C_{k-1},\wb\by_{k-1},\wb\bx_k,C,\by')
\\ \ge& \label{eq:val_assump}
\min\{\val(\wb\bx_1,\wb C_1,\wb\by_1,\dots,\wb\bx_{k-1},\wb C_{k-1},\wb\by_{k-1}),\val(\emptyset)\}.
\end{align}
Since $\cH$ does not have an infinite DSL tree; i.e., $P_A$ does not have a winning strategy, we have that $\val(\emptyset)<\Omega$ and further by \citet[Lemma B.7]{universal_learning}, $\val(\emptyset)<\omega_1$. 
Here, we claim that $\val(\emptyset)\ge0$. Consider the sequence $\bw=(\bx_1,C_1,\by_1,\dots)\in\ms P_\infty$ constructed as follows. First, fix a hypothesis $h\in\cH$. 
For each $s\ge1$, pick arbitrary $(\bx_s,C_s)\in\XPC_s$, set $\by_s=h(\bx_s)$. Then, it is obvious that $\bw\notin\W_{\fB}$. Thus, we have $\val(\emptyset)\ge 0$. 

Suppose that for some $j\in\{0,1,\dots,k-2\}$, we have that $\wb\by_s\in \wb C_s$ and 
$$
0\le \val(\wb\bx_1,\wb C_1,\wb\by_1,\dots,\wb\bx_s,\wb C_s,\wb\by_s)<\val(\emptyset)
$$ 
for all $s\in\{1,\dots,j\}$. 
Then, by the definition of $\wh\by_t$, we have
\begin{align*}
\val(\wb\bx_1,\wb C_1,\wb\by_1,\dots,\wb\bx_{j+1},\wb C_{j+1},\wb\by_{j+1})<\val(\wb\bx_1,\wb C_1,\wb\by_1,\dots,\wb\bx_{j},\wb C_{j},\wb\by_{j})<\val(\emptyset).
\end{align*}
If $\wb\by_{j+1}\notin\wb C_{j+1}$, then, for any $\bw'\in\prod_{s=j+2}^\infty\left(\XPC_s\times\Y^s\right)$, we claim that
\begin{align*}
\bw:=(\wb\bx_1,\wb C_1,\wb\by_1,\dots,\wb\bx_{j+1},\wb C_{j+1},\wb\by_{j+1},\bw')\notin\W_{\fB}.
\end{align*}
This is because for any $\tau\in[j]$, we must have $\cH|_{\wb\bx_{1},\wb\by_{1},\dots,\wb\bx_{\tau},\wb\by_{\tau}}\neq\emptyset$ since  $\val(\wb\bx_{1},\wb\by_{1},\wb C_{1},\dots,\wb\bx_{\tau},\wb C_{\tau},\wb\by_{\tau})\ge 0$ and $\wb\by_{s}\in C_s$ for any $s\in[j]$ by the induction hypothesis. 
Then, if $\wb\by_{j+1}\notin\wb C_{j+1}$, we must have $\bw\notin\W_{\fB}$. 

Since $(\wb\bx_1,\wb C_1,\wb\by_1,\dots,\wb\bx_{j+1},\wb C_{j+1},\wb\by_{j+1},\bw')\notin\W_{\fB}$ for any $\bw'\in\prod_{s=j+2}^\infty\left(\XPC_s\times\Y^s\right)$, 
there is an infinite active tree starting from $\wb\bx_1,\wb C_1,\wb\by_1,\dots,\wb\bx_{j+1},\wb C_{j+1},\wb\by_{j+1}$, which implies that 
\begin{align*}
\val(\wb\bx_1,\wb C_1,\wb\by_1,\dots,\wb\bx_{j+1},\wb C_{j+1},\wb\by_{j+1})=\Omega.
\end{align*} 
However, this cannot happen because we have shown that $\val(\wb\bx_1,\wb C_1,\wb\by_1,\dots,\wb\bx_{j+1},\wb C_{j+1},\wb\by_{j+1})<\val(\emptyset)<\omega_1$. Thus, it must hold that $\wb\by_{j+1}\in\wb C_{j+1}$ by contradiction. 

Then, we claim that 
$$
\val(\wb\bx_1,\wb C_1,\wb\by_1,\dots,\wb\bx_{j+1},\wb C_{j+1},\wb\by_{j+1})\ge 0.
$$ 
If on the contrary $\val(\wb\bx_1,\wb C_1,\wb\by_1,\dots,\wb\bx_{j+1},\wb C_{j+1},\wb\by_{j+1})=-1$, we will have $\cH|_{\wb\bx_1,\wb\by_1,\dots,\wb\bx_{j+1},\wb\by_{j+1}}=\emptyset$ because we have shown that $\wb\by_{s}\in\wb C_s$ for any $s\in[j+1]$ and $\cH|_{\wb\bx_1,\wb\by_1,\dots,\wb\bx_{j},\wb\bx_{j}}\neq \emptyset$.
However, since $(x_1,y_1,\dots,x_t,y_t)$ is a consistent sequence with $\cH$, we must have
$\cH|_{\wb\bx_1,\wb\by_1,\dots,\wb\bx_{j+1},\wb\by_{j+1}}\neq \emptyset$. Thus, there is a contradiction and we must have $\val(\wb\bx_1,\wb C_1,\wb\by_1,\dots,\wb\bx_{j+1},\wb C_{j+1},\wb\by_{j+1})\ge 0$.

Now, by induction, we can conclude that that $\wb\by_s\in \wb C_s$ and 
$$
0\le \val(\wb\bx_1,\wb C_1,\wb\by_1,\dots,\wb\bx_s,\wb C_s,\wb\by_s)<\val(\emptyset)<\omega_1
$$ 
for all $s\in\{1,\dots,k-1\}$. 
For any $\by''\in \Y^{k}\backslash C$, we must have 
$$
\val(\wb\bx_1,\wb C_1,\wb\by_1,\dots,\wb\bx_{k-1},\wb C_{k-1},\wb\by_{k-1},\wb\bx_k,C,\by'')=\Omega
$$
according to the same arguments we used for the proof that $\wb\by_{j+1}\in\wb C_{j+1}$ in the induction step. 
Then, by \eqref{eq:val_assump}, we have that
\begin{align*}
&\val(\wb\bx_1,\wb C_1,\wb\by_1,\dots,\wb\bx_{k-1},\wb C_{k-1},\wb\by_{k-1},\wb\bx_k,C,\by)
\\ \ge & 
\min\{\val(\wb\bx_1,\wb C_1,\wb\by_1,\dots,\wb\bx_{k-1},\wb C_{k-1},\wb\by_{k-1}),\val(\emptyset)\}
\end{align*}
for any $\by\in\Y^k$. However, this contradicts \citet[Propostion B.8]{universal_learning} since $$
0\le \val(\wb\bx_1,\wb C_1,\wb\by_1,\dots,\wb\bx_{k-1},\wb C_{k-1},\wb\by_{k-1})<\omega_1.
$$
Thus, we have $\dim(\cH(S,\wh\by_t))< k$.
\end{proof}

Let us fix a $\cH$-realizable distribution $P$ on $\X\times\Y$. Let $(\Omega_P,\F_P,\pr)$ denote the underlying probability space. Let $(X_1,Y_1),(X_2,Y_2),\dots$ be i.i.d. random variables on $(\Omega_P,\F_P,\pr)$ with $(X_1,Y_1)\sim P$. We have the following result regarding the consistency of the random sequence $((X_t,Y_t))_{t\ge1}$ with $\cH$.
\begin{lemma} \label{lem:consistency_random}
If $P$ is $\cH$-realizable and $(X_1,Y_1),(X_2,Y_2),\dots$ are i.i.d. random variables with distribution $P$, then, with probability one, for any $t\ge1$, there exists some $h\in\cH$ such that $h(X_s)=Y_s$ for all $s\in[t]$.
\end{lemma}
The proof of Lemma \ref{lem:consistency_random} can be found in the proof of \citet[Lemma 4.3.]{universal_learning}.

For any $k\in\N$ and function $g:\X^k\rightarrow 2^{\Y^k}$ where $2^{S}$ denotes the power set of the set $S$, we define
\begin{align*}
\per(g)=\pr[(Y_1,\dots,Y_k)\in g(X_1,\dots,X_k)].
\end{align*}
Now, let 
\begin{align*}
&\tau_t:=T_t(X_1,Y_1,\dots,X_t,Y_t),\\
&\wh\by_t(x_1,\dots,x_{\tau_{t}}):=\wh\bY_t(X_1,Y_1,\dots,X_t,Y_t,x_1,\dots,x_{\tau_{t}}).
\end{align*}
We first prove the following result when $\per(g)=0$.
\begin{lemma} \label{lem:realizable_subclass}
For any $k,n\in\N$ with $n\ge k$, any function $g:\X^k\rightarrow 2^{\Y^k}$, and any sequence $S=((X_i,Y_i))_{i=1}^n\sim P^n$, 
if $\per(g)=0$, then $((i,Y_i))_{i=1}^n$ is consistent with $\cH(S|_{\X},g)$ and $\cD$ is $\cH(S|_{\X},g)$-realizable a.s., where 
$S|_{\X}:=(X_1,X_2\dots,X_n)$ and
$\cD$ denotes the uniform distribution over $\{(i,Y_i)\}_{i=1}^n$, i.e., $\cD(\{(i,Y_i)\})=\frac{1}{n}$ for any $i\in[n]$.
\end{lemma}
\begin{proof}
Since $S\sim P^n$, according to Lemma \ref{lem:consistency_random}, there exists a random variable $H:\Omega\rightarrow\cH$ such that for $\pr$-a.e. $\omega\in\Omega$, $h=H(\omega)\in\cH$ satisfies that $h(X_i(\omega))=Y_i(\omega)$ for any $i\in[n]$.
Since $\per(g)=0$, we have that for $\pr$-a.e. $\omega\in\Omega$, $(Y_{i_1}(\omega),\dots,Y_{i_k}(\omega))\notin g(X_{i_1}(\omega),\dots,X_{i_k}(\omega))$ for all distinct $1\le t_1,\dots,t_k\le n$. 
Thus, for $\pr$-a.e. $\omega\in\Omega$, $h=H(\omega)$ satisfies that $(h(X_{i_1}(\omega)),\dots,h(X_{i_k}(\omega)))\notin g(X_{i_1}(\omega),\dots,X_{i_k}(\omega))$ for all distinct $1\le t_1,\dots,t_k\le n$.

Define the random variable $\wb H:\Omega\rightarrow\Y^{[n]}$ by $\wb H(\omega)(i):=H(\omega)(X_i(\omega))=h(X_i(\omega))$ where $h=H(\omega)$. Then, by the definition of $\cH(S|_{\X},g)$ and $\wb h$, we have that for $\pr$-a.e. $\omega\in\Omega$, $\wb h=\wb H(\omega)\in\cH(S(\omega)|_{\X},g)$ and
$\wb h(i)=Y_i(\omega)$ for any $i\in[n]$. Thus, 
\begin{align*}
\pr[\wb h(I)\neq Y_I|S]=\pr[h(X_I)\neq Y_I|S]=
\frac{1}{n}\indi\{h(X_i)\neq Y_i\}=0
\end{align*}
where $I$ is a random variable uniformly distributed over $[n]$ and is independent of $S$. Then, we know that $(I,Y_I)$ follows distribution $\cD$ conditional on $S$. 
Therefore, 
$$
\er(\wb h)=\pr[\wb h(I)\neq Y_I|S]=0, \pr\textup{-a.s.}
$$ 
Thus, $\inf_{h'\in\cH(S|_{\X},g)}\er(h')=0$ a.s., which implies that $\cD$ is $\cH(S|_{\X},g)$-realizable a.s.
\end{proof}

Similar to \citet[Lemma 5.7]{universal_learning}, we have the following lemma.
\begin{lemma} \label{lem:pattern_avoid}
$\pr[\per(\wh\by_t)>0]\rightarrow0$ as $t\rightarrow\infty$.
\end{lemma}
\begin{proof}
According to Lemma \ref{lem:consistency_random}, we have that $((X_i,Y_i))_{i\in\N}$ is consistent with $\cH$ a.s.. Then, by Proposition \ref{prop:pattern_avoid}, we have that 
\begin{align*}
T:=\sup\left\{s\ge1:(Y_{s-\tau_{s-1}+1},\dots,Y_s)\in\wh\by_{s-1}(X_{s-\tau_{s-1}+1},\dots,X_{s})\right\}
\end{align*}
is finite a.s.. 
Since $((X_i,Y_i))_{i\in\N}$ is an i.i.d. sequence of random variables, we have that for any $t\ge1$, $\wh\by_{t-1}$ is independent of $((X_i,Y_i))_{i\ge t}$.
Then, by the strong laws of large number, we have that with probability one,
\begin{align*}
&\lim_{m\rightarrow\infty}\frac{1}{m}\sum_{k=1}^m
\indi\big\{(Y_{t+(k-1)\tau_{t-1}},\dots,Y_{t+k\tau_{t-1}-1})\in\wh\by_{t-1}(X_{t+(k-1)\tau_{t-1}},\dots,X_{t+k\tau_{t-1}-1})\big\}
\\&=\bE\left[\indi\left\{(Y_{t},\dots,Y_{t+\tau_{t-1}-1})\in\wh\by_{t-1}(X_{t},\dots,X_{t+\tau_{t-1}-1})\right\}\right]
\\&=\per(\wh\by_{t-1}).
\end{align*}
Since $T<\infty$ implies that $\tau_{s-1}=\tau_{t-1}<\infty$ and $\wh\by_{s-1}=\wh\by_{t-1}$ for any $T<s,t<\infty$, it follows that for any $t\in\N$ such that $T<t<\infty$,
\begin{align*}
(Y_{t+(k-1)\tau_{t-1}},\dots,Y_{t+k\tau_{t-1}-1})\notin\wh\by_{t-1}(X_{t+(k-1)\tau_{t-1}},\dots,X_{t+k\tau_{t-1}-1}), \ \forall k\in \N
\end{align*}
and thus, 
\begin{align*}
\lim_{m\rightarrow\infty}\frac{1}{m}\sum_{k=1}^m
\indi\big\{(Y_{t+(k-1)\tau_{t-1}},\dots,Y_{t+k\tau_{t-1}-1})\in\wh\by_{t-1}(X_{t+(k-1)\tau_{t-1}},\dots,X_{t+k\tau_{t-1}-1})\big\}=0.
\end{align*}
Therefore, we can conclude that for any $t\in\N$, 
\begin{align*}
\{T<t\}\subseteq\Big\{
\lim_{m\rightarrow\infty}\frac{1}{m}\sum_{k=1}^m
\indi\big\{(Y_{t+(k-1)\tau_{t-1}},\dots,Y_{t+k\tau_{t-1}-1})\in\wh\by_{t-1}(X_{t+(k-1)\tau_{t-1}},\dots,X_{t+k\tau_{t-1}-1})\big\}=0\Big\}.
\end{align*}
Given the above results, we have
\begin{align*}
&\pr[\per(\wh\by_t)=0]
\\=&
\pr\left[\lim_{m\rightarrow\infty}\frac{1}{m}\sum_{k=1}^m
\indi\big\{(Y_{t+(k-1)\tau_{t-1}},\dots,Y_{t+k\tau_{t-1}-1})\in\wh\by_{t-1}(X_{t+(k-1)\tau_{t-1}},\dots,X_{t+k\tau_{t-1}-1})\big\}=0\right]
\\ \ge &
\pr\left[T<t\right].
\end{align*}
and
\begin{align*}
\pr[\per(\wh\by_t)>0]=1-\pr[\per(\wh\by_t)=0]\le \left[T\ge t\right].
\end{align*}
Since we have proved $T<\infty$ a.s., we have
\begin{align*}
\limsup_{t\rightarrow\infty}\pr[\per(\wh\by_t)>0]\le \lim_{t\rightarrow\infty}\pr[T\ge t]=0
\end{align*}
Therefore, $\lim_{t\rightarrow\infty}\pr[\per(\wh\by_t)>0]=0$.
\end{proof}

Analogous to \citet[Lemma 5.10]{universal_learning}, we have
\begin{lemma} \label{lem:time_func}
Given any $t^*\in\N$ such that $\pr[\per(\wh\by_{t^*})>0]\le \frac{1}{8}$, if $n\ge \max\{4(t^*+1),N\}$ for some $N\ge1$ dependent on the adversarial algorithm defined in Section \ref{sec:patern-avoid}, $P$, and $t^*$, 
then there exists a universally measurable function $\wh t_n=\wh t_n(X_1,Y_1,\dots,X_{\lfloor n/2\rfloor},Y_{\lfloor n/2\rfloor})\in[\lfloor n/4\rfloor-1]$ whose definition does not depend on the data distribution $P$ and some constants $C$ and $c$ independent of $n$ (but dependent on the algorithm, $P$, and $t^*$) such that 
\begin{align*}
\pr[\wh t_n\in\cT_{\good}]\ge 1-Ce^{-cn},
\end{align*}
where $\cT_{\good}:=\left\{1\le t\le t^*:\pr[\per(\wh\by_t)>0]\le\frac{3}{8}\right\}$.
\end{lemma}
\begin{proof}
For each $1\le t\le \lfloor n/4\rfloor-1$ and $1\le i\le \lfloor n/(4t)\rfloor$, define
\begin{align*}
\tau_t^i:=T_{t}(X_{(i-1)t+1},Y_{(i-1)t+1},\dots,X_{it},Y_{it})\le t+1\le \lfloor n/4\rfloor,
\end{align*}
\begin{align*}
\wh\by_t^i(x_1,\dots,x_{\tau_{t}^i}):=\wh\bY_{t}(X_{(i-1)t+1},Y_{(i-1)t+1},\dots,X_{it},Y_{it},x_{1},\dots,x_{\tau_t^i})
\end{align*}
for $\forall (x_1,\dots,x_{\tau_t^i})\in\X^{\tau_t^i}$, and
\begin{align*}
\wh e_t:=\frac{1}{\lfloor n/(4t)\rfloor}\sum_{i=1}^{\lfloor n/(4t)\rfloor}\indi\big\{(Y_{s+1,\dots,Y_{s+\tau_t^i}})\in\wh\by_t^i(X_{s+1,\dots,X_{s+\tau_t^i}})\textup{ for some }\lfloor n/4\rfloor\le s\le\lfloor n/2\rfloor-\tau_t^i\big\}.
\end{align*}
Since $\per(\wh\by_t^i)=0$ implies that $(Y_{s+1,\dots,Y_{s+\tau_t^i}})\notin\wh\by_t^i(X_{s+1,\dots,X_{s+\tau_t^i}})\textup{ for all }\lfloor n/4\rfloor\le s\le\lfloor n/2\rfloor-\tau_t^i$ a.s., we have that
\begin{align*}
\wh e_t\le e_t:=\frac{1}{\lfloor n/(4t)\rfloor}\sum_{i=1}^{\lfloor n/(4t)\rfloor}\indi\big\{\per(\wh\by_t^i)>0\big\}\ \textup{a.s.}
\end{align*}
Define 
$$
\wh t_n:=\min\big\{\inf\big\{1\le t\le\lfloor n/4\rfloor-1:\wh e_t<1/4\big\},\ \lfloor n/4\rfloor-1\big\}
$$
with the convention that $\inf\emptyset=+\infty$. 
Since $t^*\le \lfloor n/4\rfloor-1$, we can conclude that $\wh e_{t^*}<1/4$ implies $\wh t_n\le t^*$. Then, by Hoeffding's inequality, we have
\begin{align*}
\pr[\wh t_n>t^*]\le \pr[\wh e_{t^*}\ge1/4]\le \pr[e_{t^*}-\bE[e_{t^*}] \ge1/4]\le e^{-\lfloor n/(4t^*)\rfloor/32}.
\end{align*}
For any $t\in\N$ such that $t\le t^*$ and $\pr[\per(\wh\by_t)>0]>3/8$, since 
$$
\lim_{z\rightarrow0}\pr[\per(\wh\by_t)>z]=\pr[\per(\wh\by_t)>0]>3/8
$$ 
by the continuity of probability measures, there exists some $\epsilon_t>0$ such that $\pr[\per(\wh\by_t)>\epsilon_t]>1/4+1/16$. Because $t^*<\infty$, there exists $\epsilon>0$ such that $\pr[\per(\wh\by_t)>\epsilon]>1/4+1/16$ for all $1\le t\le t^*$ such that $\pr[\per(\wh\by_t)>0]>3/8$.

Fixing an arbitrary $1\le t\le t^*$ such that $\pr[\per(\wh\by_t)>0]>3/8$, by Hoeffding's inequality, we have
\begin{align*}
\pr\Big[\frac{1}{\lfloor n/(4t)\rfloor}\sum_{i=1}^{\lfloor n/(4t)\rfloor}\indi\left\{\per(\wh\by_t)>\epsilon\right\}<\frac{1}{4}\Big]\le e^{-\lfloor n/(4t^*)\rfloor/128}.
\end{align*}
For any $2\le \tau\le \lfloor n/4\rfloor$ and any $g:\X^\tau\rightarrow\{0,1\}^\tau$ with $\per(g)>\epsilon$, we have that
\begin{align*}
&\pr\big[(Y_{s+1},\dots,Y_{s+\tau})\notin g(X_{s+1},\dots,X_{s+\tau}) \textup{ for all }\lfloor n/4\rfloor\le s\le \lfloor n/2\rfloor-\tau\big]
\\ \le &
\pr\big[(Y_{\lfloor n/4\rfloor+(k-1)\tau+1},\dots,Y_{\lfloor n/4\rfloor+k\tau})\notin g(X_{\lfloor n/4\rfloor+(k-1)\tau+1},\dots,X_{\lfloor n/4\rfloor+k\tau})
\\&\quad\quad\quad\quad\textup{ for all }1\le k \le \lfloor n/(4\tau)\rfloor\big]
\\ =&
(1-\per(g))^{\lfloor n/(4\tau)\rfloor}
\\ \le&
(1-\epsilon)^{\lfloor n/(4\tau)\rfloor}.
\end{align*}
Thus, by union bound and the fact that $\tau_t^i\le t+1\le t^*+1$, we have
\begin{align*}
&\pr\Big[\indi\big\{\per(\wh\by_t^i)>\epsilon\big\}>
\indi\big\{(Y_{s+1,\dots,Y_{s+\tau_t^i}})\in\wh\by_t^i(X_{s+1,\dots,X_{s+\tau_t^i}})\textup{ for some }\lfloor n/4\rfloor\le s\le\lfloor n/2\rfloor-\tau_t^i\big\}
\\&\quad\quad\quad\quad\quad\quad\quad\quad\quad
\textup{ for some }1\le i\le\lfloor n/(4t)\rfloor\Big]
\\ \le &
\pr\Big[\indi\big\{(Y_{s+1,\dots,Y_{s+\tau_t^i}})\in\wh\by_t^i(X_{s+1,\dots,X_{s+\tau_t^i}})\textup{ for some }\lfloor n/4\rfloor\le s\le\lfloor n/2\rfloor-\tau_t^i\big\}=0
\\&\quad\quad\quad\quad\quad\quad\quad\quad\quad
\textup{ for some }1\le i\le\lfloor n/(4t)\rfloor\Big]
\\ = &
\pr\Big[\exists 1\le i\le\lfloor n/(4t)\rfloor \textup{ s.t. } (Y_{s+1,\dots,Y_{s+\tau_t^i}})\notin\wh\by_t^i(X_{s+1,\dots,X_{s+\tau_t^i}})\textup{ for all }\lfloor n/4\rfloor\le s\le\lfloor n/2\rfloor-\tau_t^i\Big]
\\\le&
\lfloor n/(4t)\rfloor(1-\epsilon)^{\lfloor n/(4(t^*+1))\rfloor}.
\end{align*}
Then, we can conclude that
\begin{align*}
&\pr[\wh t_n=t]\\\le &\pr[\wh e_t<1/4]
\\\le& 
\pr\Big[\frac{1}{\lfloor n/(4t)\rfloor}\sum_{i=1}^{\lfloor n/(4t)\rfloor}\indi\left\{\per(\wh\by_t)>\epsilon\right\}<\frac{1}{4}\Big]
\\&+\pr\Big[\indi\big\{\per(\wh\by_t^i)>\epsilon\big\}>
\indi\big\{(Y_{s+1,\dots,Y_{s+\tau_t^i}})\in\wh\by_t^i(X_{s+1,\dots,X_{s+\tau_t^i}})\textup{ for some }\lfloor n/4\rfloor\le s\le\lfloor n/2\rfloor-\tau_t^i\big\}
\\&\quad\quad\quad\quad\quad\quad\quad\quad\quad
\textup{ for some }1\le i\le\lfloor n/(4t)\rfloor\Big]
\\ \le &
e^{-\lfloor n/(4t^*)\rfloor/128}+\lfloor n/(4t)\rfloor(1-\epsilon)^{\lfloor n/(4(t^*+1))\rfloor}
\end{align*}
and
\begin{align*}
\pr[\wh t_n\notin \cT_{\good}]\le& \pr[\wh t_n>t^*]+\pr[\wh t_n\le t^* \textup{ and }\per(\wh\by_{\wh t_n})>3/8]
\\= &
\pr[\wh t_n>t^*]+\pr[\wh t_n=t\textup{ for some }t \textup{ s.t. }1\le t\le t^* \textup{ and }\per(\wh\by_t)>3/8]
\\ \le &
e^{-\lfloor n/(4t^*)\rfloor/32}
+t^*e^{-\lfloor n/(4t^*)\rfloor/128}
+t^*\lfloor n/4\rfloor(1-\epsilon)^{\lfloor n/(4(t^*+1))\rfloor}
\end{align*}
Note that $e^{-\lfloor n/(4t^*)\rfloor/32}\le e^{\frac{1}{32}}e^{-\frac{n}{128t^*}}$, 
$e^{-\lfloor n/(4t^*)\rfloor/32}\le t^*e^{\frac{1}{128}}e^{-\frac{n}{512t^*}}$, and
\begin{align*}
-\log\big(t^*\lfloor n/4\rfloor(1-\epsilon)^{\lfloor n/(4(t^*+1))\rfloor}\big)
\ge\frac{\log\big(\frac{1}{1-\epsilon}\big)}{8t^*+2}n-\log\big(\frac{t^*}{1-\epsilon}\big)+
\frac{\log\big(\frac{1}{1-\epsilon}\big)}{8t^*+2}n-\log( n/4),
\end{align*}
Since $\log(n/4)\le \sqrt{n}$ for all $n\ge 4$, we have that
$\frac{\log\big(\frac{1}{1-\epsilon}\big)}{8t^*+2}n\ge\log( n/4)$
and
\begin{align*}
t^*\lfloor n/4\rfloor(1-\epsilon)^{\lfloor n/(4(t^*+1))\rfloor}\le \frac{t^*}{1-\epsilon}
\exp\Big(-\frac{\log\big(\frac{1}{1-\epsilon}\big)}{8t^*+2}n\Big)
\end{align*}
for all $n\ge \max\{4,\big(\frac{8t^*+2}{\log(1-\epsilon)}\big)^2\}$.

Thus, for any $n\ge \max\big\{4(t^*+1),\big(\frac{8t^*+2}{\log(1-\epsilon)}\big)^2\big\}$, we have $\pr[\wh t_n\notin\cT_{\good}]\le Ce^{-cn}$ for $c:=\min\Big\{\frac{1}{512t^*},\frac{\log\big(\frac{1}{1-\epsilon}\big)}{8t^*+2}\Big\}$ and $C:=e^{\frac{1}{32}}+t^*e^{\frac{1}{128}}+\frac{t^*}{1-\epsilon}$. Since $\epsilon$ depends on $t^*$, the data distribution $P$, and the algorithm, but does not depend on $n$, the lemma is proved.
\end{proof}

According to \citet[Theorem 1]{brukhim2022characterization} and its proof in \citet[Section 4.5]{brukhim2022characterization}, we have the following theorem.
\begin{theorem} \label{thm:multiclass_alg}
Let $\cH\subseteq\Y^\X$ be an hypothesis class with DS dimension $d<\infty$. There is a learning algorithm $\A:\cup_{n=1}^\infty(\X\times\Y)^n\rightarrow\cH$ with the following guarantee. For any $\cH$-realizable distribution $\cD$, any $\delta\in(0,1)$, any integer $n\ge1$, any sample $(S,(X,Y))\sim\cD^{n+1}$ where $S\in(\X\times\Y)^n$ and $(X,Y)\in\X\times\Y$, the output hypothesis $\A(\cH,S)$ satisfies that
\begin{align*}
\pr[\A(\cH,S)(X)\neq Y|S]\le O\left(\frac{d^{3/2}\log^2(n)+\log(1/\delta)}{n}\right).
\end{align*}
with probability at least $1-\delta$.
\end{theorem}
We immediately have the following corollary from Theorem \ref{thm:multiclass_alg}.
\begin{corollary} \label{coro:er_multiclass}
For the hypothesis class $\cH$, learning algorithm $\A$, distribution $\cD$, and any integer $n\ge1$ described in Theorem \ref{thm:multiclass_alg}, we have
\begin{align*}
\pr[\A(\cH,S)(X)\neq Y]\le \frac{Cd^{3/2}\log^2(n)}{n}.
\end{align*}
for some constant $C>0$.
\end{corollary}
\begin{proof}
Define $R=\pr[A(S)(X)\neq Y|S]$. 
Then, by Theorem \ref{thm:multiclass_alg}, there exists some constant $C_1>0$ such that 
$$
\pr\left[R\ge\frac{C_1d^{3/2}\log^2(n)+C_1\log(1/\delta)}{n}\right]\le \delta.
$$
Define $t=\frac{C_1d^{3/2}\log^2(n)+C_1\log(1/\delta)}{n}\in(\frac{C_1d^{3/2}\log^2(n)}{n},\infty)$. Then, we have $\delta=\exp(d^{3/2}\log^2(n)-nt/C_1)$. It follows that
\begin{align*}
\pr[\A(S)(X)\neq Y]=&\bE[R]
\\=&\bE\left[R\indi\Big\{R\le\frac{C_1d^{3/2}\log^2(n)}{n}\Big\}\right]+
\bE\left[R\indi\Big\{R>\frac{C_1d^{3/2}\log^2(n)}{n}\Big\}\right]
\\ \le &\frac{C_1d^{3/2}\log^2(n)}{n}+
\int_{\frac{C_1d^{3/2}\log^2(n)}{n}}^\infty
\pr[R\ge t]dt
\\ \le &\frac{C_1d^{3/2}\log^2(n)}{n}+
\int_{\frac{C_1d^{3/2}\log^2(n)}{n}}^\infty
\exp(d^{3/2}\log^2(n)-nt/C_1) dt
\\ = & \frac{C_1d^{3/2}\log^2(n)+C_1}{n}.
\end{align*}
Thus, there exists some constant $C>0$ such that $\pr[\A(S)(X)\neq Y]\le \frac{Cd^{3/2}\log^2(n)}{n}$.
\end{proof}

The following lemma is very important in upper bounding the error probability for learning algorithms with access to leave-one-out samples using their guarantees on all samples. 
\begin{lemma} \label{lem:partial_sample}
Suppose that $A$ is an algorithm that for any positive integer $n$, any feature space $\cal Z$, and any label space $\cal W$, given a hypothesis class $\cH\subseteq \cal W^{\cal{Z}}$ and a sequence of samples $((z_i,w_i))_{i=1}^n$ consistent with $\cH$, outputs a hypothesis $h\in{\cal W}^{\cal Z}$.

Let $\X$ and $\Y$ denote the feature space and label space of the samples.
Suppose that $H:\cup_{n=1}^\infty\X^n\rightarrow\cup_{n=1}^\infty2^{\Y^{[n]}}$ is a function that for any positive integer $n$, given a sequence $(x_1,\dots,x_n)\in\X^n$, constructs a hypothesis class $H((x_1,\dots,x_n))\subseteq\Y^{[n]}$ such that $((1,x_1),\dots,(n,x_n))$ is consistent with $H((x_1,\dots,x_n))$. 

For any positive integer $n$ and any sequence $S'=((x_1,y_1),\dots,(x_{n+1},y_{n+1}))\in(\X\times\Y)^{n+1}$,
define $S'|_\X:=(x_1,x_2,\dots,x_{n+1})$. 
Let $\cD'$ denote the uniform distribution over $\{(i,y_i)\}_{i=1}^{n+1}$ and $\cD$ denote the uniform distribution over $\{(i,y_i)\}_{i=1}^{n}$ (i.e., $\cD'(\{(i,y_i)\})=\frac{1}{n+1}$ for any $i\in[n+1]$ and $\cD(\{(i,y_i)\})=\frac{1}{n}$ for any $i\in[n]$). 

Then, for any $T\sim\cD^{\lceil n/2\rceil}$ and $(T',(I,y_I))\sim(\cD')^{\lceil n/2\rceil+1}$ with $T'\in\{(i,y_i):1\le i\le n+1\}^{\lceil n/2\rceil}$, $I\in[n+1]$, and $\lceil x\rceil:=\min\{n\in\mathbb{Z}:n\ge x\}$ for any $x\in\mathbb{R}$, we have
\begin{align*}
\pr[A(H(S'|_{\X}),T)(n+1)\neq y_{n+1}]\le 2
\pr[A(H(S'|_{\X}),T')(I)\neq y_I].
\end{align*}
\end{lemma}
\begin{proof}
Since $(T',(I,y_I))\sim(\cD')^{\lceil n/2\rceil+1}$, we have that
\begin{align*}
\pr[(n+1,y_{n+1})\in T']\le |T'|\frac{1}{n+1}=\frac{\lceil n/2\rceil}{n+1}\le\frac{1}{2}.
\end{align*}
By the assumption on $H$, we know that $T'$ is consistent with $H(S'|_{\X})$. Thus, by the assumption on $A$, we have that $A(H(S'|_{\X}),T')(i)=y_i$ for any $(i,y_i)\in T'$. It follows that
\begin{align*}
&\pr[A(H(S'|_{\X}),T')(I)\neq y_I]
\\=&\bE[\pr[A(H(S'|_{\X}),T')(I)\neq y_I|T']]
\\=&
\frac{1}{n+1}\sum_{i=1}^{n+1}\pr[A(H(S'|_{\X}),T')(i)\neq y_i]
\\ \ge &
\frac{1}{n+1}\sum_{i=1}^{n+1}\bE[\indi\{(i,y_i)\notin T'\}
\indi\{A(H(S'|_{\X}),T')(i)\neq y_i\}]
\\=&
\bE[\indi\{(n+1,y_{n+1})\notin T'\}
\indi\{A(H(S'|_{\X}),T')(n+1)\neq y_{n+1}\}]
\\=&
\pr[(n+1,y_{n+1})\notin T']\bE[
\indi\{A(H(S'|_{\X}),T')(x_{n+1})\neq y_{n+1}\}|(n+1,y_{n+1})\notin T']
\\ \ge&
\frac{1}{2}\pr[
A(H(S'|_{\X}),T')(n+1)\neq y_{n+1}|(n+1,y_{n+1})\notin T']
\\=&
\frac{1}{2}\pr[
A(H(S'|_{\X}),T)(n+1)\neq y_{n+1}],
\end{align*}
where the last inequality follows from the fact that $\pr[T'\in B|x_{n+1}\notin T']=\pr[T\in B]$ for any $B\subseteq\{((i_1,y_{i_1}),\dots,(i_{\lceil n/2\rceil},y_{i_{\lceil n/2\rceil}})):1\le i_1,\dots,i_{\lceil n/2\rceil}\le n+1\}$. 
\end{proof}

Now, we are ready to prove the main theorem that relates guarantees of learning algorithms on uniform rates to universal rates.
\begin{theorem} \label{thm:universal_rate_DSL}
Suppose that $A$ is a learning algorithm which for any hypothesis class $H$ with DS dimension at most $d<\infty$, any $H$-realizable distribution $\cD$, any number $n\in\N$, and any sample $S\sim\cD^n$, outputs a hypothesis $A(H,S)$ with $\bE[\er(A(H,S))]\le r(n,d)$, where $r:\N\times\N\rightarrow[0,1]$ is some rate function non-increasing for any $d\in\N$. 
Then, there is an algorithm $A'$ satisfying that for any hypothesis class $\cH$ that does not have an infinite DSL tree and any $\cH$-realizable distribution $P$, there exist some constants $C,c>0$ and $d_0\in\N$ such that for all $n\in\N$ and $S\sim P^n$, $A'$ outputs a hypothesis $A'(\cH,S)\in\cH$ with 
$$
\bE[\er(A'(\cH,S))]\le Ce^{-cn}+32r(\lceil n/4\rceil,d_0).
$$
\end{theorem}
\begin{proof}
According to Lemma \ref{lem:pattern_avoid}, there exists $t^*\in\N$ such that $\pr[\per(\wh\by_{t^*})>0]\le \frac{1}{8}$. 
Then, for any $n\ge \max\{4(t^*+1),N\}$ with $N$ specified in Lemma \ref{lem:time_func}, let $\wh t_n\in[\lfloor n/4\rfloor-1]$ be the random time constructed in Lemma \ref{lem:time_func}. For any $t\in[\lfloor n/4\rfloor-1]$ and any $i\in[n/(4\wh t_n)]$,
define
$$
\tau_{t}^i:=T_{t}(X_{(i-1)t+1},Y_{(i-1)t+1},\dots,X_{it},Y_{it})\le t+1\le \lfloor n/4\rfloor,
$$
and
$$
\wh\by_{t}^i:\X^{\tau_{t}^i}\rightarrow\Y^{\tau_{t}^i},\ (x_1,\dots,x_{\tau_{t}^i})\mapsto 
\wh \bY_{t}(X_{(i-1)t+1},Y_{(i-1)t+1},\dots,X_{it},Y_{it},x_1,\dots,x_{\tau_{t}^i})
$$ 
as in the proof of Lemma \ref{lem:time_func}. 

For any $t\in\cT_{\good}$, since $\pr[\per(\wh\by_t)>0]\le\frac{3}{8}$, by a Chernoff bound, we have
\begin{align*}
\pr\left[\frac{1}{\lfloor n/(4t)\rfloor}\sum_{i=1}^{\lfloor n/(4t)\rfloor}\indi\{\per(\wh \by_t^i)>0\}>\frac{7}{16}\right]\le e^{-\lfloor n/(4t)\rfloor/128}
\le e^{-\lfloor n/(4t^*)\rfloor/128}.
\end{align*}
Using union bound, we have
\begin{align}
\notag
&\pr\left[\frac{1}{\lfloor n/(4\wh t_n)\rfloor}\sum_{i=1}^{\lfloor n/(4\wh t_n)\rfloor}\indi\{\per(\wh \by_{\wh t_n}^i)>0\}>\frac{7}{16},\wh t_n\in\cT_\good\right]
\\ \le & \notag
\sum_{t\in\cT_{\good}}
\pr\left[\frac{1}{\lfloor n/(4t)\rfloor}\sum_{i=1}^{\lfloor n/(4t)\rfloor}\indi\{\per(\wh \by_{t}^i)>0\}>\frac{7}{16}\right]
\\ \le & \label{eq:ub1}
t^*e^{-\lfloor n/(4t^*)\rfloor/128}.
\end{align}

Define the sequence $S:=((1,Y_{\lfloor n/2\rfloor+1}),(2,Y_{\lfloor n/2\rfloor+2}),\dots,(n-\lfloor n/2\rfloor,Y_n))$.
Let $\cD$ denote the uniform distribution over the elements in $S$ (i.e., $\cD(\{(i,Y_{\lfloor n/2\rfloor+i})\})=\frac{1}{n-\lfloor n/2\rfloor}$ for any $i\in[n-\lfloor n/2\rfloor]$). Let $T^1,\dots,T^{\lfloor n/(4\wh t_n)\rfloor}$ denote an i.i.d. sequence of random variables with $T^1\sim \cD^{\lceil(n-\lfloor n/2\rfloor)/2\rceil}$. 
For any $i\in[\lfloor n/(4\wh t_n)\rfloor]$ and any $x\in\X$, define the hypothesis class $\cH^i(x):=\cH((X_{\lfloor n/2\rfloor+1},\dots,X_{n},x),\wh\by_{\wh t_n}^i)$. 
Then, for any $i\in [\lfloor n/(4\wh t_n)\rfloor]$, we can define the following prediction function $$
\wh y^i:\X\rightarrow\Y,\ x\mapsto A(\cH^i(x),T^i)(n-\lfloor n/2\rfloor+1).
$$
Let $\wh h_n$ be the majority vote of $\wh y^i$ for $i\in[\lfloor n/(4\wh t_n)\rfloor]$. $\wh h_n$ will be the final output of our learning algorithm. Now, we need to upper bound the error rate $\bE[\er(\wh h_n)]$.

Recall that $P$ denotes the underlying data distribution that is $\cH$-realizable.
Suppose that $(X,Y)\sim P$ and is independent of $\{(X_i,Y_i)\}_{i=1}^n$.
Then, we have
\begin{align}
\notag
&\bE[\er(\wh h_n)]
\\=& \notag
\pr[\wh h_n(X)\neq Y]
\\ \le & \notag
\pr\left[\frac{1}{\lfloor n/(4\wh t_n)\rfloor}\sum_{i=1}^{\lfloor n/(4\wh t_n)\rfloor}\indi\{\wh y^i(X)\neq Y\}\ge \frac{1}{2}\right]
\\ \le & \notag
\pr[\wh t_n\notin \cT_{\good}]+
\pr\left[\frac{1}{\lfloor n/(4\wh t_n)\rfloor}\sum_{i=1}^{\lfloor n/(4\wh t_n)\rfloor}\indi\{\per(\wh \by_{\wh t_n}^i)>0\}>\frac{7}{16},\wh t_n\in\cT_\good\right]
\\&+ \label{eq:ub0}
\pr\left[
\wh t_n\in\cT_\good,\ 
\frac{1}{\lfloor n/(4\wh t_n)\rfloor}\sum_{i=1}^{\lfloor n/(4\wh t_n)\rfloor}\indi\{\per(\wh \by_{\wh t_n}^i)=0\}\ge \frac{9}{16},\ 
\frac{1}{\lfloor n/(4\wh t_n)\rfloor}\sum_{i=1}^{\lfloor n/(4\wh t_n)\rfloor}\indi\{\wh y^i(X)\neq Y\}\ge \frac{1}{2}
\right].
\end{align}

Define the sequence $S':=((1,Y_{\lfloor n/2\rfloor+1}),\dots,(n-\lfloor n/2\rfloor,Y_{n}),(n-\lfloor n/2\rfloor+1,Y))$ and conditional on $S'$, let $\cD'$ denote the uniform distribution over the elements in $S'$ (i.e., $\cD'(\{(i,Y_{\lfloor n/2\rfloor+i})\})=\frac{1}{n-\lfloor n/2\rfloor+1}$ for any $i\in[n-\lfloor n/2\rfloor]$ and $\cD'(\{(n-\lfloor n/2\rfloor+1,Y)\})=\frac{1}{n-\lfloor n/2\rfloor+1}$).
Let $T'\sim(\cD')^{\lceil(n-\lfloor n/2\rfloor)/2\rceil}$ and $(I,Y')\sim\cD'$ be two independent samples from $S'$ conditional on $S'$. 

For any $i\in[\lfloor n/(4\wh t_n)\rfloor$, by Lemma \ref{lem:consistency_random}, $(X_{(i-1)\wh t_n+1},Y_{(i-1)\wh t_n+1},\dots,X_{i\wh t_n},Y_{i\wh t_n})$ is consistent with $\cH$ a.s. Then, by Lemma \ref{lem:DS_dim_subclass}, we have that with probability 1, $\dim(\cH^i(X))<\tau_{\wh t_n}^i$ and therefore,
\begin{align*}
\indi\{\wh t_n\in\cT_{\good}\}\dim(\cH^i(X))<t^*.
\end{align*}
Moreover, if $\per(\wh \by_{\wh t_n}^i)=0$, by Lemma \ref{lem:realizable_subclass},
we have that $S'$ is consistent with $\cH^i(X)$ and $\cD'$ is $\cH^i(X)$-realizable a.s. 
Then, it follows from Lemma \ref{lem:partial_sample} and the property of $A$ that
\begin{align*}
&\indi\{\wh t_n\in\cT_{\good}\}\indi\{\per(\wh\by_{\wh t_n}^i)=0\}\pr[\wh y^i(X)\neq Y|((X_j,Y_j))_{j=1}^{n},X,Y]
\\=&
\indi\{\wh t_n\in\cT_{\good}\}\indi\{\per(\wh\by_{\wh t_n}^i)=0\}
\pr[A(\cH^i(X),T^i)(n-\lfloor n/2\rfloor+1)\neq Y|((X_j,Y_j))_{j=1}^{n},X,Y]
\\ \le& 
2\indi\{\wh t_n\in\cT_{\good}\}\indi\{\per(\wh\by_{\wh t_n}^i)=0\}\pr[A(\cH^i(X),T')(I)\neq Y'|((X_j,Y_j))_{j=1}^{n},X,Y]
\\ \le &
2r(\lceil(n-\lfloor n/2\rfloor)/2\rceil,t^*).
\end{align*}
By the properties of conditional expectation, we have that
\begin{align*}
&\indi\{\wh t_n\in\cT_{\good}\}\indi\{\per(\wh\by_{\wh t_n}^i)=0\}\pr\left[\wh y^i(X)\neq Y\big|((X_j,Y_j))_{j=1}^{\lfloor n/2\rfloor}\right]
\\=&
\bE\left[\indi\{\wh t_n\in\cT_{\good}\}\indi\{\per(\wh\by_{\wh t_n}^i)=0\}\pr[\wh y^i(X)\neq Y|((X_j,Y_j))_{j=1}^{n},X,Y]\big|((X_j,Y_j))_{j=1}^{\lfloor n/2\rfloor}\right]
\\ \le &
2r(\lceil(n-\lfloor n/2\rfloor)/2\rceil,t^*).
\end{align*}
Since $\frac{9}{16}+\frac{1}{2}=1+\frac{1}{16}$, by Markov's inequality and the above inequality, we have
\begin{align}
\notag
&\pr\left[
\wh t_n\in\cT_\good,\ 
\frac{1}{\lfloor n/(4\wh t_n)\rfloor}\sum_{i=1}^{\lfloor n/(4\wh t_n)\rfloor}\indi\{\per(\wh \by_{\wh t_n}^i)=0\}\ge \frac{9}{16},\ 
\frac{1}{\lfloor n/(4\wh t_n)\rfloor}\sum_{i=1}^{\lfloor n/(4\wh t_n)\rfloor}\indi\{\wh y^i(X)\neq Y\}\ge \frac{1}{2}
\right]
\\\le& \notag
\pr\left[
\indi\{\wh t_n\in\cT_\good\}
\frac{1}{\lfloor n/(4\wh t_n)\rfloor}\sum_{i=1}^{\lfloor n/(4\wh t_n)\rfloor}\indi\{\per(\wh \by_{\wh t_n}^i)=0\}\indi\{\wh y^i(X)\neq Y\}\ge \frac{1}{16}
\right]
\\\le & \notag
16\bE\left[
\indi\{\wh t_n\in\cT_\good\}
\frac{1}{\lfloor n/(4\wh t_n)\rfloor}\sum_{i=1}^{\lfloor n/(4\wh t_n)\rfloor}\indi\{\per(\wh \by_{\wh t_n}^i)=0\}\indi\{\wh y^i(X)\neq Y\}
\right]
\\=& \notag
16\bE\left[
\frac{1}{\lfloor n/(4\wh t_n)\rfloor}\sum_{i=1}^{\lfloor n/(4\wh t_n)\rfloor}
\indi\{\wh t_n\in\cT_\good\}
\indi\{\per(\wh \by_{\wh t_n}^i)=0\}\pr\left[\wh y^i(X)\neq Y
\big|((X_j,Y_j))_{j=1}^{\lfloor n/2\rfloor}\right]\right]
\\\le& \notag
32\bE\left[\frac{1}{\lfloor n/(4\wh t_n)\rfloor}\sum_{i=1}^{\lfloor n/(4\wh t_n)\rfloor} r(\lceil(n-\lfloor n/2\rfloor)/2\rceil,t^*)\right]
\\ \le & \notag
32r(\lceil(n-\lfloor n/2\rfloor)/2\rceil,t^*)
\\ \le &
\label{eq:ub2}
32r(\lceil n/4\rceil,t^*).
\end{align}
By \eqref{eq:ub0}, \eqref{eq:ub1}, Lemma \ref{lem:time_func}, and \eqref{eq:ub2}, we have
\begin{align*}
\bE[\er(\wh h)]
\le 
C_1e^{-c_1n}+t^*e^{-\lfloor n/(4t^*)\rfloor/128}+
32r(\lceil n/4\rceil,t^*).
\end{align*}
\end{proof}

Then, we immediately have the following result.
\begin{theorem} \label{thm:linear_rate}
If $\cH$ does not have an infinite DSL tree, then $\cH$ is learnable at rate $\frac{\log^2(n)}{n}$.
\end{theorem}
\begin{proof}
According to Corollary \ref{coro:er_multiclass} and Theorem \ref{thm:universal_rate_DSL}, we know that there exists an algorithm $A$ satisfies that for any $\cH$-realizable distribution $P$, there exists some constants $C_0,C_1,c_0>0$ and $d_0\in\N$ such that for all $n\in\N$ large enough and $S\sim P^n$, the output hypothesis $A(\cH,S)$ of $A$ has the error rate
\begin{align*}
\bE[A(\cH,S)]\le C_0e^{-c_0n}+C_1\frac{d_0^{3/2}\log^2(n)}{n}.
\end{align*}
Thus, there exists some constants $C>0$ such that 
\begin{align*}
\bE[A(\cH,S)]\le C\frac{\log^2(n)}{n},
\end{align*}
which implies that $\cH$ is learnable at rate $\frac{\log^2(n)}{n}$.
\end{proof}

\subsection{Concluding proof}
We conclude with the proof of Theorem \ref{thm:near-linear-rates}.
\begin{proof}[of Theorem \ref{thm:near-linear-rates}]
Theorem \ref{thm:near-linear-rates} follows directly from Theorem \ref{thm:linear_rate_lower} and Theorem \ref{thm:linear_rate}.
\end{proof}

\section{Arbitrary Slow Rates} \label{sec:arb_slow}
In this section, we provide the complete proof of Theorem \ref{thm:arb_slow}. First, we show two lemmas regarding the properties of pseudo-cubes. 
\begin{lemma} \label{lem:pseudo-cube-prop}
For any positive integer $d$, any label class $\Y$, any pseudo-cube $H\subseteq\Y^d$ of dimension $d$, any $j\in[d]$, and any label $y\in\Y$, define $N:=|H|$ and $H_{y}^j:=\{h\in H:h(j)=y\}$. Then, we have
\begin{align*}
|H_y^j|\le \frac{1}{2}N.
\end{align*}
\end{lemma}
\begin{proof}
We prove by contradiction. Suppose on the contrary that there exist some $j\in[d]$ and $y\in\Y$ such that $|H_y^j|>\frac{1}{2}n$. 
The definition of pseudo-cube implies that $|H|\ge 2$. Then, there exist $h, h'\in H_y^j$ with $h\neq h'$. 
Let $f$ and $f'$ denote the $j$-neighbors of $h$ and $h'$ in $H$; i.e., there exists $f,f'\in H$ such that $f(j)\neq h(j)=y$, $f'(j)\neq h'(j)=y$, $f(i)=h(i)$, and $f'(i)=h'(i)$ for any $i\in[d]\backslash\{j\}$. 
Since $h\neq h'$ and $h(j)=y=h'(j)$, there exists some $j'\in[d]\backslash\{j\}$ such that $h(j')\neq h'(j')$. It follows that $f(j')=h(j')\neq h'(j')=f'(j')$ and $f'\neq f$. Then, we have
\begin{align*}
|\{h\in H:h(j)\neq y\}|\ge |H_y^j|>\frac{1}{2}n
\end{align*}
and
\begin{align*}
n=|\{h\in H:h(j)\neq y\}|+|H_y^j|>n,
\end{align*}
which is a contradiction. Thus, we must have $|H_y^j|\le \frac{1}{2}n$.
\end{proof}

\begin{lemma} \label{lem:pseudo-cube-proj}
For any integer $d\ge 2$, $n\in[d-1]$, and $1\le j_1<\dots<j_n\le d$, any label class $\Y$, any pseudo-cube $H\subseteq\Y^d$ of dimension $d$, and any hypothesis $g\in H$, define $\sfJ:=(j_1,\dots,j_n)$ and $\sfK=(k_1,\dots,k_{d-n})$ such that $1\le k_1<\dots<k_{n-d}\le d$ and $\{j_1,\dots,j_n,k_1,\dots,k_{n-d}\}=[d]$. 
Then, $H_{g,\sfJ}:=\{h|_{\sfK}:h\in H,h(j_i)=g(j_i),\forall i\in[n]\}$ is a pseudo-cube of dimension $n-d$.
\end{lemma}
\begin{proof}
For any $f\in H_{g,J}$, there exists some $f'\in H$ such that $f=f'|_{\sfK}$. Then, for any $i\in[n-d]$, there exists some $h'\in H$ such that $h'(k_i)\neq f'(k_i)$ and $h'(l)=f'(l)$ for all $l\in[d]\backslash\{k_i\}$. Since $k_i\notin\{j_1,\dots,j_n\}$, we have $h'|_{\sfJ}=f'|_{\sfJ}=g|_{\sfJ}$ and $h:=h'|_{\sfK}\in H_{g,\sfJ}$. 
Then, we have $h(i)=h'(k_i)\neq f'(k_i)=f(i)$ and $h(m)=h'(k_m)=f'(k_m)=f(m)$ for any $m\in[n-d]\backslash\{i\}$, which implies that $H_{g,\sfJ}$ is a pseudo-cube. 

\end{proof}

Now, we present the proof of Theorem \ref{thm:arb_slow}
\begin{proof}[of Theorem \ref{thm:arb_slow}]
Suppose that $\cH$ has an infinite DSL tree. Fix an arbitrary rate $R$ with $\lim_{n\rightarrow\infty}R(n)=0$ and an arbitrary learning algorithm $A$. 
According to \citet[Lemma 5.12]{universal_learning}, there exist a sequence of non-negative numbers $(p_i)_{i\in\N}$ for which $\sum_{i=1}^\infty p_i=1$, two strictly increasing sequences of positive integers $(n_i)_{i\in\N}$ and $(k_i)_{i\in\N}$, and a constant $\frac{1}{2}\le C\le 1$ such that for all $i\in\N$, we have $\sum_{k>k_i}p_k\le\frac{1}{n_i}$, $n_ip_{k_i}\le k_i$, and $p_{k_i}=CR(n_i)$. 

For the infinite DSL tree, let $v_{\emptyset}\in\X$ denote the root node and $c_{\emptyset}\in\N$ denote the number of children of $v_{\emptyset}$. 
For any $i\in[c_{\emptyset}]$, let $v_i$ denote the $i$-th child of $v_{\emptyset}$ and $c_{i}$ denote the number of the children of $v_i$. 
Suppose that for some $k\in\N$, $v_{i_1,\dots,i_{k}}$ and $c_{i_1,\dots,i_k}$ has been defined for any $i_1\in [c_{\emptyset}],\dots,i_{k}\in [c_{i_1,\dots,i_{k-1}}]$. 
For any $i\in[c_{i_1,\dots,i_k}]$, let $v_{i_1,\dots,i_k,i}$ denote the $i$-th child of $v_{i_1,\dots,i_k}$ and $c_{i_1,\dotsi_k,i,}$ denote the number of children of $v_{i_1,\dots,i_k,i}$. 
Then, by induction, $\{c_{i_1,\dots,i_k}\}_{i_1\in[c_\emptyset],\dots,i_k\in[c_{i_1,\dots,i_{k-1}}]}$ and $\{v_{i_1,\dots,i_k}\}_{i_1\in[c_\emptyset],\dots,i_k\in[c_{i_1,\dots,i_{k-1}}]}$ are defined for all $k\ge 0$. Thus, every node in the infinite DSL tree has been denoted and the tree can be denoted with $\bt:=\{v_{i_1,\dots,i_k}:i_1\in[c_\emptyset],\dots,i_k\in[c_{i_1,\dots,i_{k-1}}],k\ge 0\}$. 
For any $k\ge 0$ and any $i_1\in[c_\emptyset],\dots,i_k\in[c_{i_1,\dots,i_{k-1}}],i_{k+1}\in[c_{i_1,\dots,i_k}]$, define $\bx_{i_1,\dots,i_k}\in\X^{k+1}$ to be the label of $v_{i_1,\dots,i_k}$ and $\by_{i_1,\dots,i_k,i_{k+1}}$ to be the label of the edge connecting $v_{i_1,\dots,i_k}$ and $v_{i_1,\dots,i_k,i_{k+1}}$.

Let $I_1$ be a random variable following the uniform distribution over $[c_{\emptyset}]$ (i.e., $\pr(I_1=i)=\frac{1}{c_{\emptyset}}$ for any $i\in[c_{\emptyset}]$). For any $k\ge 1$, suppose that $I_j$ has been defined for all $j\in[k]$. Define $I_{k+1}$ to be a random variable such that conditional on $I_1,\dots,I_k$, $I_{k+1}$ follows the uniform distribution over $c_{I_1,\dots,I_{k}}$ (i.e., $\pr(I_{k+1}=i|I_1,\dots,I_k)=\frac{1}{c_{I_1,\dots,I_k}}$ for any $i\in[c_{I_1,\dots,I_k}]$). 
Define $\bI:=(I_1,I_2,\dots)$.
Then, the support of $\bI$ is 
$$
\cI:=\{(i_1,i_2,\dots,i_k,\dots):i_1\in[c_{\emptyset}],i_2\in[c_{i_1}],\dots,i_k\in[c_{i_1,\dots,i_{k-1}}],\dots\}.
$$

For any $\bi=(i_1,i_2,\dots)\in\cI$, define the distribution $P_{\bi}$ on $\X\times\Y$ as 
\begin{align*}
P_{\bi}(\{(x^j_{i_1,\dots,i_{k-1}},y^j_{i_1,\dots,i_{k}})\})=\frac{p_k}{k}\ \ \textup{for }
j\in[k],k\in\N,
\end{align*}
where $x^j_{i_1,\dots,i_{k-1}}$ and $y^j_{i_1,\dots,i_{k}}$ denote the $j$-th element in $\bx_{i_1,\dots,i_{k-1}}$ and $\by_{i_1,\dots,i_{k}}$ respectively. 
Note that as in the proof of Theorem \ref{thm:linear_rate_lower}, the mapping $\bi\mapsto P_{\bi}$, $\bi\in\cI$ is measurable. 
By the definition of DSL tree, for any $n\in\N$, there exists $h_n\in\cH$ such that $h_n(\bx_{i_1,\dots,i_{k-1}})=\by_{i_1,\dots,i_{k}}$ for all $k\in[n]$. Hence, 
\begin{align*}
\er_{\bi}(h_n):=P_{\bi}(\{(x,y)\in\X\times\Y:h_n(x)\neq y\})\le \sum_{k>n}p_k,
\end{align*}
which, together with the fact that $\sum_{k=1}^\infty p_k=1$ and $p_k\ge0$ for all $k\in\N$, implies that $\inf_{h\in\X}\er_{\bi}(h)=0$. Thus, $P_{\bi}$ is $\cH$-realizable for any $\bi\in\cI$. 

Let $(T,J),(T_1,J_1),(T_2,J_2),\dots$ be a sequence of i.i.d. random variables, independent of $\bI$, with distribution 
\begin{align*}
\pr(T=k,J=j)=\frac{p_k}{k}\quad \textup{for } j\in[k],\ k\in\N.
\end{align*}
Define 
\begin{align*}
X=x_{I_1,\dots,I_{T-1}}^{J},\ 
Y=y_{I_1,\dots,I_{T}}^{J},\ 
X_i=x_{I_1,\dots,I_{T_i-1}}^{J_i},\ 
Y_i=y_{I_1,\dots,I_{T_i}}^{J_i}\ \textup{for } i\in\N.
\end{align*}
Then, we know that conditional on $\bI$, $(X,Y),(X_1,Y_1),(X_2,Y_2),\dots$ is a sequence of i.i.d. random variables with distribution $\pr_{\bI}$. 

For any $k\in\N$ and any $(i_1,i_2,\dots,i_{k-1})\in\N^{k-1}$ such that $i_1\in[c_{\emptyset}],i_2\in[c_{i_1}],\dots,i_{k-1}\in[c_{i_1,\dots,i_{k-2}}]$, we know that 
$C_{i_1,\dots,i_{k-1}}:=\{\by_{i_1,\dots,i_{k-1},i}:i\in [c_{i_1,\dots,i_{k-1}}]\}\subseteq\Y^{[k]}$ is a pseudo-cube of dimension $k$ by the definition of DSL trees. 

For any $n\in\N$ and $j\in[k]$, define the sequence family 
$$
\cJ_{j,k,n}:=\left\{(j_1,\dots,j_m)\in([k]\backslash\{j\})^m: j_1<j_2<\dots<j_m,m\in[\min\{k-1,n\}]\right\}.
$$
For any $i_k\in [c_{i_1,\dots,i_{k-1}}]$ and $\sfJ=(j_1,\dots,j_m)\in\cJ_{j,k,n}$ with $m\in[\min\{k-1,n\}]$, by Lemma \ref{lem:pseudo-cube-proj}, we know that
$$(C_{i_1,\dots,i_{k-1}})_{\by_{i_1,\dots,i_k},\sfJ}$$ 
is a pseudo-cube of dimension $k-m$ following the notation given in Lemma \ref{lem:pseudo-cube-proj}. 
Then, by Lemma \ref{lem:pseudo-cube-prop}, for any $y\in\Y$, we have
\begin{align} \label{eq:label_size}
|((C_{i_1,\dots,i_{k-1}})_{\by_{i_1,\dots,i_k},\sfJ})_{y}^{j'}|\le\frac{1}{2}|(C_{i_1,\dots,i_{k-1}})_{\by_{i_1,\dots,i_k},\sfJ}|,
\end{align}
where $j':=j-\max\{l\in[m]:j_l< j\}$ with the convention that $\max\emptyset:=0$. 

By the definition of $\bI$, $I_k$ follows the uniform distribution over $c_{I_1,\dots,I_{k-1}}$ conditional on $I_1,\dots,I_{k-1}$. 
Note that for any $i,i'\in[c_{i_1,\dots,i_{k-1}}]$ such that $i\neq i'$ and  $\by_{i_1,\dots,i_{k-1},i}|_\sfJ=\by_{i_1,\dots,i_{k-1},i'}|_\sfJ$, we must have 
\begin{align} \label{eq:neq-prop-J}
\by_{i_1,\dots,i_{k-1},i}|_{\sfJ'}\neq \by_{i_1,\dots,i_{k-1},i'}|_{\sfJ'},
\end{align}
where $\sfJ':=(j'_1,\dots,j'_{k-m})$ with $1\le j'_1<\dots<j'_{k-m}\le k$ and $\{j_1,\dots,j_m,j'_1,\dots,j'_{k-m}\}=[k]$. 
Therefore, conditional on $I_1,\dots,I_{k-1}$ and $\by_{I_1,\dots,I_k}|_{\sfJ}$, $\by_{I_1,\dots,I_k}|_{\sfJ'}$ distributes uniformly over the set $(C_{I_1,\dots,I_{k-1}})_{\by_{I_1,\dots,I_k},\sfJ}$. 
Then, by \eqref{eq:label_size}, we have
\begin{align*}
\pr\left(y_{I_1,\dots,I_k}^j\neq y\Big|I_1,\dots,I_{k-1},(\sfJ,\by_{I_1,\dots,I_k}|_{\sfJ})\right)\ge\frac{1}{2}.
\end{align*}
By Lemma \ref{lem:pseudo-cube-prop}, we also have
\begin{align*}
|(C_{i_1,\dots,i_{k-1}})_y^j|\le \frac{1}{2}|C_{i_1,\dots,i_{k-1}}|,
\end{align*}
which implies that
\begin{align}  \label{eq:neq-prop}
\pr\left(y_{I_1,\dots,I_k}^j\neq y\Big|I_1,\dots,I_{k-1}\right)\ge\frac{1}{2}.
\end{align}

Now, define $\wh h_n:=A(\cH,((X_1,Y_1),\dots,(X_n,Y_n)))$ and the random sequence $\fkJ:=\seq(\{J_i:T_i=k,i\in[n]\})$, where $\seq(\emptyset):=\emptyset$ and for a finite set of integers $\{a_1,\dots,a_q\}$ with $q\in \N$, $\seq(\{a_1,\dots,a_q\}):=(a_{(1)},\dots,a_{(q)})$ where $a_{(i)}$ denotes the $i$-th smallest element among $(a_1,\dots,a_q)$ for any $i\in[q]$. 
Then, by \eqref{eq:neq-prop-J} and \eqref{eq:neq-prop}, we have
\begin{align*}
&\pr\left(\wh h_n(X)\neq Y,T=k\right)
\\ \ge &
\sum_{j=1}^k
\pr\left(\wh h_n(x_{I_1,\dots,I_{k-1}}^j)\neq y_{I_1,\dots,I_k}^j, T=k, J=j, T_1,\dots,T_n\le k,(T_1,J_1),\dots,(T_n,J_n)\neq (k,j)\right)
\\ = &
\sum_{j=1}^k
\bE\Big[\indi\{T=k, J=j, T_1,\dots,T_n\le k,(T_1,J_1),\dots,(T_n,J_n)\neq (k,j)\}
\\ &\quad\quad\quad \cdot
\pr\Big(\wh h_n(x_{I_1,\dots,I_{k-1}}^j)\neq y_{I_1,\dots,I_k}^j\Big|I_1,\dots,I_{k-1},T_1,\dots,T_n,J_1,\dots,J_n,(\fkJ,\by_{I_1,\dots,I_k}|_{\fkJ})\Big)
\Big]
\\ \ge &
\sum_{j=1}^k
\bE\Big[\frac{1}{2}\indi\{T=k, J=j, T_1,\dots,T_n\le k,(T_1,J_1),\dots,(T_n,J_n)\neq (k,j)\}\Big]
\\ = &
\frac{1}{2}
\sum_{j=1}^k
\pr\left(T=k, J=j, T_1,\dots,T_n\le k,(T_1,J_1),\dots,(T_n,J_n)\neq (k,j)\right)
\\ = &
\frac{p_k}{2}
\left(1-\sum_{l> k}p_l-\frac{p_k}{k}\right)^n.
\end{align*}
Then, for any $i\ge 3$, by \citet[Lemma 5.12]{universal_learning}, we have
\begin{align*}
\pr\left(\wh h_{n_i}(X)\neq Y,T=k_i\right) \ge &
\frac{p_{k_i}}{2}
\left(1-\sum_{l>k_{i}}p_l-\frac{p_{k_i}}{k_i}\right)^{n_i}
\\ \ge &
\frac{p_{k_i}}{2}
\left(1-\frac{2}{n_i}\right)^{n_i}
\\ \ge &
\frac{CR(n_i)}{54}.
\end{align*}
Since 
\begin{align*}
\frac{1}{R(n_i)}\pr\left(\wh h_{n_i}(X)\neq Y,T=k_i|\bI\right)
\le 
\frac{1}{R(n_i)}\pr\left(T=k_i|\bI\right)
=
\frac{1}{R(n_i)}\pr\left(T=k_i\right)
=
\frac{p_{k_i}}{R(n_i)}=C\ 
\textup{a.s.},
\end{align*}
by Fatou's lemma, we have
\begin{align*}
\bE\left[
\limsup_{i\rightarrow\infty}\frac{1}{R(n_i)}\pr\left(\wh h_{n_i}(X)\neq Y,T=k_i|\bI\right)\right]
\ge 
\limsup_{i\rightarrow\infty}\frac{1}{R(n_i)}\pr\left(\wh h_{n_i}(X)\neq Y,T=k_i\right)
\ge 
\frac{C}{54}.
\end{align*}
Because
\begin{align*}
\bE[\er_{\bI}(\wh h_n)|\bI]
=
\pr(\wh h_n(X)\neq Y|\bI)
\ge 
\pr(\wh h_n(X)\neq Y,T=k|\bI) \textup{ a.s.},
\end{align*}
we have $\bE[\lim\sup_{i\rightarrow\infty}\frac{1}{R(n_i)}\bE[\er_{\bI}(\wh h_{n_i})|\bI]]\ge\frac{C}{54}>\frac{C}{55}$, which implies that there exists $\bi\in \mathcal{I}$ such that $\bE[\er_{\bi}(\wh h_{n})]\ge \frac{C}{55}R(n)$ for infinitely many $n$. 
By choosing $P=P_{\bi}$, we see that $\cH$ requires arbitrarily slow rates. 

Since $\X$ is Polish and $\Y$ is countable, there exists a learning algorithm with $\bE[\er(\wh h_n)]\rightarrow0$ for all realizable distributions $P$ \citep{10.1214/20-AOS2029}. It follows that $\cH$ is learnable but requires arbitrarily slow rates.
\end{proof}

\section{Proof of Theorem \ref{thm:NL-counter}} \label{sec:proof-NL-counter}
In this section, we provide the complete proof of Theorem \ref{thm:NL-counter} below.
\begin{proof}
According to the proof of \citet[Theorem 2]{brukhim2022characterization}, for any $d\in\N$, there exists a $d$-dimensional pseudo-cube $B_d\subseteq Y_d^{X_d}$ for some spaces $X_d$ and $Y_d$ with $|X_d|=d$ and $|Y_d|<\infty$. 

Therefore, for $B_1\subseteq Y_1^{X_1}$, we can pick $c_1:=|B_1|$ feature spaces $X_{1,1},\dots,X_{1,d_1}$ of size 2, label spaces $Y_{1,1},\dots,Y_{1,d_1}$, and pseudo-cubes $B_{1,1}\subseteq Y_{1,1}^{X_{1,1}},\dots,B_{1,c_1}\subseteq Y_{1,c_1}^{X_{1,c_1}}$ of dimension $2$ such that $X_1,X_{1,1},\dots,X_{1,c_1}$ are pairwise disjoint and $Y_1,Y_{1,1},\dots,Y_{1,c_1}$ are also pairwise disjoint. Define $c_{1,i}:=|B_{1,i}|$ for any $i\in [c_1]$. 
Now, suppose that for some $d\in\N$, $c_{\bi_{k}}$, $X_{\bi_{k}}$,  $Y_{\bi_{k}}$ and $B_{\bi_{k}}\subseteq Y_{\bi_{k}}^{X_{\bi_{k}}}$ have been defined such that $|X_{\bi_{k}}|=k$ and $B_{\bi_{k}}$ is a pseudo-cube of dimension $k$ for any $k\in[d]$, $\bi_k\in\cI_k:=\{(i_1,\dots,i_k):i_1\in[1],i_2\in [c_{i_1}],\dots,i_k\in[c_{i_1,\dots,i_{k-1}}]\}$, $\{X_{\bi}:\bi\in\cI_k, k\in[d]\}$ are pairwise disjoint, and $\{Y_{i_1,\dots,i_k}:\bi\in\cI_k, k\in[d]\}$ are also pairwise disjoint. 
Then, for any $\bi_d\in\cI_d$, pick $c_{\bi_d}$ feature spaces $X_{\bi_d,1},\dots,X_{\bi_d,c_{\bi_d}}$ of size $d+1$, label spaces $Y_{\bi_d,1},\dots,Y_{\bi_d,c_{\bi_d}}$, and pseudo-cubes $B_{\bi_d,1}\subseteq Y_{\bi_d,1}^{X_{\bi_d,1}},\dots,B_{\bi_d,c_{\bi_d}}\subseteq Y_{\bi_d,c_{\bi_d}}^{X_{\bi_d,c_{\bi_d}}}$ of dimension $d+1$ such that 
$\{X_{\bi}:\bi\in\cI_k,k\in[d+1]\}$ are pairwise disjoint and $\{Y_{\bi}:\bi\in\cI_k,k\in[d+1]\}$ are also pairwise disjoint where $\cI_{d+1}:=\{(\bi_d,i):i\in [c_{\bi}],\bi_d\in\cI_d\}$. 
Then, we define $c_{i_1,\dots,i_{d},i}:=|B_{i_1,\dots,i_{d},i}|$ for any $i\in [c_{i_1,\dots,i_{d}}]$. 

By induction, for any $k\in\N$ and any $\bi_k\in \cI_{k}$, $c_{\bi_k}$, $X_{\bi_k}$, $Y_{\bi_k}$, and $B_{\bi_k}\subseteq Y_{\bi_k}^{X_{\bi_k}}$ have been defined such that $|X_{\bi_k}|=k$, $B_{\bi_k}$ is a pseudo-cube of dimension $k$, $\{X_{\bi}:\bi\in\cI_k,k\in\N\}$ are pairwise disjoint, and $\{Y_{\bi}:\bi\in\cI_k,k\in\N\}$ are also pairwise disjoint. 

Now, we define $\cI:=\cup_{k\in\N}\cI_k$, $\X:=\cup_{\bi\in\cI}X_{\bi}$, and $\Y:=\cup_{\bi\in\cI}Y_{\bi}\cup\{\star\}$ where $\star\notin\cup_{\bi\in\cI}Y_{\bi}$ is a new label. Note that $\X$ and $\Y$ are countable.
Now, for any $d\in\N$ and $\bi=(i_1,\dots,i_d)\in\cI_d$, since $|B_{\bi}|=c_{\bi}\in\N$, we use $h_{\bi}^{(i)}$ to denote the $i$-th hypothesis in $B_{\bi}$ for any $i\in[c_{\bi}]$ and extend the domain of $h_{\bi}^{(i)}$ to $\X$ by defining $h_{\bi}^{(i)}|_{X_{i_1,\dots,i_k}}:=h_{i_1,\dots,i_k}^{(i_{k+1})}|_{X_{i_1,\dots,i_k}}$ for any $k\in[d-1]$ and $h_{\bi}^{(i)}(x):=\star$ for any $x\in\X\backslash(\cup_{k\in[d]}X_{i_1,\dots,i_k})$. Letting $H_{\bi}$ denote the extended hypotheses in $B_{\bi}$, we define the following hypothesis class
\begin{align*}
\cH:=\cup_{\bi\in\cI}H_{\bi}.
\end{align*}

By setting $\{X_{\bi}:\bi\in\cI\}$ to be the set of nodes and $\{H_{\bi}:\bi\in\cI\}$ to be the set of edges, we obtain an infinite DSL tree of $\cH$. 
To prove that $\cH$ has no NL tree of depth 2, it suffices to show that the Natarajan dimension of $\cH$ is 1. 
For any $k_1, k_2\in \N$ with $k_1\le k_2$, $\bi_1\in\cI_{k_1},\bi_2\in\cI_{k_2}$, and $x_1\in X_{\bi_1}$ and $x_2\in X_{\bi_2}$ with $x_1\neq x_2$, if $\bi_2|_{1:k_1}\neq \bi_1$, then, $\star\in\{h(x_1),h(x_2)\}$ for any $h\in\cH$, which implies that $\{x_1,x_2\}$ is not N-shattered (see \citealt[Definition 4]{brukhim2022characterization} for the definition of ``N-shattered'') by $\cH$.
If $k_1<k_2$ and $\bi_2|_{1:k_1}= \bi_1$, then, for any $h_1,h_2\in\cH$, in order to have $h_1(x_1)\neq h_2(x_1)$ and $h_1(x_2)\neq h_2(x_2)$, we must have either $\{h_1(x_1),h_2(x_1)\}\times\{h_1(x_2),h_2(x_2)\}=\{\star,y_1\}\times\{\star,y_2\}$ 
or $\{h_1(x_1),h_2(x_1)\}\times\{h_1(x_2),h_2(x_2)\}=\{y_1',y_1''\}\times\{\star,y_2'\}$ for some $y_1,y_1',y_1'',y_2,y_2'\in\cH\backslash\{\star\}$ with $y_1'\neq y_1''$. 
For $\{h_1(x_1),h_2(x_1)\}\times\{h_1(x_2),h_2(x_2)\}=\{\star,y_1\}\times\{\star,y_2\}$, there is no $h\in\cH$ such that $(h(x_1),h(x_2))=(\star,y_2)$ by our construction. 
For $\{h_1(x_1),h_2(x_1)\}\times\{h_1(x_2),h_2(x_2)\}=\{y_1',y_1''\}\times\{\star,y_2'\}$, WOLG, we may assume that $(h_1(x_1),h_1(x_2))=(y_1',y_2')$. Then, there is no $h\in\cH$ such that $(h(x_1),h(x_2))=(y_1'',y_2')$ because any $h$ such that $h(x_2)=h_1(x_2)\neq \star$ must have $h(x_1)=h_1(x_1)=y_1'$.
Thus, $\{x_1,x_2\}$ is not N-shattered by $\cH$.
Finally, if $\bi_1=\bi_2=\bi$, for any $h_1\in\cH\backslash\wb H_{\bi}$ where $\wb H_{\bi}:=\{h\in H_{\bi'}:\bi'\in\cI_k,\bi'|_{1:k_1}=\bi,k\ge k_1\}$, we have $h_1(x_1)=h_1(x_2)=\star$. However, there is no $h\in\cH$ such that $(h(x_1),h(x_2))=(\star,y_2)$ for $y_2\neq \star$. On ther other hand, for any $h_1,h_2\in \wb H_{\bi}$, we have $\wb H_{\bi}|_{(x_1,x_2)}=B_{\bi}|_{(x_1,x_2)}$. Since the Natarajan dimension of $B_{\bi}$ is 1, $(x_1,x_2)$ is also not N-shattered by $\wb H_{\bi'}$. Thus, $(x_1,x_2)$ is also not N-shattered by $\cH$.
In conclusion, any $(x_1,x_2)\in\X^2$ is not N-shattered by $\cH$, the Natarajan dimension of $\cH$ is 1, and $\cH$ has no NL tree of depth 2. 

\end{proof}

\section{Proof of Theorem \ref{thm:NL-GL}} \label{sec:thmNLGL}
In this section, we prove Theorem \ref{thm:NL-GL}. We first prove the following general lemma. 
\begin{lemma} \label{lem:GL}
Suppose that $\cH\subseteq\Y^{\X}$ with $\Y:=[K]$ for some $K\in\N\backslash\{1\}$ has an infinite GL tree $\ms T=\cup_{n=0}^{\infty}\{(\bx_{\bu},\bs_{\bu})\in\X^{n+1}\times\Y^{n+1}:\bu\in\prod_{l=1}^n\{0,1\}^l\}$ with its associated hypothesis set $\{h_{\bu}\in\cH:\bu\in\prod_{l=1}^n\{0,1\}^l,n\in\N\}$. 
For any $d\in\N_0$,  $\bfeta\in\prod_{l=1}^d\{0,1\}^l$, and $\bw\in\{0,1\}^{d+1}\backslash\{0\}^{d+1}$, there exist a sequence $\by_{\bfeta,\bw}\in\Y^{d+1}$ and an infinite GL tree $\cup_{n=0}^{\infty}\{(\wt\bx_{\bu},\wt\bs_{\bu})\in\X^{n+1}\times\Y^{n+1}:\bu\in\prod_{l=1}^n\{0,1\}^l\}$ with its associated hypothesis set $\{\wt h_{\bu}\in\cH:\bu\in\prod_{l=1}^n\{0,1\}^l,n\in\N\}$
such that $(\wt x_{\emptyset},\wt s_{\emptyset})=(x_{\emptyset},s_{\emptyset})$,
$(\wt\bx_{\bu},\wt\bs_{\bu},\wt h_{\bu})=(\bx_{\bu},\bs_{\bu},h_{\bu})$ for all $\bu\in\big(\cup_{n\in\N}\prod_{l=1}^n\{0,1\}^l\big)\backslash\big(\cup_{n=d+1}^\infty(\{\bfeta,\bw\}\times\prod_{l=d+2}^n\{0,1\}^l)\big)$,  
$\{\wt h_{\bu}:\bu\in\{\bfeta,\bw\}\times\prod_{l=d+2}^n\{0,1\}^l,n\ge d+1\}\subseteq \{h_{\bu}:\bu\in\{\bfeta,\bw\}\times\prod_{l=d+2}^n\{0,1\}^l,n\ge d+1\}$, and
for all $0\le i\le d$ and $\bu\in\cup_{n=d+1}^\infty\big(\{\bfeta,\bw\}\times\prod_{l=d+2}^n\{0,1\}^l\big)$, we have $\wt h_{\bu}(\wt x_{\bfeta}^i)=y_{\bfeta,\bw}^i=s_{\bfeta}^i$ if $w^i=0$ and $\wt h_{\bu}(\wt x_{\bfeta}^i)=y_{\bfeta,\bw}^i\neq s_{\bfeta}^i$ if $w^i=1$.
\end{lemma}
\begin{proof}
For any $\bu\in\cup_{n=d+1}^\infty(\{\bfeta,\bw\}\times\prod_{l=d+2}^n\{0,1\}^l)$, we color $v_{\bu}:=(\bx_{\bu},\bs_{\bu},h_\bu)$ with 
$$
(h_{\bu}(x_{\bfeta}^0),\dots,h_{\bu}(x_{\bfeta}^{d}))\in\Y^{d+1}.
$$ 
Since $|\Y|=K<\infty$, by the Milliken's tree theorem \citep{MILLIKEN1979215}, for the colored infinite tree $\ms T_{\bfeta,\bw}:=\{v_{\bu}:\bu\in\{\bfeta,\bw\}\times\prod_{l=d+2}^n\{0,1\}^l,n\ge d+1\}$,
there exists some color $\by_{\bfeta,\bw}\in\Y^{d+1}$ and a strongly embedded infinite subtree $\breve{\ms T}_{\bfeta,\bw}$
of $\ms T_{\bfeta,\bw}$ such that all the nodes in $\breve{\ms T}_{\bfeta,\bw}$ have the same color $\by_{\bfeta,\bw}$.
Since $\breve{\ms T}_{\bfeta,\bw}$ is a strongly embedded subtree of $\ms T_{\bfeta,\bw}$, there exists some sequence $(n_l)_{l\in\N_0}\in\N^{\N_0}$ such that $n_{l+1}>n_{l}\ge d+1$ for any $l\in\N_0$,
$$
\breve{\ms T}_{\bfeta,\bw}=\cup_{\tau=0}^\infty\Big\{({\breve\bx}_{\bb},{\breve\bs}_{\bb},{\breve h}_{\bb})\in\X^{n_\tau+1}\times\Y^{n_{\tau}+1}\times\cH:\bb\in\prod_{l=0}^{\tau-1}\{0,1\}^{n_{l}+1}\Big\},
$$
and $({\breve\bx}_{\bb},{\breve\bs}_{\bb},{\breve h}_{\bb})$ is a node in level $n_\tau$ of $\ms T$ for all $\bb\in\prod_{l=0}^{\tau-1}\{0,1\}^{n_{l}+1}$ and $\tau\in\N_0$. 
For any $t\in\N$ with $t\ge d+1$ and $\bu=(u_1^0,(u_2^0,u_2^1),\dots,(u_{t}^0,\dots,u_t^{t-1}))\in\{\bfeta,\bw\}\times\prod_{l=d+2}^t\{0,1\}^l$, define 
\begin{align*}
\bbeta(\bu):=&((\beta(\bu)_{1}^0,\dots,\beta(\bu)_{1}^{n_0}),(\beta(\bu)_{2}^0,\dots,\beta(\bu)_{2}^{n_1}),\dots,(\beta(\bu)_{t-d-1}^0,\dots,\beta(\bu)_{t-d-1}^{n_{t-d-2}}))
\\ \in &
\prod_{l=0}^{t-d-2}\{0,1\}^{n_{l}+1}
\end{align*}
by
\begin{align*}
\beta(\bu)_l^{i}:=
\begin{cases}
&u_{l+d+1}^i, \quad \textup{if } 0\le i\le l+d,\\
&0, \quad \textup{if } l+d+1\le i\le n_{l-1}
\end{cases}
\end{align*}
for all $0\le i\le n_{l-1}$ and $1\le l\le t-d-1$ and
\begin{align*}
\wt v_{\bu}:=({\wt\bx}_{\bu},{\wt\bs}_{\bu},
{\wt h}_{\bu})
\end{align*}
with
\begin{align*}
\wt\bx_{\bu}:=({\breve\bx}_{\bbeta(\bu)}^{0},\dots,{\breve\bs}_{\bbeta(\bu)}^{n})\in\X^{n+1},
\ 
\wt\bs_{\bu}:=({\breve\bx}_{\bbeta(\bu)}^{0},\dots,{\breve\bs}_{\bbeta(\bu)}^{n})\in\Y^{n+1},
\textup{and }
\wt h_{\bu}:=\breve h_{\bbeta(\bu)}\in\cH.
\end{align*}
Define $\wt x_{\emptyset}:=x_{\emptyset}$, $\wt s_{\emptyset}:=s_{\emptyset}$, and $\wt\bx_{\bu}:=\bx_{\bu}, \wt\bs_{\bu}:=\bs_{\bu}, \wt h_{\bu}:=h_{\bu}$ for any 
$$
\bu\in\Big(\cup_{n\in\N}\prod_{l=1}^n\{0,1\}^l\Big)\backslash\Big(\cup_{n=d+1}^\infty\Big(\{\bfeta,\bw\}\times\prod_{l=d+2}^n\{0,1\}^l\Big)\Big).
$$
Then, we obtain the following infinite tree
\begin{align*}
\wt{\ms T}:=
\cup_{n=0}^\infty\Big\{(\wt\bx_{\bu},\wt\bs_{\bu}):\bu\in\prod_{l=1}^n\{0,1\}^l\Big\}.
\end{align*}
Since $\ms T$ is an infinite GL tree and $\breve{\ms T}_{\bfeta,\bw}$ is a strongly embedded infinite subtree of $\ms T_{\bfeta,\bw}$, we have that $\wt{\ms T}$ is an infinite GL tree with the associated hypothesis set $\{h_{\bu}\in\cH:\bu\in\prod_{l=1}^n\{0,1\}^l,n\in\N_0\}$. 
Since all the nodes in $\breve{\ms T}_{\bfeta,\bw}$ have the smae color $\by_{\bfeta,\bw}$, by the construction of coloring and $\wt{\ms T}$, for all $0\le i\le d$ and $\bu\in\cup_{n=d+1}^\infty\big(\{\bfeta,\bw\}\times\prod_{l=d+2}^n\{0,1\}^l\big)$, we have $\wt h_{\bu}(\wt x_{\bfeta}^i)=y_{\bfeta,\bw}^i=s_{\bfeta}^i$ if $w^i=0$ and $\wt h_{\bu}(\wt x_{\bfeta}^i)=y_{\bfeta,\bw}^i\neq s_{\bfeta}^i$ if $w^i=1$.
Finally, we have $(\wt x_{\emptyset},\wt s_{\emptyset})=(x_{\emptyset},s_{\emptyset})$ and $(\bx_{\bu},\bs_{\bu},h_{\bu})=(\wt\bx_{\bu},\wt\bs_{\bu},\wt h_{\bu})$ for all $\bu\in\big(\cup_{n\in\N}\prod_{l=1}^n\{0,1\}^l\big)\backslash\big(\cup_{n=d+1}^\infty(\{\bfeta,\bw\}\times\prod_{l=d+2}^n\{0,1\}^l)\big)$ by our definition. 

\end{proof}

Now, we are ready to carry out the proof of Theorem \ref{thm:NL-GL}.
\begin{proof}[of Theorem \ref{thm:NL-GL}]
For any NL tree $\cup_{n=0}^{d-1}\{(\bx_{\bu},\bs^{(0)}_{\bu},\bs^{(1)}_{\bu})\in\X^{n+1}\times\Y^{n+1}\times\Y^{n+1}:\bu\in\prod_{l=1}^n\{0,1\}^l\}$ of $\cH$ of depth $1\le d\le \infty$, $\cup_{n=0}^{d-1}\{(\bx_{\bu},\bs^{(0)}_{\bu})\in\X^{n+1}\times\Y^{n+1}:\bu\in\prod_{l=1}^n\{0,1\}^l\}$ is a GL tree of $\cH$ of the same depth $d$. Thus, an infinite NL tree of $\cH$ implies an infinite GL tree of $\cH$. 

Now, suppose that $\cup_{n=0}^{\infty}\{(\bx_{\bu},\bs_{\bu})\in\X^{n+1}\times\Y^{n+1}:\bu\in\prod_{l=1}^n\{0,1\}^l\}$ is an infinite GL tree of $\cH$. 
For any $n\in\N$ and $\bu=(u_1^0,(u_2^0,u_2^1),\dots,(u_n^0,\dots,u_n^{n-1}))\in\prod_{l=1}^{n}\{0,1\}^l$, there exists some $h_{\bu}\in\cH$ such that $h_{\bu}(x_{\bu_{\le l}}^i)=s_{\bu_{\le l}}^i$ if $u_{l+1}^i=0$ and $h_{\bu}(x_{\bu_{\le l}}^i)\neq s_{\bu_{\le l}}^i$ otherwise for all $0\le i\le l$ and $0\le l<n$, where
\begin{align*}
\bu_{\le l}:=(u_1^0,(u_2^0,u_2^1),\dots,(u_l^{0},\dots,u_l^{l-1})),\ 
x_{\bu_{\le l}}:=(x_{\bu_{\le 1}}^0,\dots,x_{\bu_{\le 1}}^{l}).
\end{align*}
Then, we define $\ms T_G=\{v_{\emptyset}=(x_{\emptyset},s_{\emptyset})\}\cup\{v_{\bu}=(\bx_{\bu},\bs_{\bu},h_{\bu}):\bu\in\prod_{l=1}^n\{0,1\}^l,1\le n< \infty\}$ which is the infinite GL tree with the associated hypotheses. 
Next, we use induction to show that $\cH$ has an infinite NL tree.

Applying Lemma \ref{lem:GL} to $\ms T_G$ for $d=0$, $\bfeta=\emptyset$, and $\bw=1$, 
we obtain a label $\wb s_{\emptyset}^0\in\Y\backslash\{s_{\emptyset}^0\}$ and an infinite GL tree with the associated hypotheses $\breve{\ms T}_G=\{\breve v_{\emptyset}=(\breve x_{\emptyset},\breve s_{\emptyset})\}\cup\big(\cup_{n=1}^\infty\{\breve v_{\bu}=(\breve\bx_{\bu},\breve\bs_{\bu},\breve h_{\bu}):\bu\in\prod_{l=1}^n\{0,1\}^l\}\big)$ such that  
for all $\bu\in\cup_{n=1}^\infty\big(\{1\}\times\prod_{l=2}^n\{0,1\}^l\big)$, we have $\breve h_{\bu}(\breve x_{\emptyset})=\wb s_{\emptyset}^0$.
Then, we replace $\ms T_G$ with $\breve{\ms T}_{G}$. With abuse of notation, we still use $\ms T_G$ to denote $\breve{\ms T}_{G}$, use  $v_{\emptyset}=(x_{\emptyset},s_{\emptyset})$ to denote $\breve v_{\emptyset}=(\breve x_{\emptyset},\breve s_{\emptyset})$, and use $v_{\bu}=(\bx_{\bu},\bs_{\bu},h_\bu)$ to denote $\breve v_{\bu}=(\breve\bx_{\bu},\breve\bs_{\bu},{\breve h}_\bu)$ for all $\bu\in\cup_{n=1}^{\infty}(\prod_{l=1}^n\{0,1\}^l)$.

Suppose that for some $d\in\N$, there exists a set $\cup_{n=0}^{d-1}\{\wb\bs_{\bu}\in\Y^{n+1}:\bu\in\prod_{l=1}^{n}\{0,1\}^l\}$ 
and an infinite GL tree with the associated hypotheses $\ms T_G=\{v_{\emptyset}=(x_{\emptyset},s_{\emptyset})\}\cup\big(\cup_{n=1}^{\infty}\{v_{\bu}=(\bx_{\bu},\bs_{\bu},h_{\bu}):\bu\in\prod_{l=1}^{n}\{0,1\}^l\}\big)$ of $\cH$ 
such that $\cup_{n=0}^{d-1}\{(\bx_{\bu},\bs_{\bu},\wb\bs_{\bu}):\bu\in\prod_{l=1}^{n}\{0,1\}^l\}$ is a NL of $\cH$ of depth $d$ and for any $\bu\in\cup_{n=d}^\infty(\prod_{l=1}^n\{0,1\}^l)$, we have $h_{\bu}(x_{\bu_{\le l}}^i)=s_{\bu_{\le l}}^i$ if $u_{l+1}^i=0$ and $h_{\bu}(x_{\bu_{\le l}}^i)=\wb s_{\bu_{\le l}}^i$ otherwise for all $0\le i\le l$ and $0\le l<d$. 
Define $r:=\lceil5\log_2K\rceil\in\N$. 
For any $\bkappa\in\prod_{l=1}^d\{0,1\}^l$, consider $\bfeta\in\{\bkappa\}\times\prod_{l=d+1}^{r(d+1)-1}\{0\}^l$. 
Applying Lemma \ref{lem:GL} for $\bfeta$ and each $\bw\in\{0,1\}^{r(d+1)}\backslash\{0\}^{r(d+1)}$ iteratively and defining $\by_{\bfeta,\bw}:=\bs_{\bfeta}$ for $\bw\in\{0\}^{r(d+1)}$, 
we obtain a class $H_{\bfeta}=\{\by_{\bfeta,\bw}\in\Y^{r(d+1)}:\bw\in\{0,1\}^{r(d+1)}\}$ and an infinite GL tree with the associated hypotheses $\wt{\ms T}_{G,\bfeta}=\{\wt v_{\emptyset}=(x_{\emptyset},s_{\emptyset})\}\cup\big(\cup_{n=1}^\infty\{\wt v_{\bu}=(\wt\bx_{\bu},\wt\bs_{\bu},\wt h_{\bu}):\bu\in\prod_{l=1}^n\{0,1\}^l\}\big)$ such that $\wt{\ms T}_{G,\bfeta}$ satisfies the induction hypothesis for $d$ and $\cup_{n=0}^{d-1}\{\wb\bs_{\bu}\in\Y^{n+1}:\bu\in\prod_{l=1}^{n}\{0,1\}^l\}$ and 
for all $\bw\in\{0,1\}^{r(d+1)}\backslash\{0\}^{r(d+1)}$, $0\le i\le r(d+1)-1$, and $\bu\in\cup_{n=r(d+1)}^\infty\big(\{\bfeta,\bw\}\times\prod_{l=r(d+1)+1}^n\{0,1\}^l\big)$, we have $\wt h_{\bu}(\wt x_{\bfeta}^i)=y_{\bfeta,\bw}^i=s_{\bfeta}^i$ if $w^i=0$ and $\wt h_{\bu}(\wt x_{\bfeta}^i)=y_{\bfeta,\bw}^i\neq s_{\bfeta}^i$ if $w^i=1$.
Then, we replace $\ms T_G$ with $\wt{\ms T}_{G,\bfeta}$. With abuse of notation, we still use $\ms T_G$ to denote $\wt{\ms T}_{G,\bfeta}$, use  $v_{\emptyset}=(x_{\emptyset},s_{\emptyset})$ to denote $\wt v_{\emptyset}=(\wt x_{\emptyset},\wt s_{\emptyset})$, and use $v_{\bu}=(\bx_{\bu},\bs_{\bu},h_\bu)$ to denote $\wt v_{\bu}=(\wt\bx_{\bu},\wt\bs_{\bu},{\wt h}_\bu)$ for all $\bu\in\cup_{n=1}^{\infty}(\prod_{l=1}^n\{0,1\}^l)$. 

Since we have shown that $y_{\bfeta,\bw}^i\neq s_{\bfeta}^i$ if $w^i=1$ and $y_{\bfeta,\bw}^i=s_{\bfeta}^i$ if $w^i=0$ for any $\bw\in\{0,1\}^{r(d+1)}$ and $0\le i\le r(d+1)-1$, we have $\dim_{G}(H_{\bfeta})=r(d+1)$ and by \citet{BENDAVID199574}, $\dim_{N}(H_{\bfeta})>\frac{\dim_G(H_{\bfeta})}{5\log_2K}\ge d+1$. 
Thus, there exists a subset $\{i_0,\dots,i_{d}\}\subseteq [r(d+1)]$ and two sequences $\bff_0=(f^0_0,\dots,f^{d}_0),\ \bff_1=(f^0_1,\dots,f^{d}_1)\in\Y^{d+1}$ such that $f_0^i\neq f_1^i$ for all $0\le i\le d$ and 
$H_{\bfeta}|_{(i_0,\dots,i_{d})}\supseteq\{f_0^0,f_1^0\}\times\dots\times\{f_0^d,f_1^d\}$. 
Then, we must have either $\bff_0=\bs_{\bfeta}|_{(i_0,\dots,i_d)}$ or $\bff_1=\bs_{\bfeta}|_{(i_0,\dots,i_d)}$. 
If $\bff_0=\bs_{\bfeta}|_{(i_0,\dots,i_d)}$, we define $\wb\bs_{\bkappa}:=\bff_1$ and $\bs_{\bkappa}:=\bff_0$.
If $\bff_1=\bs_{\bfeta}|_{(i_0,\dots,i_d)}$, we define
$\wb\bs_{\bkappa}:=\bff_0$ and $\bs_{\bkappa}:=\bff_1$. Then, we have ${\wb s}_{\bkappa}^i\neq s_{\bkappa}^i$ for all $0\le i\le d$ and $H_{\bfeta}|_{(i_0,\dots,i_{d})}\supseteq\{s_{\bkappa}^0,\wb s_{\bkappa}^0\}\times\dots\times\{s_{\bkappa}^d,\wb s_{\bkappa}^d\}$. 

Define the set $W_{\bfeta}:=\{\bw\in\{0,1\}^{r(d+1)}:y_{\bfeta,\bw}^{i_j}={\wb s}_{\bkappa}^{i_j}\textup{ for all }j\in[d+1]\textup{ s.t. }w^{i_j}=1\}$ and $H'_{\bfeta}:=\{\by_{\bfeta,\bw}|_{(i_0,\dots,i_d)}:\bw\in W_{\bfeta}\}$. 
We have $H_{\bfeta}'=\{s_{\bkappa}^0,\wb s_{\bkappa}^0\}\times\dots\times\{s_{\bkappa}^d,\wb s_{\bkappa}^d\}$ and $|H_{\bfeta}'|=2^{d+1}$. 
For any $\bg=(g_1,\dots,g_{d+1})\in\{0,1\}^{d+1}$, define the set 
$$
W_{\bfeta,\bg}:=W_{\bfeta}\cap\{\bw\in\{0,1\}^{r(d+1)}:w^{i_j}=g^{j}
\textup{ for all }j\in[d+1]\}.
$$
Since $H_{\bfeta}|_{(i_1,\dots,i_d)}\supseteq\{s_{\bkappa}^0,\wb s_{\bkappa}^0\}\times\dots\times\{s_{\bkappa}^d,\wb s_{\bkappa}^d\}$, we have $W_{\bfeta,\bg}\neq \emptyset$. Then, we pick a sequence $\bw_{\bg}\in W_{\bfeta,\bg}$ for any $\bg\in\{0,1\}^{d+1}$. 
For any $n\in\N$ with $n\ge d+1$ and any $\bu=(\bu_{d+1},\dots,\bu_{n})\in\prod_{l=d+1}^{n}\{0,1\}^l$, we define
\begin{align*}
\balpha(\bu):=\big(&(\bw_{\bu_{d+1}},(\alpha(\bu)_{r(d+1)+1}^0,\dots,\alpha(\bu)_{r(d+1)+1}^{r(d+1)}),\dots,
\\&
(\alpha(\bu)_{n+(r-1)(d+1)}^0,\dots,\alpha(\bu)_{n+(r-1)(d+1)}^{n+(r-1)(d+1)-1})
\big)
\\ \in \quad &
\prod_{l=r(d+1)}^{n+(r-1)(d+1)}\{0,1\}^l
\end{align*}
with
\begin{align*}
\alpha(\bu)_{l}^{i}:=
\begin{cases}
u_{l-(r-1)(d+1)}^{i},\quad\textup{if }
0\le i\le l-(r-1)(d+1)-1,\\
0,\quad\textup{if }
l-(r-1)(d+1) \le i\le l
\end{cases}
\end{align*}
for any $r(d+1)+1\le l\le n+(r-1)(d+1)$ and $0\le i\le l-1$. 

Next, for any $n\in\N$ with $n\ge d$ and any $\bu\in \prod_{l=d+1}^n\{0,1\}^l$, define
\begin{align*}
\wt v_{\bkappa,\bu}
:=
\begin{cases}
((\bx_{\bfeta}^{i_0},\dots,\bx_{\bfeta}^{i_d}),(\bs_{\bfeta}^{i_0},\dots,\bs_{\bfeta}^{i_d}),h_{\bfeta})
\quad \textup{if }n=d,\\
((\bx_{\bfeta,\balpha(\bu)}^{0},\dots,\bx_{\bfeta,\balpha(\bu)}^{n}),(\bs_{\bfeta,\balpha(\bu)}^{0},\dots,\bs_{\bfeta,\balpha(\bu)}^{n}),
h_{\bfeta,\balpha(\bu)}),
\quad \textup{if }n\ge d+1.
\end{cases}
\end{align*}
Then, we obtain the following tree
\begin{align*}
\wt{\ms T}_{G,\bkappa}:=\cup_{n=d}^\infty\Big\{
\wt v_{\bkappa,\bu}:
\bu\in\prod_{l=d+1}^n\{0,1\}^d
\Big\}=
\Big\{
\wt v_{\bu}:
\bu\in\cup_{n=d}^\infty\Big(\{\bkappa\}\times\prod_{l=d+1}^n\{0,1\}^d\Big) 
\Big\}.
\end{align*}
We replace $\ms T_{G,\bkappa}$ with $\wt{\ms T}_{G,\bkappa}$ in $\ms T_G$ by replacing $v_{\bu}=(\bx_{\bu},\bs_\bu,h_\bu)$ with $\wt v_{\bu}$ and still use $v_{\bu}=(\bx_{\bu},\bs_{\bu},h_{\bu})$ to denote $\wt v_{\bu}$ in $\ms T_{G}$ after the replacement for all $\bu\in\cup_{n=d}^\infty(\{\bkappa\}\times\prod_{l=d+1}^n\{0,1\}^l)$. 
Now, we have $h_{\bu}(x_{\bkappa}^i)= s_{\bkappa}^i$ if $u_{d+1}^{i}=0$ and $h_{\bu}(x_{\bkappa}^i)=\wb s_{\bkappa}^i$ if $u_{d+1}^{i}=1$ for all $0\le i\le d$ and $\bu\in\cup_{n=d}^\infty(\{\bkappa\}\times\prod_{l=d+1}^n\{0,1\}^l)$ and $\ms T_G$ is still an infinite GL tree with the associated hypotheses after the replacement. 

After the above procedure for all $\bkappa\in\prod_{l=1}^d\{0,1\}^l$, we obtain a set $\cup_{n=0}^d\{\wb\bs_{\bu}\in\Y^{n+1}:\bu\in\prod_{l=1}^n\{0,1\}^l\}$ and an infinite GL tree with the associated hypotheses  $\{(x_{\emptyset},s_{\emptyset})\}\cup\big(\cup_{n=1}^{\infty}\{(\bx_{\bu},\bs_{\bu},h_{\bu}):\bu\in\prod_{l=1}^{n}\{0,1\}^l\}\big)$ of $\cH$ such that 
$\cup_{n=0}^{d}\{(\bx_{\bu},\bs_{\bu},\wb\bs_{\bu}):\bu\in\prod_{l=1}^{n}\{0,1\}^l\}$ is a NL of $\cH$ of depth $d+1$ and for any $\bu\in\cup_{n=d+1}^\infty(\prod_{l=1}^n\{0,1\}^l)$, we have $h_{\bu}(x_{\bu_{\le l}}^i)=s_{\bu_{\le l}}^i$ if $u_{l+1}^i=0$ and $h_{\bu}(x_{\bu_{\le l}}^i)=\wb s_{\bu_{\le l}}^i$ otherwise for all $0\le i\le l$ and $0\le l\le d$. 
Thus, the induction hypothesis has been shown for $d+1$. 

By induction, there exists an infinite set $\cup_{n=0}^\infty\{\wb\bs_{\bu}\in\Y^{n+1}:\bu\in\prod_{l=1}^n\{0,1\}^l\}$ and an infinite GL tree $\cup_{n=0}^{\infty}\{(\bx_{\bu},\bs_{\bu}):\bu\in\prod_{l=1}^{n}\{0,1\}^l\}$ of $\cH$ such that 
$\cup_{n=0}^{\infty}\{(\bx_{\bu},\bs_{\bu},\wb\bs_{\bu}):\bu\in\prod_{l=1}^{n}\{0,1\}^l\}$ is an infinite NL of $\cH$. 
It follows that an infinite GL tree of $\cH$ implies an infinite NL tree of $\cH$. 

Finally, we can conclude that $\cH$ has an infinite NL tree if and only if it has an infinite GL tree.

\end{proof}

\section{Proof of Proposition \ref{prop:hn_eg}} \label{sec:claim_eg}

\begin{proof}
For any $n\in\N$ and $\bx\in\X$, if $\wh h_n(\bx)=k$ for some $k\in[K]\backslash\{1\}$, there exists some $\bz'\in[0,\infty)^d$, $1\le i_1<\dots<i_t\le n$ and $(\alpha_1,\dots\alpha_t)\in [0,1]^t$ such that $\sum_{\tau=1}^t\alpha_\tau=1$ and
\begin{align*}
\bx=\bz'+\sum_{\tau=1}^t\alpha_\tau\bx_{i_\tau}.
\end{align*}
Then, for any $k'<k$, we have
\begin{align*}
\bw_{k'}\cdot\bx-b_{k'}=
\bw_{k'}\cdot\bz'+\sum_{\tau=1}^t\alpha_\tau(\bw_{k'}\cdot\bx_{i_\tau}-b_{k'})
\le 
\bw_{k}\cdot\bz'+\sum_{\tau=1}^t\alpha_\tau(\bw_{k}\cdot\bx_{i_\tau}-b_{k}),
\end{align*}
which implies that $h_{n+1}(\bx)\ge k=\wh h_n(\bx)$.
Then, for any $n\in\N$ such that
$\wh h_{n}(\bx_{n+1})\neq y_{n+1}=h_{n+1}(\bx_{n+1})$, we must have $\wh h_n(\bx_{n+1})<y_{n+1}$. 
It follows from the definition of $\wh h_{n}$ that for every $i\le n$ such that $y_{i}\ge y_{n+1}$, there exists some $j\in[d]$ such that $(\bx_i)_j>(\bx_{n+1})_{j}$. 

Suppose on the contrary that there exists a strictly increasing infinite sequence $(n_t)_{t\in\N}$ such that $\wh h_{n_t}(\bx_{n_{t+1}})\neq y_{n_{t+1}}$ for all $t\in\N$.
Now, define an infinite complete graph with vertex set $\{\bx_{n_{t}}\}_{t\in\N}$ and color each edge $\{\bx_{n_{t}},\bx_{n_{t'}}\}$ with $t<t'$ to be $\min\{j\in[d]:(\bx_{n_t})_j>(\bx_{n_{t'}})_j\}\in[d]$. 
Then, by the infinite Ramsey theory, there exist some $j\in[d]$ a strictly increasing infinite sequence $(t_i)_{i\in\N}$ such that the edge $\{\bx_{n_{t_{i}}},\bx_{n_{t_{i'}}}\}$ is colored with $j$ for all $i\neq i'$. Thus, by the rule of coloring, $(\bx_{n_{t_i}})_j$ is a strictly decreasing infinite sequence in $i$, which contradicts the fact that $(\bx_{n_{t_i}})_j\in \N$ for all $i\in\N$. 
Therefore, $(\wh h_n)_{n\in\N}$ only makes finitely many mistakes for any consistent sequence $((\bx_n,\by_n))_{n\in\N}$. 

Moreover, if $\wh h_n(\bx_{n+1})=y_{n+1}$, we claim that $\wh h_{n+1}=\wh h_n$. Indeed, for any $\bx\in\X$, if $\mathsf{Y}_{S_{n},\bx}\neq \emptyset$, we have $\mathsf{Y}_{S_{n+1},\bx}=\mathsf{Y}_{S_{n},\bx}$. Thus, $\wh h_{n+1}(\bx)=\wh h_n(\bx)$. 
If $\mathsf{Y}_{S_{n},\bx}=\emptyset$, we must have $\wh h_n(\bx)=1$, which implies that we have $k\notin\mathsf{Y}_{S_{n+1},\bx}$ for any $k>1$. 
Thus, $\wh h_{n+1}(\bx)=1=\wh h_n(\bx)$. 
\end{proof}

\end{document}